\documentclass{article}

\usepackage[preprint]{neurips_2026}

\usepackage[utf8]{inputenc} 
\usepackage[T1]{fontenc}    
\usepackage{hyperref}       
\usepackage{url}            
\usepackage{booktabs}       
\usepackage{amsfonts}       
\usepackage{nicefrac}       
\usepackage{microtype}      
\usepackage{xcolor}         

\usepackage{graphicx}
\usepackage{subcaption}
\usepackage{multirow}
\usepackage{enumitem}

\usepackage{siunitx}
\usepackage{sidecap}        
\sidecaptionvpos{figure}{t}
\usepackage{titletoc}       

\usepackage{algorithm}      
\usepackage{algorithmic}    
\usepackage{amsthm,amsfonts,amsmath,amssymb}     
\usepackage{mathtools} 
\usepackage[capitalize,noabbrev]{cleveref}

\theoremstyle{plain}
\newtheorem{theorem}{Theorem}[section]
\newtheorem{proposition}[theorem]{Proposition}
\newtheorem{lemma}[theorem]{Lemma}
\newtheorem{corollary}[theorem]{Corollary}
\theoremstyle{definition}
\newtheorem{definition}[theorem]{Definition}

\newtheorem{remark}[theorem]{Remark}

\newcommand{\R}{\mathbb{R}}
\newcommand{\Hhyperbn}{\mathcal{H}_{\mathrm{hyper}}^{\mathrm{bn}}}
\newcommand{\Hhyperlin}{\mathcal{H}_{\mathrm{hyper}}^{\mathrm{lin}}}

\usepackage[textsize=tiny]{todonotes}

\title{Dynamics-Aligned Shared Hypernetworks for Contextual RL under Discontinuous Shifts}

\author{%
 Jan~Benad$^{1}$\thanks{Equal contribution} \quad
  Pradeep~Kr.~Banerjee$^{1}$\footnotemark[1] \quad
  Frank~R\"oder$^1$ \\
  \textbf{Nihat~Ay}$^{1\,2}$ \quad
  \textbf{Martin~V.~Butz}$^3$ \quad
  \textbf{Manfred~Eppe}$^1$ \\
  $^1$Institute for Data Science Foundations, TUHH, Germany\\
  $^2$Santa Fe Institute, USA\\
  $^3$Neuro-Cognitive Modeling Group, University of T{\"u}bingen, Germany
}

\begin{document}

\maketitle

\begin{abstract}
Zero-shot generalization in contextual reinforcement learning remains a core challenge, particularly when the context is latent and must be inferred from data. A canonical failure mode arises when latent context discontinuously changes how actions affect the environment, requiring incompatible control responses across contexts. We propose DMA*-SH, a framework where a single hypernetwork, trained solely via dynamics prediction, generates a small set of adapter weights shared across the dynamics model, policy, and action-value function. This shared modulation imparts an inductive bias matched to discontinuous context-to-dynamics shifts, while input/output normalization and random input masking stabilize context inference, promoting directionally concentrated representations. We provide theoretical support via expressivity separation results for hypernetwork modulation, and a variance decomposition with policy-gradient variance bounds that formalize how within-mode compression improves learning under non-overlapping contexts. For evaluation, we introduce the Actuator Inversion Benchmark (AIB), a suite of environments designed to isolate challenging context-to-dynamics interactions, including actuator inversion, actuator permutations, and weakly non-overlapping continuous dynamics. On AIB’s held-out tasks, DMA*-SH achieves zero-shot generalization, outperforming domain randomization by 58.1\% and surpassing a standard context-aware baseline by 11.5\% on average.

\end{abstract}


\section{Introduction}
\label{sec:intro}

Imagine switching to a colleague's laptop and discovering the trackpad scrolls in the opposite direction. The screen moves \textit{against} your fingers instead of with them. Despite years of motor expertise, you are suddenly disoriented. Every instinctual swipe produces the wrong outcome. Trying harder does not help; you must override muscle memory. This scenario captures a core challenge in contextual reinforcement learning (RL): zero-shot adaptation can fail when context induces discontinuous changes in action effects.

The trackpad scenario is an instance of what we call \emph{actuator inversion}, and serves as a minimal example of a broader problem: latent context can change the action-effect map discontinuously. The same difficulty arises in swapped control channels, coordinate remappings, gain changes, or sim-to-real actuator mismatches. In such settings, there may be no valid interpolation between incompatible control laws; an ``averaged'' behavior can be incoherent. Smooth inductive biases, such as context concatenation or Gaussian structure in latent context, can then blur distinct modes into an averaged representation. Standard domain randomization (DR) can similarly yield policies that compromise across regimes and degrade performance in both. These failures reflect structural mismatches between architectural inductive biases and the problem's inherent discontinuity. Zero-shot adaptation to such context shifts requires representing and executing multiplicative switching reliably. In practice, this is difficult when context is inferred from short, noisy windows, since small inference errors can trigger the wrong mode and yield unstable behavior.

Such discontinuous shifts require a specific inductive primitive: \emph{multiplicative} context-dependent modulation of network parameters, which can represent functional opposites, for example by inducing approximately sign-reversed or permuted adaptations across modes. Hypernetworks \citep{ha2017hypernetworks,jayakumar2020multiplicative} provide a natural mechanism for this modulation. Building on this principle, we propose DMA*-SH (Dynamics Model-Aligned, Shared Hypernetwork), a framework for context-inferred RL where a single hypernetwork, trained solely via forward dynamics prediction, generates a small set of adapter weights shared across the dynamics model, policy, and Q-function. This design couples control to dynamics-consistent modulation and supports discontinuous, multiplicative regime switches. Carefully chosen normalizations and input masking further encourage directional concentration of context information and are associated with stable representations that compress irrelevant continuous variations while preserving task-critical discontinuous structure.

This paper makes the following contributions:
\begin{itemize}[leftmargin=*]
\item \textbf{Actuator Inversion Benchmark (AIB).}
We introduce AIB, a suite of contextual RL environments designed to isolate challenging context-to-dynamics interactions, including actuator inversion, actuator permutations, and weakly non-overlapping continuous dynamics, while still encompassing standard physical parameter variations.

\item \textbf{DMA*-SH Architecture}.
We propose a shared-hypernetwork framework for context-inferred RL that provides multiplicative modulation for discontinuous context-to-dynamics shifts and couples adaptation to dynamics-consistent conditioning through shared adapters across modules.

\item \textbf{Theoretical Analysis.} We provide (i) an expressivity separation result for hypernetwork modulation (Theorems~\ref{thm:hyper_expressiveness}, \ref{thm:operator-family} and \ref{thm:bottleneck-separation}) formalizing the multiplicative advantage of DMA*-SH for discontinuous contexts; (ii) a variance decomposition (Theorem~\ref{thm:var_decomp_SU}) that isolates distinct sources of embedding variance under actuator inversion; and (iii) a policy-gradient variance bound (Theorem~\ref{thm:pg_var_bound}) linking within-mode compression to learning stability, motivating a structural information bottleneck view.

\item \textbf{Empirical validation.}
On AIB’s held-out, DMA*-SH achieves zero-shot generalization, outperforming domain randomization by 58.1\%, and context-aware baselines Concat (concatenation-based and known context) by 11.5\%, and DA (with separate policy/value hypernetworks and known context) by 4.6\% on average. Ablations confirm that normalization, masking, and hypernetwork sharing each contribute to performance. Code is available at: \url{https://github.com/dma-sh/dmash}.
\end{itemize}

\section{Background}
\label{sec:background}

\textbf{Contextual reinforcement learning.} We formalize the problem using the \textit{Contextual Markov Decision Process} (CMDP) framework \citep{hallak2015contextual,benjamins2023contextualize}. A CMDP is defined by the tuple $(\mathcal{C}, \mathcal{S}, \mathcal{A}, \{P^c\}, \{r^c\}, \gamma)$, where $\mathcal{C}$ is the context space, $\mathcal{S}$ and $\mathcal{A}$ are the state and action spaces, $P^c(s'|s,a)$ specifies the transition probability from state $s$ to $s'$ under action $a$ in context $c\in \mathcal{C}$, $r^c\colon \mathcal{S} \times \mathcal{A} \to \mathbb{R}$ is the reward function, and $\gamma \in (0,1)$ is the discount factor. Each context $c$ defines a distinct MDP with shared $\mathcal{S}$ and $\mathcal{A}$ but possibly differing dynamics $P^c$ and reward $r^c$. The context is fixed within an episode. Following prior work \citep{beukman2023dynamics,benjamins2023contextualize,prasanna2024dreaming,roder2025dynamics}, we focus on variations in transition dynamics only, keeping the reward fixed: $r^c = r$ for all $c \in \mathcal{C}$.

\textbf{Zero-shot generalization.} To evaluate generalization, we define three context sets \citep{kirk2023survey}: $\mathcal{C}_{\text{train}}$ for training, $\mathcal{C}_{\text{eval-in}}$ for in-distribution (interpolation) evaluation, and $\mathcal{C}_{\text{eval-out}}$ for out-of-distribution (extrapolation) evaluation. Each set consists of context instances sampled from environment-specific ranges; sampled instances are pairwise disjoint across sets. The agent learns a policy $\pi_\theta$ maximizing expected return over training contexts: $\frac{1}{|\mathcal{C}_{\text{train}}|} \sum_{c \in \mathcal{C}_{\text{train}}} \mathbb{E}_{\pi_\theta}\!\left[\sum_{t=0}^\infty \gamma^t r(s_t, a_t)\right]$, where $s_{t+1} \sim P^c(\cdot | s_t, a_t)$. During zero-shot evaluation, the agent receives no gradient updates.

\section{Related Work}
\label{sec:related}
Contextual RL provides a natural framework for studying zero-shot generalization across varying environments \citep{hallak2015contextual,modi2018markov,beck2023survey,kirk2023survey}. Prior work broadly assumes either observable context variables \citep{chen2018hardware,seyed2019smile,ball2021augmented,eghbal2021context,sodhani2021multi,mu2022domino,benjamins2023contextualize,prasanna2024dreaming} or infers context implicitly from interaction history \citep{chen2018hardware,xu2019densephysnet,lee2020context,seo2020trajectory,xian2021hyperdynamics,sodhani2022block,melo2022transformers,evans2022context,ndir2024inferring,roder2025dynamics}. 
We follow the latter and learn context representations via self-supervised alignment with a dynamics model, without assuming explicit context access. Although recurrent agents can acquire latent contextual representations \citep{grigsby2024amago,grigsby2024bamago,luo2024efficient,hafner2019learning,hafner2025mastering}, these representations are typically not grounded in environment dynamics. Closely related, \citet{beukman2023dynamics} employ hypernetworks \citep{ha2017hypernetworks} conditioned on explicitly provided context variables rather than inferring context from experience; in contrast, we train a single dynamics-aligned hypernetwork shared across policy and value functions, enabling unified context inference for zero-shot generalization.
An extended discussion of contextual RL and related approaches is provided in Appendix~\ref{sec:related_extended}.

\section{Context Encoding and Utilization} \label{sec:DMA-schemes}

We present a framework for learning a \textbf{d}ynamic \textbf{m}odel-\textbf{a}ligned (DMA) context representation. We first introduce \textbf{DMA*}, which incorporates key enhancements to vanilla DMA that refine and stabilize this latent representation. We then present our shared-hypernetwork approach for incorporating latent context information. We refer to this method as \textbf{DMA*-SH}, as it extends DMA* with a \textbf{s}hared \textbf{h}ypernetwork that jointly informs the dynamics model, policy, and action-value function.

\subsection{Context Inference by Dynamic Model-Aligned Representation Learning}
\label{sec:dma}
We denote by $\tau^c_t$ a sliding window of the past $K$ transitions from the same context $c$, each given as a tuple $(s_t, a_t, \delta s_{t+1})$, where $\delta s_{t+1} = s_{t+1} - s_t$ is the state difference. The sequence $\tau^c_t$ is passed through a \emph{context encoder} $g_\phi(\tau^c_t)$ to produce a context representation $z_t\in \mathbb{R}^{d_z}$. The context encoder is trained jointly with a forward dynamics (FD) model $f_\theta$ that predicts the next state difference $\delta \hat{s}_{t+1}$ given the current state $s_t$, action $a_t$, and inferred context $z_t$. The objective is a reconstruction loss between predicted and true next state differences:
\begin{equation}
\label{eq:dma}
    L_{\phi,\theta} = \left\lVert \delta \hat{s}_{t+1} - \delta s_{t+1} \right\rVert_2^2 .
\end{equation}
This defines the vanilla \textbf{DMA} context-encoding objective. Next, we describe two key modifications that improve the quality and robustness of the learned context representations: random input masking and input/output normalization. We refer to this enhanced context encoder as \textbf{DMA*}; see Figure~\ref{fig:encoder} (Appendix~\ref{sec:hyperparameters}) for the full encoder pipeline. Ablations for our design choices appear in Appendix~\ref{sec:ablations}.

\textbf{Input masking.}
Random masking of input features has been shown to improve representation learning across vision, language, and decision-making domains \citep{devlin2019bert,liu2022masked,he2022masked}. 
We use masking purely as input corruption, rather than optimizing a masked-prediction objective. During training, we apply random masking independently at each timestep within the $K$-step window: for each tuple $(s_t, a_t, \delta s_{t+1})$, we independently zero out the state, action, and next-state-difference vectors with a fixed masking probability. Masking reduces reliance on brittle feature co-adaptations and discourages spurious correlations that do not generalize across contexts.

\textbf{Input normalization.} 
After masking, the concatenations of $(s_{t},a_{t},\delta s_{t+1})$ from $\tau _{t}^{c}$ are processed by a linear layer and normalized via AvgL1Norm \citep{fujimoto2023sale}, defined as \(\text{AvgL1Norm}(x)=\frac{N x}{\sum _{i}|x_{i}|}\). Unlike BatchNorm \citep{ioffe2015batch}, which relies on running statistics and can degrade under small-batch online RL, AvgL1Norm is a per-sample, statistic-free operator suited to sliding windows. It prevents monotonic growth in representation space while preserving relative feature scales \citep{gelada2019deepmdp}, yielding consistent embeddings across $\mathcal{C}_{\text{train}}$ and $\mathcal{C}_{\text{eval-out}}$.

\textbf{Output normalization.}
The normalized and masked sequence is processed by an LSTM, and the final hidden state is projected to a compact context embedding \(z_{t}\in \mathbb{R}^{d_{z}}\) with \(d_{z}=8\). We normalize \(z_{t}\) using SimNorm \citep{lavoie2023simplicial, hansen2024tdmpc2}, which projects \(z_{t}\) into \(L=2\) distinct \(V=4\) dimensional simplices via a group-wise softmax. By constraining the representation to a product of simplices, SimNorm stabilizes online RL by bounding the representation scale and promoting sparsity without relying on batch statistics, preventing representation collapse and improving sample efficiency \citep{obando2025simplicial}.

\begin{figure*}[ht]
\centering
\begin{subfigure}[b]{.4\textwidth}
    \centering
    \includegraphics[height=2.25cm]{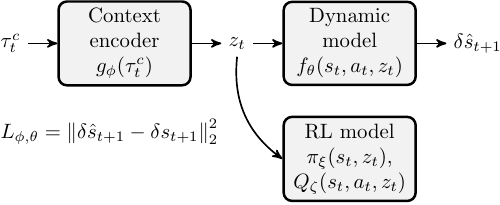}
    \caption{Vanilla DMA.}
    \label{fig:architectureA}
\end{subfigure}\hfill
\begin{subfigure}[b]{.58\textwidth}
    \centering
    \includegraphics[height=2.25cm]{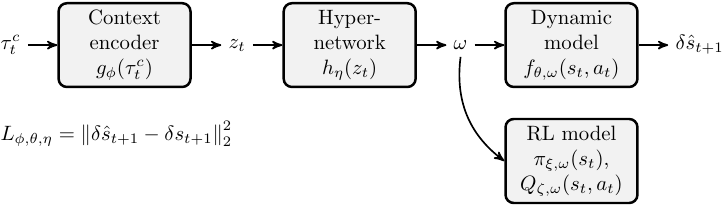}
    \caption{DMA*-SH.}
    \label{fig:architectureB}
\end{subfigure}
\caption{
(a) In vanilla \textbf{DMA}, the inferred context $z_t$ is concatenated to the RL inputs. (b) In \textbf{DMA*-SH}, a hypernetwork $h_\eta$ conditioned on $z_t$ generates adapter weights $\omega$ that are used by the forward dynamics model and the RL networks. The context encoder and hypernetwork are trained via the reconstruction objective $L_{\phi,\theta,\eta}$ in \eqref{eq:hnopt}, while during RL updates gradients through $\omega=h_\eta(z_t)$ are stopped so that the policy and critic losses do not backpropagate to $\eta$ (or to $z_t$) through the shared adapter pathway $z_t \to \omega \to (\pi, Q)$. 
}
\label{fig:architecture}
\end{figure*}

\subsection{Context Utilization by a Shared Dynamic Model-Aligned Hypernetwork (DMA*-SH)}\label{sec:DMA-SH}

In the vanilla DMA setup, the policy and Q-function receive the concatenation of the state $s_t$ and the inferred context $z_t$ as input (Figure~\ref{fig:architectureA}). In contrast, we incorporate $z_t$ using a hypernetwork \citep{ha2017hypernetworks}, which is a meta- or second-order neural network \citep{pollack1990recursive,sugita2011simultaneously,beukman2023dynamics} that generates weights for a target network in an end-to-end differentiable manner. In our approach, the hypernetwork generates weights for only a subset of the main network. We refer to these second order parametrized parts as \emph{adapters}.

As described in Section~\ref{sec:dma}, the context representation $z_t$ is first inferred from {past transitions in} $\tau_t^c$ via dynamic model-aligned representation learning. A hypernetwork $h_\eta$ is then conditioned on $z_t$ to produce weights $\omega$ for the adapters in the dynamic model $f_{\theta,\omega}$, whose parameters are therefore split into generated weights $\omega$ and remaining base weights $\theta$. 
The parameters $\phi$, $\theta$, and $\eta$ for the context encoder, dynamics model, and hypernetwork are updated jointly using the reconstruction loss:
\begin{equation}\label{eq:hnopt}
    L_{\phi,\theta,\eta} = \left\lVert \delta \hat{s}_{t+1} - \delta s_{t+1} \right\rVert_2^2.
\end{equation}
Finally, without further modification, the generated adapter weights $\omega=h_\eta(z_t)$ are reused by the policy and Q-function (with base parameters $\xi$ and $\zeta$ and adapter weights $\omega$) via the adapter pathway $z_t \to \omega \to \{\pi_{\xi,\omega},\, Q_{\zeta,\omega}\}$ (Figure~\ref{fig:architectureB}). 
During RL updates, we detach $\omega$ in the actor and critic losses to prevent reward gradients from modifying the shared hypernetwork and the adapter mapping $z_t \mapsto \omega$; consequently, these gradients also do not reshape the context embedding $z_t$ through the adapter pathway. Thus, this mapping (and the induced context representation it relies on) is trained exclusively through the reconstruction loss~\eqref{eq:hnopt}, embedding dynamics-aligned features into $\omega$. This constraint acts as a structural prior, requiring the actor and critic to process context through adapters shaped by the dynamics. This dynamics-grounded factorization separates mode identification from mode-conditioned control, unlike recurrent or Transformer agents that entangle both in an RL-trained hidden state (Appendices~\ref{app:gradcon} and~\ref{sec:TxrecurDMASH}). For detailed pseudocode, see Algorithms~\ref{alg:dma} and~\ref{alg:dmash} in Appendix~\ref{sec:algorithms}.

\subsection{Expressive Advantage of DMA*-SH via Multiplicative Adapters}
\label{sec:RefHyper_expressivity}

A key architectural advantage of DMA*-SH over vanilla DMA 
arises from using hypernetwork-conditioned \emph{multiplicative adapters} rather than simple concatenation. In vanilla DMA, context enters the actor and critic through the input channel, via concatenated vectors such as $[s_t; z_t]$ and $[s_t; a_t; z_t]$. In this case, any context-dependent transformation of internal features must be synthesized implicitly by the downstream network from additive access to $z_t$.
In contrast, DMA*-SH uses a hypernetwork $h_\eta$ to generate adapter parameters $\omega = h_\eta(z_t)$ that modulate small bottleneck modules injected into the policy, action-value, and dynamics models (Figure~\ref{fig:adapter}). A single adapter with a residual (skip) connection is inserted after the feature trunk and before the output head in each network.
Let $x_t$ denote the trunk features. The adapted features are:
\begin{equation}\label{eq:hypadapter}
    \tilde{x}_t = x_t + g_{\text{adapter}}(x_t;\, \omega),
\end{equation}
where $g_{\text{adapter}}$ is a bottleneck network whose weights $\omega$ are generated by the hypernetwork. This design provides context-dependent modulation of features before they are mapped to action means, action-value estimates, or predicted state differences.

The modulation is multiplicative in the following sense. Inside the adapter, features are transformed by weight matrices that depend on context through $\omega=h_\eta(z_t)$. Concretely, the computation contains terms of the form $W(z_t)\,x_t$, which induces explicit multiplicative interactions between feature coordinates and context coordinates. This form of multiplicative modulation~\citep{jayakumar2020multiplicative} can exactly represent certain functions outside the concatenation hypothesis class.

Theorems~\ref{thm:hyper_expressiveness}, \ref{thm:operator-family}, and \ref{thm:bottleneck-separation} in Appendix~\ref{app:expressiveness} formalize this separation by showing that hypernetwork-conditioned adapters can represent context-dependent multiplicative transformations outside the hypothesis class of concatenation-based ReLU policies. In non-overlapping context shifts, this capacity supports efficient representation of discontinuous changes in action effects, whereas concatenation-based methods must realize such structure only indirectly (Remark~\ref{rem:expressiveness}).

\section{The Actuator Inversion Benchmark (AIB)}
\label{sec:environments}

We introduce the Actuator Inversion Benchmark (AIB) as a diagnostic suite for contextual RL with discontinuous context-to-dynamics interactions. The name reflects the minimal motivating case: actuator inversion, where the same motor command produces opposite effects under a latent mode. Beyond sign flips, AIB includes actuator permutations, where control channels are swapped, and continuous weakly non-overlapping systems such as ODE/ODE-k, where the required action direction changes across regions of the continuous context space. AIB therefore isolates when zero-shot contextual RL requires reliable mode identification and mode-conditioned computation, while still encompassing standard physical parameter variations that primarily demand smooth interpolation.
Each AIB environment has two context dimensions and is classified as either:
\begin{itemize}[leftmargin=*, topsep=2pt, itemsep=1pt]
    \item \textbf{Overlapping:} A single context-unaware policy can perform well across contexts, as with continuous physical parameters such as mass or gravity.
    \item \textbf{Non-overlapping:} Every context-unaware policy without memory incurs nontrivial regret in at least one context, requiring context conditioning.
\end{itemize}
Formal definitions appear in Appendix~\ref{app:overnonover} (Definition~\ref{def:OverlapNonoverlap}).
Table~\ref{tab:envs_summary} summarizes all AIB environments, their context variables, and overlap type. Appendix~\ref{app:env_details} provides task descriptions and classification rationales, and Table~\ref{tab:envs} specifies context supports and return bounds for score normalization. Unless noted otherwise, each environment samples a two-dimensional context once per episode and holds it fixed throughout; context instances are pairwise disjoint across $\mathcal{C}_{\text{train}}$, $\mathcal{C}_{\text{eval-in}}$, and $\mathcal{C}_{\text{eval-out}}$.

Unlike CARL~\citep{benjamins2023contextualize}, which focuses on continuous and categorical context variations, or Meta-World~\citep{yu2020meta}, which varies goals and tasks but not action effects, AIB includes discontinuous changes in action effects, such as sign inversions and channel permutations, yielding mutually incompatible control laws. See Appendix~\ref{sec:Compbench}.

\addtolength{\tabcolsep}{-0.25em}
\begin{table}[t]
\centering
\caption{\textbf{Actuator Inversion Benchmark (AIB).} Each environment has two context dimensions (except ODE-$k$, which uses $k$ dimensions). \textit{Non-overlapping} environments include either a binary actuator inversion factor $c\in\{\pm 1\}$ that flips action effects, or an actuator permutation factor $q \in \{0,1\}$ that swaps action dimensions, both typically yielding mutually incompatible control laws. \textit{Overlapping} environments vary only continuous physical parameters. \emph{Weakly non-overlapping} denotes continuous contexts with empirically low policy overlap. 
$^\dagger$ODE/ODE-$k$ adapted from prior work; details and attribution in Appendix~\ref{app:env_details}.
}
\label{tab:envs_summary}
\small
\begin{tabular}{l l p{6.2cm} l}
\toprule
\textbf{Environment} & \textbf{Source} & \textbf{Context variables} & \textbf{Type} \\
\midrule
DI & Custom & mass; actuator inversion ($\pm1$) & Non-overlap \\
DI-Friction & Custom & mass; friction & Overlap \\
DI-Perm & Custom & mass; actuator permutation $q \in \{0,1\}$ & Non-overlap \\
ODE & Custom$^\dagger$ & ODE coeffs $(c_1,c_2)$ & Weakly non-overlap \\
ODE-$k$ & Custom$^\dagger$ & polynomial ODE coeffs $(c_1,\ldots,c_k)$ & Weakly non-overlap \\
Cartpole & DMC & pole length; actuator inversion ($\pm1$) & Non-overlap \\
Cheetah & DMC & leg length; actuator inversion ($\pm1$) & Non-overlap \\
Reacher (E/H) & DMC & arm length; actuator inversion ($\pm1$) & Non-overlap \\
Reacher (E/H)-Perm & DMC & arm length; actuator permutation $q \in \{0,1\}$ & Non-overlap \\
BallInCup & DMC & tendon length; gravity & Overlap \\
Walker & DMC & actuator strength; gravity & Overlap \\
WalkerGym & Gymnasium & actuator strength; gravity & Overlap \\
HopperGym & Gymnasium & actuator strength; gravity & Overlap \\
\bottomrule
\end{tabular}
\end{table}

\section{Results}
\subsection{Metrics}
\label{sec:metrics}

We adopt a standard evaluation protocol for zero-shot generalization in contextual RL \citep{kirk2023survey,beukman2023dynamics,benjamins2023contextualize}. Specifically, we sample $n_c = 20$ contexts from the environment-specific context ranges listed in Table~\ref{tab:envs} to create the sets $\mathcal{C}_{\text{train}}$, $\mathcal{C}_{\text{eval-in}}$, and $\mathcal{C}_{\text{eval-out}}$, respectively. The agent is trained on $\mathcal{C}_{\text{train}}$. For evaluation, we measure the cumulative episodic return of the trained agent across $n_e = 10$ rollouts per context, and then average within each context set. This yields three averaged episodic returns (AER) \citep{beukman2023dynamics}, one per set. Following \citet{agarwal2021deep}, we report the interquartile mean (IQM) with empirical confidence intervals, after min--max scaling by environment-specific return bounds (Table~\ref{tab:envs}; Appendix~\ref{app:env_details}). Unless stated otherwise, results are averaged over $n_s = 10$ independent random seeds.

\subsection{Baselines}
\label{sec:baselines}

For our methods DMA* and DMA*-SH, as well as for all baselines except Amago, we use Soft Actor-Critic (SAC) \citep{haarnoja2018soft} as the underlying RL algorithm. To ensure comparability, we use SAC with standard hyperparameters and avoid additional tuning. Hyperparameters and implementation details are in Appendix~\ref{sec:hyperparameters}. All approaches are trained under the same procedure: the agent is trained in parallel across the $n_c = 20$ contexts in $\mathcal{C}_{\text{train}}$.

\textbf{Concat} ({\textit{Context-Aware}}).
This baseline concatenates explicit context with the state as input to the policy and Q-function, a standard approach when context variables are available \citep{ball2021augmented,eghbal2021context}.

\textbf{Decision Adapter (DA)} ({\textit{Context-Aware}}).
\citet{beukman2023dynamics} propose a stronger context-aware baseline that uses hypernetworks to adapt policy and Q-function parameters based on the context, improving over other context-aware approaches such as FLAP \citep{peng2021linear} and cGate \citep{benjamins2023contextualize}.

\textbf{Domain Randomization (DR)} ({\textit{Context-Unaware}}).
This baseline ignores explicit context and relies on domain randomization \citep{tobin2017domain} across multiple contexts.

\textbf{Amago} ({\textit{Context-Unaware}}).
Recurrent agents can accumulate latent information over time, enabling in-context adaptation. Amago \citep{grigsby2024amago} is a general-purpose in-context meta-RL algorithm. We use the improved Amago-2 variant \citep{grigsby2024bamago} with a GRU trajectory encoder (Figure~\ref{fig:iqm-amago}). In Appendix~\ref{sec:TxrecurDMASH}, we evaluate a Transformer variant to test whether history-based sequence modeling can replace DMA*-SH's dynamics-grounded shared operator map.

\textbf{Dynamic Model Alignment (DMA)} ({\textit{Context-Inferred}}).
Prior methods such as DALI \citep{roder2025dynamics}, IIDA \citep{evans2022context}, and CaDM \citep{lee2020context} infer context from recent experience via dynamic model alignment. The inferred latent representation is then passed to the policy and Q-function. As DMA* extends this paradigm, we include vanilla DMA as a baseline.

\textbf{DMA-Pearl} ({\textit{Context-Inferred}}).
Pearl \citep{rakelly2019efficient} is a meta-RL algorithm that infers context with a probabilistic encoder trained via Q-function gradients. While \citet{rakelly2019efficient} evaluated Pearl primarily under reward variations, we adapt it to transition dynamics variations by training the context encoder jointly with a dynamics model. This yields a probabilistic extension of DMA, where the context representation is regularized by a KL penalty against a unit Gaussian prior $\mathcal{N}(0,I)$, weighted by $\beta=0.4$ (Figure~\ref{fig:iqm-pearl}). 

\subsection{Zero-Shot Generalization}
\label{sec:zsg}
IQM scores aggregated across all contextualized environments (Figure~\ref{fig:iqm}) show that DMA* and DMA*-SH achieve strong zero-shot generalization, particularly in the out-of-distribution setting.

The strongest competitors are the context-aware Concat and DA baselines, which DMA*-SH consistently outperforms across all three regimes. DMA*-SH also attains consistently strong AER scores across environments and contextualization types (Table~\ref{tab:aer}). We complement these aggregated metrics with a per-context analysis, which reveals how performance changes as evaluation contexts diverge from the training distribution and exposes failure modes that aggregated statistics can obscure. Full per-context heatmaps, bar plots, and learning curves are provided in Appendix~\ref{sec:detailed-results}.

In Walker, simple domain randomization suffices, suggesting that explicit or inferred context can sometimes hinder performance. Despite not being tailored to transition-dynamics variation, the context-unaware Amago algorithm performs competitively, especially in the training and eval-in regime, in most environments including non-overlapping settings such as DI, which cannot be solved by simple domain randomization (unlike DI-Friction with overlapping contexts). We note that Amago uses a substantially higher parameter count than the other SAC-based approaches in our experiments.
DMA-Pearl achieves strong results in overlapping contextualizations; however, its smooth KL prior makes it uncompetitive in non-overlapping settings (Remark~\ref{rem:VariBAD} and Appendix~\ref{sec:ablations}).

\begin{SCfigure}[][ht]
    \includegraphics[width=.5\textwidth]{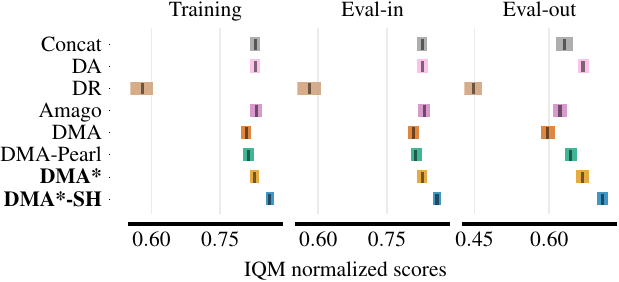}
  \caption{Interquartile mean (IQM) with $95\%$ confidence intervals computed from AER scores, aggregated across environments. Results are reported for the three context sets $\mathcal{C}_{\text{train}}$, $\mathcal{C}_{\text{eval-in}}$, and $\mathcal{C}_{\text{eval-out}}$, comparing DMA*, DMA*-SH, and all baselines.}
  \label{fig:iqm}
\end{SCfigure}

\begin{table}[ht]
  \caption{
  AER scores with $95\%$ confidence intervals for all environments. Results are aggregated across context sets $\mathcal{C}_{\text{train}}$, $\mathcal{C}_{\text{eval-in}}$, and $\mathcal{C}_{\text{eval-out}}$, comparing DMA*, DMA*-SH, and all baselines. Best AER scores are in  bold; if multiple methods are highlighted for an environment, their differences are not distinguishable under the probability of improvement with $95\%$ confidence intervals \citep{agarwal2021deep}. \textit{All} aggregates all environments, \textit{Overlap} aggregates those with overlapping contexts, and \textit{Non-overlap} aggregates those with non-overlapping contexts. For the rows \textit{All}, \textit{Overlap}, and \textit{Non-overlap}, scores are aggregated after environment-wise min--max scaling using environment-specific return bounds (Table~\ref{tab:envs}). Per-set AER tables 
  are provided in Appendix~\ref{sec:detailed-results}.
  }
  \fontsize{5pt}{5pt}\selectfont
  \label{tab:aer}
  \centering
  \begin{tabular}{lrrrrrrrr}
    \toprule
    &\multicolumn{2}{c}{Context-Aware}&\multicolumn{2}{c}{Context-Unaware}&\multicolumn{4}{c}{Context-Inferred} \\
    \cmidrule(l){2-3}\cmidrule(l){4-5}\cmidrule(l){6-9}
    Environment & Concat & DA & DR & Amago& DMA & DMA-Pearl & \textbf{DMA*} & \textbf{DMA*-SH} \\
    \midrule
    DI & 72 [69, 74] & \textbf{75 [74, 76]} & 16 [9, 24] & 61 [52, 70] & 63 [62, 65] & 68 [66, 70] & 75 [74, 75] & \textbf{76 [75, 77]} \\
    DI-Friction & 65 [50, 74] & 76 [74, 77] & 69 [53, 77] & \textbf{79 [78, 79]} & 56 [41, 69] & 74 [73, 76] & 68 [52, 77] & 77 [76, 77] \\
    DI-Perm & 73 [71, 74] & \textbf{76 [75, 77]} & 50 [40, 58] & \textbf{72 [66, 77]} & 68 [66, 70] & 70 [69, 71] & 72 [70, 74] & \textbf{75 [73, 76]} \\
    ODE & 162 [156, 167] & \textbf{179 [176, 183]} & 63 [54, 72] & 168 [167, 169] & 166 [164, 168] & 171 [165, 176] & \textbf{175 [170, 180]} & \textbf{179 [176, 182]} \\
    Cartpole & 863 [851, 877] & 892 [857, 926] & 644 [595, 691] & 639 [568, 715] & 900 [879, 923] & 884 [839, 917] & 927 [904, 950] & \textbf{967 [954, 978]} \\
    Cheetah & 385 [372, 402] & 384 [373, 396] & 281 [261, 303] & \textbf{414 [402, 429]} & \textbf{388 [366, 411]} & 355 [335, 379] & 383 [368, 399] & \textbf{408 [390, 425]} \\
    Reacher (E) & \textbf{890 [859, 922]} & \textbf{878 [843, 910]} & 564 [515, 614] & \textbf{914 [901, 925]} & \textbf{871 [834, 904]} & \textbf{875 [845, 905]} & \textbf{894 [870, 918]} & \textbf{898 [881, 916]} \\
    Reacher (E)-Perm & \textbf{883 [853, 910]} & \textbf{866 [825, 899]} & 659 [625, 701] & 862 [844, 883] & \textbf{881 [849, 912]} & \textbf{892 [870, 911]} & \textbf{869 [841, 897]} & \textbf{903 [881, 922]} \\
    Reacher (H) & 683 [569, 774] & 713 [668, 763] & 266 [185, 350] & \textbf{853 [831, 872]} & 654 [573, 734] & 686 [606, 746] & 697 [611, 771] & \textbf{805 [743, 855]} \\
    Reacher (H)-Perm & 712 [624, 789] & 749 [711, 790] & 328 [244, 417] & 739 [693, 779] & 668 [558, 751] & 668 [609, 730] & 679 [576, 764] & \textbf{840 [800, 873]} \\
    BallInCup & \textbf{924 [915, 932]} & 879 [868, 892] & 862 [824, 887] & 651 [540, 763] & 912 [904, 920] & 903 [895, 912] & 900 [884, 914] & 890 [878, 902] \\
    Walker & 783 [772, 794] & \textbf{784 [765, 801]} & \textbf{798 [785, 810]} & 753 [740, 764] & 762 [718, 790] & \textbf{804 [796, 811]} & 780 [763, 793] & \textbf{806 [798, 814]} \\
    WalkerGym & 2701 [2418, 2960] & 3321 [3180, 3486] & 2811 [2626, 2997] & \textbf{3924 [3675, 4207]} & 3102 [2919, 3307] & 3162 [2902, 3417] & 2960 [2652, 3264] & 3258 [3170, 3338] \\
    HopperGym & 2521 [2434, 2606] & 2542 [2485, 2601] & 2377 [2313, 2437] & 2501 [2427, 2569] & \textbf{2629 [2549, 2696]} & \textbf{2646 [2595, 2696]} & \textbf{2635 [2586, 2692]} & 2563 [2535, 2588] \\
    \midrule
    All & 0.73 [0.71, 0.75] & 0.76 [0.75, 0.77] & 0.52 [0.5, 0.54] & 0.73 [0.72, 0.74] & 0.72 [0.7, 0.74] & 0.74 [0.73, 0.75] & 0.75 [0.73, 0.76] & \textbf{0.79 [0.78, 0.79]} \\
    Overlap & 0.71 [0.68, 0.74] & \textbf{0.75 [0.74, 0.76]} & 0.71 [0.67, 0.73] & \textbf{0.73 [0.7, 0.75]} & 0.71 [0.68, 0.74] & \textbf{0.76 [0.74, 0.77]} & 0.73 [0.69, 0.75] & \textbf{0.76 [0.75, 0.76]} \\
    Non-overlap & 0.74 [0.72, 0.76] & 0.77 [0.75, 0.78] & 0.41 [0.39, 0.44] & 0.73 [0.72, 0.75] & 0.72 [0.7, 0.74] & 0.73 [0.72, 0.74] & 0.75 [0.74, 0.77] & \textbf{0.8 [0.79, 0.81]} \\
    \bottomrule
  \end{tabular}
\end{table}

\subsection{Context-Embedding Diagnostics and Geometry}
\label{sec:variability}
To analyze how RL task performance relates to $z_t$, we introduce three evaluation criteria: \emph{Informativeness}, \emph{Variability}, and \emph{Representation-Overlap} ($\mathrm{RO}$).

For each split $M\in\{\text{train},\text{eval-in},\text{eval-out}\}$, we collect a dataset of trajectory windows $\tau$ from contexts $c\in\mathcal{C}_M$ and embed each window as $z=g_\phi(\tau)$.

\textbf{Informativeness.}
We quantify how strongly the embedding $z_t$ depends on the true context $c$ via the mutual information $I(z_t;c)$. We estimate $I(z_t;c)$ using the $k$-nearest-neighbors entropy estimator \citep{kraskov2004estimating} with $k=3$ under $L_2$ distance, and use this estimate as a diagnostic of context dependence in the learned representation \citep{garcin2025studying}. A higher value of $I(z_{t};c)$ signifies that \(z_{t}\) varies more systematically with $c$ across the collected trajectories.

\textbf{Variability.}
We measure dataset-level spread of context representations $z\in\mathbb{R}^{d_z}$ as $\mathrm{Variability}(M)=\frac{1}{d_z}\operatorname{tr}(\mathrm{Cov}(Z))$ (Definition~\ref{def:Variability} in Appendix~\ref{app:varcompress}). Lower Variability corresponds to more consistent context signals across trajectory windows, which can facilitate robust policy learning.

\textbf{Representation-Overlap ($\mathrm{RO}$).}
$\mathrm{RO}$ is the average pairwise cosine similarity between per-context mean embeddings over all context pairs (Definition~\ref{def:RO}, \eqref{eq:RO_cos}). Higher $\mathrm{RO}$ indicates stronger global directional concentration of context means in representation space.

Figure~\ref{fig:variability} shows that Variability decreases along the ablation path DMA $\rightarrow$ DMA* $\rightarrow$ DMA*-SH across tasks. Theorem~\ref{thm:pg_var_bound} (Appendix~\ref{app:VariabilityTh}) shows lower Variability tightens gradient-variance bounds, explaining this correlation with RL performance. Surprisingly, DMA*-SH has lower Informativeness $I(z_t;c)$ yet better returns. Proposition~\ref{prop:paradox} resolves this apparent paradox: within-mode compression reduces policy-gradient variance even as total information decreases. Interpreted through an approximate \textit{structural information bottleneck} lens (Appendix~\ref{app:sib}), this reflects mode sufficiency with selective within-mode compression. The key structural link is Theorem~\ref{thm:var_decomp_SU}, which decomposes Variability into within-context noise, within-mode spread along continuous context dimensions, and between-mode separation induced by actuator inversion.

These findings are complemented by a RO analysis (Figure~\ref{fig:variability}). Across our ablations, DMA*-SH attains higher $\mathrm{RO}$, suggesting more \textit{directionally concentrated} representations: within-mode variation induces smaller directional changes, while actuator-inversion modes remain separated. In DI, the t-SNE and cosine analyses in Figure~\ref{fig:tsnero} are consistent with reduced within-mode spread across mass while preserving separation across actuator inversion, aligning with the compression and separation terms in Theorem~\ref{thm:var_decomp_SU}. See Appendix~\ref{app:varcompress} for details on directional geometry induced by normalization and shared hypernetwork conditioning.

\begin{SCfigure}[][ht]
    \includegraphics[width=.63\textwidth]{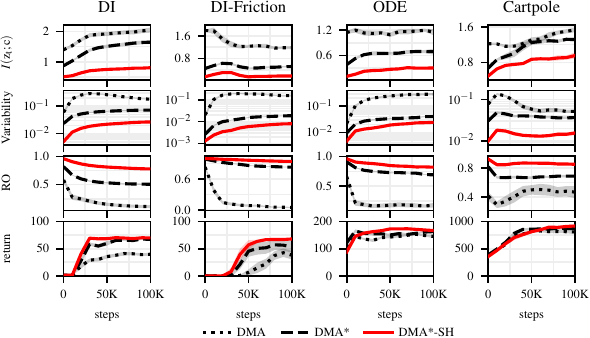}
    \caption{
    Informativeness $I(z_t;c)$, Variability, Representation-Overlap ($\mathrm{RO}$), and episodic returns on $\mathcal{C}_{\text{eval-out}}$. 
    DMA*-SH reduces Variability and increases $\mathrm{RO}$, indicating directionally concentrated embeddings that are less sensitive to within-mode nuisance variation while preserving task-relevant mode structure. This geometry correlates with stronger zero-shot returns and supports the structural information-bottleneck interpretation in Appendix~\ref{app:sib}. 
    }
    \label{fig:variability}
\end{SCfigure}

\begin{figure}[ht]
\centering
\begin{subfigure}[b]{.31\textwidth}
    \centering
    \includegraphics[width=0.9\textwidth]{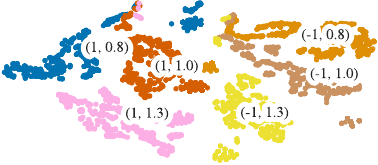}\\  
    \includegraphics[width=\textwidth]{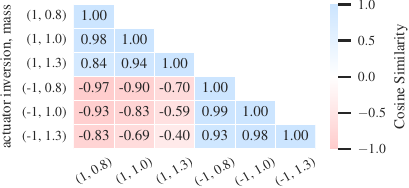}
    \caption{DMA.}
    \label{fig:tsneroA}
\end{subfigure}\hspace{.02\textwidth}
\begin{subfigure}[b]{.31\textwidth}
    \centering
    \includegraphics[width=0.9\textwidth]{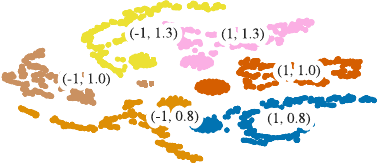}\\
    \includegraphics[width=\textwidth]{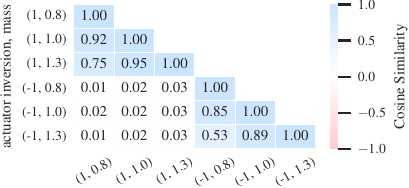}
    \caption{DMA*.}
    \label{fig:tsneroB}
\end{subfigure}\hspace{.02\textwidth}
\begin{subfigure}[b]{.31\textwidth}
    \centering
    \includegraphics[width=0.9\textwidth]{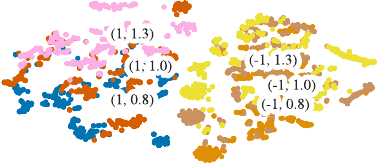}\\
    \includegraphics[width=\textwidth]{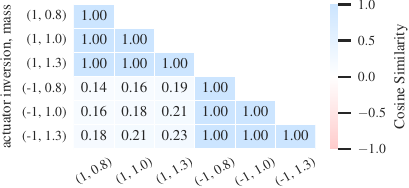}
    \caption{DMA*-SH.}
    \label{fig:tsneroC}
\end{subfigure}
\caption{
DI (non-overlapping): t-SNE of inferred embeddings $z_t$ (top) and cosine similarity heatmaps of per-context mean embeddings (bottom; \eqref{eq:pairwise_cosim}) for DMA, DMA*, and DMA*-SH. DMA*-SH shows stronger within-mode alignment across the continuous mass dimension while maintaining separation between actuator-inversion modes, consistent with the compression/separation terms in Theorem~\ref{thm:var_decomp_SU}. 
Mass clusters overlap more for DMA*-SH, yet returns are higher, consistent with mass having largely overlapping policy effects.
}
\label{fig:tsnero}
\end{figure}

\subsection{Implicit Gradient Regularization via Shared Hypernetworks}\label{sec:implreg}

In DMA*-SH, shared adapters $\omega=h_\eta(z_t)$ trained only by~\eqref{eq:hnopt} and detached in RL losses can be viewed as an implicit gradient regularizer for context-conditioned policy learning. We compare DMA*-SH to a separate-hypernetwork variant, DMA*-H, in which the dynamics, actor, and critic each have their own hypernetwork (with parameters $\eta^f,\eta^\pi,\eta^Q$, producing $\omega^f,\omega^\pi,\omega^Q$), and the actor/critic hypernetworks can adapt directly to their respective RL objectives. To quantify how strongly the policy objective depends on the inferred context through the adapter pathway, we report the mean norm of a \emph{shadow} $z$-space gradient, $\mathbb{E}\|\nabla_z L_\pi\|$, in Figure~\ref{fig:gradcon-selection}. The shadow gradient is computed by temporarily removing the stop-gradient on $\omega=h_\eta(z)$; under the actual training rule in DMA*-SH, $\nabla_z L_\pi=0$ since $\omega$ is detached in the RL losses and $z_t$ is used by the actor only through $\omega$. Shadow gradients serve purely as diagnostic signatures of effective context utilization; they are never applied during training. Across environments, DMA*-SH shows persistently non-negligible shadow norms, indicating sustained hypothetical sensitivity of the policy objective to the context signal under shared, dynamics-trained adapters. In contrast, DMA*-H exhibits substantially smaller values, consistent with weaker effective context dependence along the corresponding adapter pathway. This separation co-occurs with faster learning and higher returns for DMA*-SH (Figure~\ref{fig:gradcon-selection}). See Appendices~\ref{app:gradcon} and~\ref{sec:TxrecurDMASH} for an extended discussion and additional diagnostics.

\begin{SCfigure}[][ht]
    \includegraphics[width=.63\textwidth]{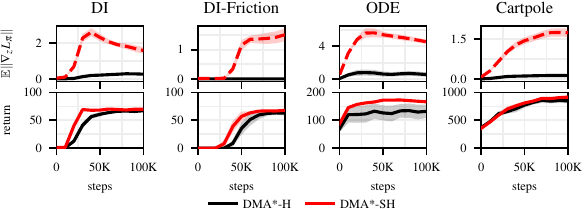}
    \caption{
    Implicit regularization of RL via a dynamics-trained shared hypernetwork. Top: Mean policy context sensitivity in shared embedding space $\mathbb{E}\|\nabla_z L_\pi\|$. Bottom: Episodic returns.}
    \label{fig:gradcon-selection}
\end{SCfigure}

\section{Conclusions}
\label{sec:Conclusions}

We introduced the Actuator Inversion Benchmark (AIB) as a diagnostic suite for zero-shot generalization under discontinuous context-to-dynamics shifts. We proposed DMA*-SH, a context-inferred RL framework using a shared hypernetwork to couple latent context inference to dynamics-aligned modulation. Theoretically, we proved expressivity separation results for hypernetwork-conditioned adapters (Theorems~\ref{thm:hyper_expressiveness}, \ref{thm:operator-family}, \ref{thm:bottleneck-separation}), identifying multiplicative modulation as an appropriate inductive primitive for non-overlapping context shifts. Empirically, DMA*-SH embeddings show stronger within-mode directional alignment across continuous context variation while preserving mode separation. Theorem~\ref{thm:var_decomp_SU} connects these patterns to Variability through within-mode compression and between-mode separation; Theorem~\ref{thm:pg_var_bound} links reduced Variability to tighter policy-gradient variance bounds. Together, these results support an approximate structural information bottleneck view: preserve mode-relevant information while compressing nuisance variation. Finally, shared hypernetwork design regularizes RL updates toward dynamics-consistent solutions. Notably, these effects emerge \textit{without} auxiliary objectives. Directional concentration, nuisance compression, structural-bottleneck-like behavior, and implicit gradient regularization arise from the interlocking design of dynamics-aligned training, shared modulation, and input/output normalization. Overall, our results support DMA*-SH as a principled approach to zero-shot contextual RL under discontinuous shifts.

\paragraph{Limitations and Future Work.}
DMA*-SH couples context inference to learned dynamics through a shared hypernetwork, so model errors can propagate into the representation and impair adaptation under misspecification or rapidly shifting dynamics. Multiplicative modulation is natural for discontinuous action-effect shifts, but may be less effective when context modifies rewards or induces non-factorizable policy changes. Hypernetwork capacity is also a practical bottleneck: too little limits expressiveness, while too much can overfit and reduce stability.

Promising directions include robustness to model uncertainty via ensembles or Bayesian hypernetworks~\citep{krueger2017bayesian}, extensions to reward-modifying contexts via dual hypernetworks or multi-view encoders, and explicit objectives for strengthening compression/separation~\citep{li2024Unicorn,li2024efficient}. Real-robot deployments involving actuator degradation, swapped channels, or intermittent faults provide natural testbeds for shared modulation under genuinely discontinuous dynamics shifts.

\begin{ack}
JB and ME gratefully acknowledge funding by the German Research Foundation DFG through the MoReSpace (402776968) project.
\end{ack}

{
\small
\bibliography{main}
\bibliographystyle{plainnat}
}

\newpage
\appendix

\section*{Appendix}
\startcontents[appendix]
\printcontents[appendix]{l}{0}

\clearpage

\section{Theoretical Results and Supplementary Analyses}\label{app:DMA-theory}

\subsection{Multiplicative Interactions in Contextual Policies}\label{app:expressiveness}

One standard approach to conditioning a policy on a context embedding $z_t \in \mathbb{R}^{d_z}$ is through concatenation with the state $s_t$, yielding an input $[s_t; z_t]$ to a ReLU MLP. Such networks are continuous piecewise-linear (CPWL) in the joint input, with the Hessian matrix vanishing almost everywhere.
CPWL functions can approximate complex interactions by using many linear regions, but they cannot represent truly bilinear couplings exactly on any domain with non-empty interior. As a result, mode-switching and sign-dependent transformations, such as those required for actuator inversion (Section~\ref{sec:environments}), are typically realized only indirectly through a fine partition of the input space.

In contrast, \emph{multiplicative interactions} enable richer structure via bilinear forms: $f(s, z) = z^\top W s + \text{lower-order terms}$, where $W$ captures cross-terms between $s$ and $z$~\citep{jayakumar2020multiplicative,galanti2020modularity}. Hypernetworks instantiate such interactions: when $h_\eta(z)$ generates weights for an adapter layer computing $W^{(\omega)} x$, where $x$ are features derived from $s$, the map is piecewise bilinear in $(x, z)$, with regions determined by the activation patterns of both the hypernetwork and the adapter.

We show that hypernetwork-conditioned adapters can realize functions that concatenation ReLU MLPs cannot represent exactly. We first state the separation in the scalar-output and linear adapter case:
\begin{theorem}[Separation of hypernetwork-adapter and concatenation hypothesis classes]
\label{thm:hyper_expressiveness}
Let $\mathcal{H}_{\mathrm{concat}}$ be the class of functions $f: \mathbb{R}^{d_s} \times \mathbb{R}^{d_z} \to \mathbb{R}$ realized by finite ReLU MLPs on input $[s; z]$. Let $\mathcal{H}_{\mathrm{hyper}}$ be the class of functions of the form
\begin{equation}\label{eq:hyper_class}
f(s,z) = w^\top \Big( x(s) + g_{\text{adapter}}\bigl(x(s);\, \omega = h_\eta(z)\bigr) \Big) + b,
\end{equation}
where:
\begin{itemize}[leftmargin=*, topsep=2pt, itemsep=1pt]
    \item $x: \mathbb{R}^{d_s} \to \mathbb{R}^{n}$ is a finite ReLU network (trunk),
    \item $h_\eta: \mathbb{R}^{d_z} \to \mathbb{R}^{d_\omega}$ is a finite ReLU network with linear output layer (hypernetwork),
    \item for each $\omega \in \mathbb{R}^{d_\omega}$, $g_{\text{adapter}}(\cdot\,; \omega): \mathbb{R}^{n} \to \mathbb{R}^{n}$ is the linear map $x \mapsto W^{(\omega)} x$, where $W^{(\omega)} \in \mathbb{R}^{n \times n}$ is the adapter weight matrix parameterized by $\omega$,
    \item $w \in \mathbb{R}^{n}$, $b \in \mathbb{R}$ (linear output head).
\end{itemize}

Assume $s \in \mathcal{S} \subset \mathbb{R}^{d_s}$ and $z \in \mathcal{Z} \subset \mathbb{R}^{d_z}$ range over compact sets with non-empty interior. Then:
$$
\mathcal{H}_{\mathrm{hyper}} \not\subseteq \mathcal{H}_{\mathrm{concat}}.
$$
That is, there exists $f^\star \in \mathcal{H}_{\mathrm{hyper}}$ such that $f^\star \notin \mathcal{H}_{\mathrm{concat}}$.
\end{theorem}

\begin{proof}
We exhibit $f^\star(s, z) = s \cdot z$ (scalar case, $d_s = d_z = 1$) and show $f^\star \in \mathcal{H}_{\mathrm{hyper}} \setminus \mathcal{H}_{\mathrm{concat}}$.

\textbf{Membership in $\mathcal{H}_{\mathrm{hyper}}$.}
Choose $n = 2$. Define:
\begin{align*}
x(s) &= \bigl[\mathrm{ReLU}(s), \mathrm{ReLU}(-s)\bigr]^\top \in \mathbb{R}^2, \\
h_\eta(z) &= [z-1, -(z+1)]^\top \in \mathbb{R}^2, \\
W^{(\omega)} &= \mathrm{diag}(\omega) = \begin{pmatrix} \omega_1 & 0 \\ 0 & \omega_2 \end{pmatrix}, \\
w &= [1, 1]^\top, \quad b = 0.
\end{align*}
The feature trunk $x$ is a valid finite ReLU network. The hypernetwork $h_{\eta }$ is affine in $z$ and thus realizable by a ReLU network with linear output. 
The adapter computes $g_{\text{adapter}}(x; \omega) = W^{(\omega)} x = [\omega_1 x_1,\; \omega_2 x_2]^\top$, 
resulting in adapted features with a skip connection
$\tilde{x} = x + g_{\text{adapter}}(x; \omega) = \bigl[(1 + \omega_1) x_1, (1 + \omega_2) x_2\bigr]^\top$. By substituting $\omega_1 = z - 1$ and $\omega_2 = -(z+1)$, we obtain $\tilde{x} = \bigl[z \cdot \mathrm{ReLU}(s), -z \cdot \mathrm{ReLU}(-s)\bigr]^\top$. 
The output is $$f(s,z) = w^\top \tilde{x} = z \cdot \mathrm{ReLU}(s) - z \cdot \mathrm{ReLU}(-s) = z \cdot s,$$ using the identity $\mathrm{ReLU}(s) - \mathrm{ReLU}(-s) = s$. Hence $f^\star \in \mathcal{H}_{\mathrm{hyper}}$.

\textbf{Non-membership in $\mathcal{H}_{\mathrm{concat}}$.} 
Any finite ReLU MLP realizes a CPWL function~\citep{montufar2014number}. A CPWL function is affine on each piece of a polyhedral partition, so its Hessian vanishes almost everywhere. However, $f^\star(s, z) = sz$ has mixed partial derivative $\frac{\partial^2 f^\star}{\partial s \, \partial z} = 1 \neq 0$ everywhere. No CPWL function can equal $f^\star$ on a domain with non-empty interior. Hence $f^\star \notin \mathcal{H}_{\mathrm{concat}}$.
\end{proof}

Theorem~\ref{thm:hyper_expressiveness} proves the expressive gap in the minimal sign-flip setting. The following proposition and corollary show that the same hypernetwork-adapter mechanism extends to broader context-dependent operator switching. Theorem~\ref{thm:operator-family} gives exact realization for broader operator families, including discrete permutations and gain modulation; Corollary~\ref{cor:continuous-separation} gives an exact CPWL separation result for continuous affine operator families.
\begin{theorem}[Operator-family realization by hypernetwork-conditioned adapters]
\label{thm:operator-family}
Fix a feature map $x:\R^{d_s}\to\R^n$, an output vector $w\in\R^n$, a bias $b\in\R$, and a context domain $Z\subseteq\R^{d_z}$.

Let $T:Z\to\R^{n\times n}$ be a matrix-valued map.
Assume the hypernetwork and adapter parameterization can realize
\begin{align}\label{eq:adapter-assumption}
  W(h_\eta(z))=T(z)-I_n
  \qquad\text{for all }z\in Z,
\end{align}
where the linear adapter is $g_{\mathrm{adapter}}(x;\omega)=W(\omega)x$.
Define
\begin{align}\label{eq:f-T}
  f_T(s,z)=w^\top T(z)x(s)+b.
\end{align}
Then $f_T\in\mathcal{H}_{\mathrm{hyper}}$.
In particular:
\begin{enumerate}[leftmargin=*]
  \item Actuator inversion: when the mode variable is a discrete input $c\in\{\pm1\}$ (or an equivalent one-hot encoding), $T(c)\in\{I_n,-I_n\}$ gives $W\in\{0_{n\times n},-2I_n\}$. The map $c \mapsto T(c)-I_n$ takes only two values on the discrete mode domain $c\in\{\pm1\}$ and is therefore exactly realizable.
  \item Permutation-induced wire swaps: when the mode is a discrete variable $q\in\{1,\ldots,|G|\}$ indexing a finite group $G$, $T(q)=P_q$ is permutation-valued, and $W(h_\eta(q))=P_q-I_n$ takes finitely many values. When $q$ is presented as a discrete input (e.g., one-hot), the   hypernetwork need only map finitely many input symbols to finitely many fixed output vectors, which is exactly realizable by a ReLU network with linear output. 
  \item Gain modulation corresponds to $T(z)=D(z)$ diagonal, and is exactly realizable when the diagonal entries are continuous piecewise-linear functions of $z$.
\end{enumerate}
\end{theorem}

\begin{proof}
By the definition of $\mathcal{H}_{\mathrm{hyper}}$ in Theorem~A.1, any function of the form
$$f(s,z) = w^\top\bigl(x(s)+g_{\mathrm{adapter}}(x(s);h_\eta(z))\bigr)+b =  w^\top\bigl(x(s)+W(h_\eta(z))\,x(s)\bigr)+b
$$
belongs to $\mathcal{H}_{\mathrm{hyper}}$.
Using assumption~\eqref{eq:adapter-assumption},
\begin{align}
  x(s)+g_{\mathrm{adapter}}(x(s);h_\eta(z))
  &= x(s)+W(h_\eta(z))\,x(s) \notag \\
  &= x(s)+(T(z)-I_n)\,x(s) \notag \\
  &= T(z)\,x(s). \label{eq:Tx}
\end{align}
Therefore $f(s,z)=w^\top T(z)\,x(s)+b=f_T(s,z)$, proving $f_T\in\mathcal{H}_{\mathrm{hyper}}$.

For the specific cases:
\begin{enumerate}[leftmargin=*]
  \item $T(c)\in\{I_n,-I_n\}$ gives $W(h_\eta(c))\in\{0_{n\times n},-2I_n\}$. Since the mode variable $c$ takes only two discrete values, the hypernetwork need only map two input symbols to two fixed output vectors, which is exactly realizable.
  \item For a finite group $G$ with $|G|<\infty$ and a discrete mode variable $q\in\{1,\ldots,|G|\}$, $T(q)=P_q$ requires $h_\eta$ to map $q$ to the vectorized entries of $P_q-I_n$, a function taking finitely many values on a discrete domain. When $q$ is presented as a discrete input (e.g., integer or one-hot), a ReLU network with linear output can realize this map exactly. 
  \item For $T(z)=\mathrm{diag}(d_1(z),\ldots,d_n(z))$ with entries that are continuous piecewise-linear in $z$, the map $z\mapsto T(z)-I_n$ is also CPWL and hence realizable by a finite ReLU network. \qedhere
\end{enumerate}
\end{proof}

\begin{corollary}[Continuous operator-family separation]
\label{cor:continuous-separation}
Let $S\subset\R^{d_s}$ and $Z\subset\R^{d_z}$ be compact domains with non-empty interior. Suppose there exists a non-empty open set $U\subset S$ on which the feature map is affine:
\begin{align}\label{eq:affine-trunk}
  x(s)=As+a,\qquad s\in U,
\end{align}
for some matrix $A\in\R^{n\times d_s}$ and vector $a\in\R^n$.
Let
\begin{align}\label{eq:affine-T}
  T(z)=T_0+\sum_{j=1}^{d_z}z_j\,T_j,\qquad z\in Z,
\end{align}
with fixed matrices $T_0,T_1,\ldots,T_{d_z}\in\R^{n\times n}$, and define
\begin{align}\label{eq:f-affine}
  f(s,z)=w^\top T(z)\,x(s)+b.
\end{align}
Assume there exist indices $i\in\{1,\ldots,d_s\}$ and
$j\in\{1,\ldots,d_z\}$ such that
\begin{align}\label{eq:nondegen}
  w^\top T_j\,A\,e_i\neq 0,
\end{align}
where $e_i$ is the $i$-th standard basis vector in $\R^{d_s}$. Then
$$f\in\mathcal{H}_{\mathrm{hyper}} \qquad\text{and}\qquad f\notin\mathcal{H}_{\mathrm{concat}}.
$$
\end{corollary}

\begin{proof}
\textbf{Membership in $\mathcal{H}_{\mathrm{hyper}}$.}
Since $T(z)$ is affine in $z$, the map
$z\mapsto W(h_\eta(z))=T(z)-I_n=(T_0-I_n)+\sum_j z_j T_j$ is affine in $z$ and therefore realizable by a ReLU network with linear output layer.
The result follows from Theorem~\ref{thm:operator-family}.

\textbf{Non-membership in $\mathcal{H}_{\mathrm{concat}}$.}
Fix any non-empty open set $V\subset Z$.
On $U\times V$, using~\eqref{eq:affine-trunk}
and~\eqref{eq:affine-T}:
$$f(s,z) = w^\top\Bigl(T_0+\sum_{j=1}^{d_z}z_j\,T_j\Bigr)(As+a)+b.$$
This is a polynomial (in fact, at most degree~2) in $(s,z)$, hence smooth on $U\times V$. Its mixed second derivative is $$\frac{\partial^2 f}{\partial s_i\,\partial z_j}(s,z) = w^\top T_j\,A\,e_i,$$
which is a \emph{constant} that is nonzero by assumption~\eqref{eq:nondegen}.

By contrast, any function in $\mathcal{H}_{\mathrm{concat}}$ is realized by a finite ReLU MLP on input $[s;z]$, which is continuous piecewise-linear (CPWL). A CPWL function is affine on each cell of a polyhedral partition, so its Hessian vanishes almost everywhere with respect to Lebesgue measure. Since $U\times V$ is open and non-empty, it has positive Lebesgue measure. But $\frac{\partial^2 f}{\partial s_i\partial z_j}\neq 0$ \emph{everywhere} on $U\times V$, contradicting the almost-everywhere vanishing of the CPWL Hessian. Therefore $f\notin\mathcal{H}_{\mathrm{concat}}$.
\end{proof}

\begin{remark}[SAC with a shared hypernetwork-conditioned bottleneck adapter]
\label{rem:HypPolicyArch}
Our implementation uses Soft Actor-Critic (SAC)~\citep{haarnoja2018soft}. 
The actor produces a Gaussian policy by outputting $(\mu,\log\sigma)$; 
the critic outputs scalar action-values. 
A single bottleneck adapter is inserted after the trunk and before the output head in each network (Figure~\ref{fig:adapter}). In our implementation (Table~\ref{tab:hyperparameters}), the trunk feature dimension is $d=256$ and the bottleneck dimension is $k=32$. 
The adapter includes a ReLU nonlinearity and has the form:
\begin{align}\label{eq:adapter_arch}
g_{\text{adapter}}(x; \omega)
=
W_{\text{up}}^{(\omega)} \cdot \mathrm{ReLU}\!\bigl(W_{\text{down}}^{(\omega)} \cdot x\bigr),
\end{align}
where $W_{\text{down}}^{(\omega)}: \mathbb{R}^{d} \to \mathbb{R}^{k}$ and
$W_{\text{up}}^{(\omega)}: \mathbb{R}^{k} \to \mathbb{R}^{d}$ are generated by a hypernetwork $h_\eta(z_t)$. 
The adapted features are $\tilde{x} = x + g_{\text{adapter}}(x; \omega)$ (skip connection; \eqref{eq:hypadapter}). 
Following \citet{beukman2023dynamics}, we omit the activation function at the end of the trunk in both the actor and critic, so the adapter receives linear features.
We share the same hypernetwork parameters $\eta$ across the actor, critic, and dynamics adapters. Consequently, gradients from the dynamics objective update the shared mapping $z_t \mapsto \omega$, and the actor/critic can exploit this same context-to-parameter mapping through the shared adapters.
\end{remark}

\begin{figure}[ht]
\centering
\begin{subfigure}[b]{0.4\textwidth}
    \centering
    \includegraphics[scale=0.75]{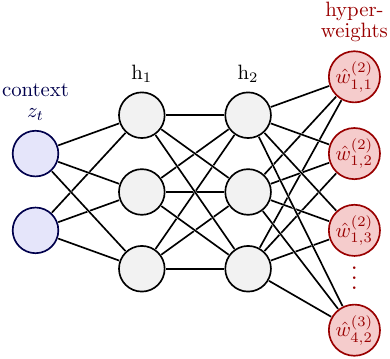}
    \caption{Hypernetwork.}
    \label{fig:adapterA}
\end{subfigure}\hfill
\begin{subfigure}[b]{0.6\textwidth}
    \centering
    \includegraphics[scale=0.75,trim=0 -12 0 0]{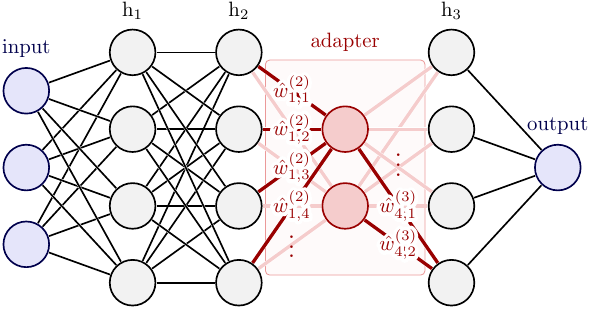}
    \caption{Model network with adapter.}
    \label{fig:adapterB}
\end{subfigure}
\caption{\textbf{Bottleneck adapter}. A hypernetwork (a) predicts parameters that are used within the dynamic model and RL networks (b).}
\label{fig:adapter}
\end{figure}

\paragraph{Separation result for the bottleneck adapter architecture.}
The proof of Theorem~\ref{thm:hyper_expressiveness} uses a simplified linear adapter $g_{\text{adapter}}(x; \omega)=W^{(\omega)}x$ to isolate the core mechanism: context-dependent weights multiplying features to produce bilinear interactions in $(x,z)$. DMA*-SH uses the bottleneck adapter in~\eqref{eq:adapter_arch}, which instantiates the same weight-modulation principle and adds capacity through the intermediate nonlinearity. 
Define the bottleneck-adapter hypothesis class:
$$\Hhyperbn:\quad  f(s,z)=w^\top\bigl(x(s)+g_{\mathrm{adapter}}(x(s);h_\eta(z))\bigr)+b
$$
with $g_{\mathrm{adapter}}$ as in~\eqref{eq:adapter_arch}.
The linear adapter with full $W(\omega)\in\R^{d\times d}$ can represent adapter maps of rank up to $d$. The bottleneck adapter~\eqref{eq:adapter_arch}, even ignoring the ReLU nonlinearity, computes $W_{\mathrm{up}} W_{\mathrm{down}} x$ which has rank at most $k<d$. Therefore, an arbitrary full-rank linear adapter $W(\omega)x$ with $\mathrm{rank}(W)>k$ cannot be represented by the bottleneck adapter. This means $\Hhyperlin\not\subseteq\Hhyperbn$ in general. Despite this, the specific witness function $f^*(s,z)=sz$ from Theorem~\ref{thm:hyper_expressiveness} is realizable by the bottleneck adapter, because the construction requires only rank~2.

\begin{theorem}[Separation for the bottleneck adapter]
\label{thm:bottleneck-separation}
Let $\mathcal{H}_{\mathrm{hyper}}^{\mathrm{bn}}$ be the bottleneck-adapter hypothesis class with bottleneck dimension $k\geq 2$, trunk dimension $d$, and trunk/hypernetwork realized by finite ReLU networks (with linear output for the hypernetwork). Let $\mathcal{H}_{\mathrm{concat}}$ be as in Theorem~A.1.
Then
$$
\mathcal{H}_{\mathrm{hyper}}^{\mathrm{bn}}\not\subseteq\mathcal{H}_{\mathrm{concat}}.
$$
\end{theorem}

\begin{proof}
We show that the witness $f^*(s,z)=s\cdot z$ (with $d_s=d_z=1$) belongs to $\Hhyperbn\setminus\mathcal{H}_{\mathrm{concat}}$.

\medskip\noindent
\textbf{Non-membership in $\mathcal{H}_{\mathrm{concat}}$.} 
This is proved in Theorem~\ref{thm:hyper_expressiveness} via the mixed-partial argument:
$\frac{\partial^2 f^*}{\partial s\,\partial z}=1\neq 0$ everywhere,
which is incompatible with CPWL functions on any domain with non-empty
interior.

\medskip\noindent
\textbf{Membership in $\Hhyperbn$.}
We exhibit an explicit construction with $k\geq 2$.
Define the trunk as a ReLU network $x:\R\to\R^d$:
$$x(s)=\bigl(\mathrm{ReLU}(s),\;\mathrm{ReLU}(-s),\;0,\ldots,0\bigr)^\top\in\R^d.$$
Note that $x_1(s)=\mathrm{ReLU}(s)\geq 0$ and $x_2(s)=\mathrm{ReLU}(-s)\geq 0$ for all $s\in\R$.

Define the hypernetwork output $\omega=h_\eta(z)$ to parameterize:
\begin{align}  W_{\mathrm{down}}^{(\omega)}&:\R^d\to\R^k,\qquad (W_{\mathrm{down}})_{1,1}=1,\;(W_{\mathrm{down}})_{2,2}=1,\; \text{all other entries }0. \label{eq:Wdown} \\
W_{\mathrm{up}}^{(\omega)}&:\R^k\to\R^d,\qquad (W_{\mathrm{up}})_{1,1}=\alpha_1(z),\;
  (W_{\mathrm{up}})_{2,2}=\alpha_2(z),\;
  \text{all other entries }0, \label{eq:Wup}
\end{align}
where $\alpha_1(z)$ and $\alpha_2(z)$ are to be determined.

Since $x_1(s),x_2(s)\geq 0$, and $W_{\mathrm{down}}$ selects these two coordinates, the input to the ReLU in~\eqref{eq:adapter_arch} is
$$W_{\mathrm{down}} x(s) = \bigl(x_1(s),\;x_2(s),\;0,\ldots,0\bigr)^\top\in\R^k.$$
Since all components are non-negative,
$\mathrm{ReLU}(W_{\mathrm{down}} x(s))=W_{\mathrm{down}} x(s)$. Thus the ReLU acts as the identity on these inputs, and the adapter
computes:
$$g_{\mathrm{adapter}}(x;\omega) = W_{\mathrm{up}}\,W_{\mathrm{down}}\,x(s) = \bigl(\alpha_1(z)\,x_1(s),\;\alpha_2(z)\,x_2(s),\;0,\ldots,0\bigr)^\top.$$
With the skip connection, the adapted features are:
$$\tilde{x}(s) = x(s)+g_{\mathrm{adapter}}(x(s);\omega) = \bigl((1+\alpha_1(z))\,\mathrm{ReLU}(s),\; (1+\alpha_2(z))\,\mathrm{ReLU}(-s),\;0,\ldots,0\bigr)^\top.$$
Setting $w=(1,1,0,\ldots,0)^\top$ and $b=0$:
$$f(s,z) = w^\top\tilde{x}(s) =  (1+\alpha_1(z))\,\mathrm{ReLU}(s)+ 1+\alpha_2(z))\,\mathrm{ReLU}(-s).$$
Using $\mathrm{ReLU}(s)-\mathrm{ReLU}(-s)=s$, we require:
$$1+\alpha_1(z)=z,\qquad 1+\alpha_2(z)=-z.$$
This gives $\alpha_1(z)=z-1$ and $\alpha_2(z)=-(z+1)$, both affine in $z$. The map $z\mapsto(\alpha_1(z),\alpha_2(z))=(z-1,-(z+1))$ is affine and therefore realizable by a ReLU network with linear output. All other entries of $W_{\mathrm{down}}$ are context-independent constants
(0 or 1), also trivially realizable.

Therefore $f(s,z)=z\cdot\mathrm{ReLU}(s)-z\cdot\mathrm{ReLU}(-s)=sz=f^*(s,z)$,
and $f^*\in\Hhyperbn$.

Combining Steps~1 and~2: $f^*\in\Hhyperbn\setminus\mathcal{H}_{\mathrm{concat}}$, hence $\Hhyperbn\not\subseteq\mathcal{H}_{\mathrm{concat}}$.
\end{proof}

Theorem~\ref{thm:bottleneck-separation} shows that the separation
$\Hhyperbn\not\subseteq\mathcal{H}_{\mathrm{concat}}$ holds. The proof works precisely because the witness requires only rank~2, and the bottleneck ReLU acts as the identity on the relevant non-negative inputs. In general, the bottleneck ReLU can implement additional nonlinear transformations that the linear adapter cannot, so $\Hhyperbn$ can represent functions beyond those in $\Hhyperlin$ (provided the rank-$k$ bottleneck suffices for the relevant subspace). While the blanket containment $\Hhyperlin\subseteq\Hhyperbn$ fails, the separation from $\mathcal{H}_{\mathrm{concat}}$ transfers because the specific witness construction fits within the bottleneck's capacity.

Table~\ref{tab:consolidatedAER} shows that DMA*-SH outperforms the standard concatenation baseline Concat by 11.5\% on eval-out, consistent with the architectural advantage identified by our results.

The separation results support a plausible mechanism for improved generalization. Hypernetworks can represent multiplicative interactions (including bilinear terms) exactly in simple cases, while concatenation MLPs typically approximate them via increasingly fine CPWL partitions. This approximation error may compound over trajectories, particularly in non-overlapping contexts where optimal policies differ drastically across modes.

\begin{remark}[Hypernetwork advantage for actuator inversion]\label{rem:expressiveness}
Actuator inversion provides a concrete illustration of the expressive gap (Definition~\ref{def:OverlapNonoverlap}). Suppose the environment contains a latent binary context $c \in \{-1, +1\}$, and the optimal policy satisfies $\pi^\star(s, c) = c \cdot \pi_{\mathrm{base}}(s)$, a sign flip in the action space. Since the agent does not observe $c$ directly, DMA*-SH infers a continuous embedding $z_t = g_\phi(\tau_t^c)$ whose values cluster into two well-separated regions corresponding to the two actuator modes.

A hypernetwork can map these two embedding regions to adapter parameters $\omega^{(+)}$ and $\omega^{(-)}$ such that the induced mean actions satisfy $\mu_{\omega^{(-)}}(s) \approx -\mu_{\omega^{(+)}}(s)$. The adapter can induce opposite-signed mean actions by changing the effective mapping from trunk features to action means between the two modes (through the adapter–head composition), i.e., producing adapted features $\tilde{x}$ such that $A\tilde{x}^{(-)} \approx -A\tilde{x}^{(+)}$ for the linear head matrix $A$. 
More generally, whenever the hypernetwork generates weights for a linear adapter layer, and these weights depend locally affinely on $z$, the resulting map is locally bilinear in $(x, z)$. By Theorem~\ref{thm:hyper_expressiveness}, such bilinear structure cannot be represented exactly by concatenation ReLU MLPs.

Concatenation-based policies \emph{can} approximate mode-switching behavior via CPWL partitioning: the network learns decision boundaries separating context regions and implements different linear maps on each piece. However, this indirect representation is more sensitive to encoder noise (small perturbations in $z$ near a decision boundary cause large policy changes) and less parameter-efficient than direct multiplicative modulation. Table~\ref{tab:aer} confirms that DMA*-SH substantially outperforms concatenation baselines on non-overlapping environments where this architectural advantage is most critical.
\end{remark}

\begin{remark}[Parameter, compute, and memory overhead]
\label{rem:param-compute-overhead}
DMA*-SH introduces additional parameters through the hypernetwork, but the generated modulation is deliberately low-rank and shared.
For trunk feature dimension $d$ and adapter bottleneck dimension $k$, applying the adapter adds $O(dk)$ compute per adapted module, rather than the $O(d^2)$ cost of a full context-conditioned layer.

The most relevant practical comparison is to DA, the closest context-aware hypernetwork baseline. Measured under the same hardware, batch size, precision, and update schedule, DMA*-SH uses only $4.1\%$ more parameters than DA, trains $43.5\%$ faster, and uses $24.4\%$ less peak memory. This indicates that sharing one context-to-adapter map across dynamics, actor, and critic is more efficient than maintaining separate hypernetwork pathways. Relative to simpler baselines without a hypernetwork, the cost is larger: compared with DMA, DMA*-SH uses $9.5\%$ more parameters, trains $79.2\%$ slower, and uses $21.7\%$ more peak memory.
We view this as the computational price of the dynamics-grounded operator modulation that drives the gains on non-overlapping contexts.

This overhead should be interpreted together with the expressivity result: concatenation-based ReLU networks can only approximate multiplicative context-dependent transformations through CPWL partitioning, whereas hypernetwork-conditioned adapters implement the relevant operator modulation directly. 
Empirically, DMA*-SH converges in comparable or fewer training steps than concatenation baselines (Figure~\ref{fig:returns}) while achieving superior zero-shot performance on non-overlapping contexts (Table~\ref{tab:aer-eval-out}), indicating that the expressive advantage outweighs the parameter overhead.
\end{remark}

\subsection{Policy Overlap: Overlapping and Non-Overlapping Contexts}\label{app:overnonover}

We categorize contextual families by the degree of \textit{policy overlap}: whether a single context-unaware policy can achieve near-optimal performance across all contexts. We formalize this using a normalized worst-case regret criterion:

\begin{definition}[Overlapping and Non-Overlapping Contexts]\label{def:OverlapNonoverlap}
Let $\mathcal{C}$ be a set of contexts for environment $E$. For each $c \in \mathcal{C}$, let $P^{c}$ be the transition dynamics and let $\pi^{*}_{c}$ be an optimal policy achieving return $J_c(\pi) = \mathbb{E}_{\pi, P^c}\!\left[ \sum_{t=0}^\infty \gamma^t r(s_t, a_t) \right]$. We normalize returns as $\bar J_c(\pi)\coloneqq ({J_c(\pi)-J^{\mathrm{lo}}_E})/({J^{\mathrm{hi}}_E-J^{\mathrm{lo}}_E})$, where $J^{\mathrm{lo}}_E$ and $J^{\mathrm{hi}}_E$ are fixed return bounds for $E$ with $J^{\mathrm{hi}}_E > J^{\mathrm{lo}}_E$ (Table~\ref{tab:envs}). We say $\mathcal{C}$ is \emph{non-overlapping} if there exists $\epsilon > 0$ such that for every context-unaware policy $\pi$,
$$
\max_{c\in\mathcal C}\bigl(\bar J_c(\pi^{*}_{c})-\bar J_c(\pi)\bigr)\ge \epsilon.
$$
Otherwise, $\mathcal{C}$ is \emph{overlapping}.
\end{definition}

A policy is \emph{context-unaware} if it conditions only on the observed state (e.g., $\pi(a\mid s)$) and does not take the context (or any inferred context representation) as input. Intuitively, in \textit{overlapping} families there exist context-unaware policies whose worst-case normalized regret can be made arbitrarily small, i.e., a single robust policy can approach per-context optimal performance across the family. In \textit{non-overlapping} families, optimal behaviors are mutually incompatible: every single context-unaware policy incurs a regret bounded away from zero in at least one context, and explicit context conditioning is necessary to be near-optimal across modes.
Definition~\ref{def:OverlapNonoverlap} is conceptual; our overlapping and non-overlapping labels in Table~\ref{tab:envs_summary} are operational, based on empirical behavior of strong context-unaware baselines and actuator-inversion incompatibility.

\begin{remark}[Non-overlapping functional modes in practice]\label{rem:Nonovpract}
Reversed wiring, swapped control channels, coordinate-frame mismatches, discrete gear-ratio shifts, and sim-to-real actuator remappings are natural examples of operator-switched dynamics that can exhibit non-overlap in the sense of Definition~\ref{def:OverlapNonoverlap}. See also Section~\ref{sec:Compbench}.
\end{remark} 

We use actuator inversion as the canonical generator of non-overlapping structure (in the sense of Definition~\ref{def:OverlapNonoverlap}) as it induces incompatible optimal policies across modes. Formally:
\begin{definition}[Actuator inversion]\label{def:actInv}
Assume the action space is sign-symmetric so that $c\cdot a \in \mathcal{A}$ for all $a\in\mathcal{A}$ and $c\in\{\pm 1\}$. A binary context $c\in\{\pm 1\}$ defines actuator-inverted dynamics by
$$
P^c(s_{t+1}\mid s_t,a_t)=P(s_{t+1}\mid s_t, c\cdot a_t),
$$
where $P$ denotes the nominal dynamics. We focus on tasks where this inversion induces qualitatively different optimal control laws across $c=\pm 1$, for example when the optimal mean action satisfies a sign-flip relation across contexts: $\pi^*(s, -1) \approx -\pi^*(s, +1)$.
\end{definition}

In the task class considered above, the optimal policies across modes often differ approximately by a sign flip. Hypernetwork-conditioned adapters can represent such mode-dependent transformations directly, whereas concatenation-based methods must approximate them via CPWL partitioning (Theorem~\ref{thm:hyper_expressiveness}), a representation that is less robust when the inferred context $z_t$ is noisy.

We examine how overlapping and non-overlapping context structures influence task difficulty and the stability of learned policies.

\subsubsection{Context-Aware (e.g., Concat/DA Baselines)}\label{app:overnonover2}

\begin{itemize}[leftmargin=*]
    \item \textit{Overlapping} contexts (\textbf{Easy})

    The policy class $\Pi_{\text{aware}} = \{\pi(a \mid s, c)\}$ is the set of all \textit{context-aware} policies that depend on the ground-truth context $c$. Since the functions $\pi^*(s, c)$ are similar for different $c$, the agent can smoothly vary its behavior based on $c$. The complexity is effectively that of $|\mathcal{C}_{\text{train}}|$ separate policies, but shared structure across contexts can facilitate learning and enable generalization to $\mathcal{C}_{\text{eval-out}}$ via continuity.

    \item \textit{Non-overlapping} contexts 
    (\textbf{Solvable with context})

    When the agent has access to the ground-truth binary context $c \in \{\pm 1\}$, it can directly gate between mode-specific behaviors. In actuator inversion, the optimal policies across modes often satisfy a sign relation $\pi^*(s,-1) \approx -\pi^*(s,+1)$, which can be represented either by explicit gating on $c$ or by mapping $c$ to distinct adapter parameters. Hypernetwork-conditioned adapters provide a parameter-efficient way to implement mode-dependent transformations, but the key difficulty in our setting arises when the context must be inferred noisily (Section~\ref{app:overnonover3}).
\end{itemize}

\subsubsection{Context-Unaware (e.g., Standard DR with Markovian Policy)}\label{app:overnonover1}
\begin{itemize}[leftmargin=*]
    \item \textit{Overlapping} contexts (\textbf{Solvable})
    
    The policy class $\Pi_{\text{unaware}}$ is the set of all \textit{context-unaware} Markovian policies $\pi(a\mid s)$ that do not receive the context $c$, an inferred context embedding $z_t$, or trajectory history $\tau_t$ as input. If the contexts are overlapping, the optimal policies $\pi^*_c$ for different $c$ are similar. Therefore, a single policy $\pi \in \Pi_{\text{unaware}}$ can exist that is near-optimal for all $c \in \mathcal{C}$. The agent is effectively solving a single, slightly broader MDP.
    
    \item \textit{Non-overlapping} contexts (\textbf{Inherently limited}) 

    By Definition~\ref{def:OverlapNonoverlap}, every context-unaware Markovian
    policy $\pi(a\mid s)$ must incur normalized regret at least $\epsilon$ in some
    context $c\in\mathcal C$. Thus a state-only policy is forced to compromise
    between incompatible context-specific control laws.
\end{itemize}

\begin{lemma}[Failure of Markovian DR under actuator inversion]
\label{lem:DRfail}
Assume a symmetric action space and contexts $c\in\{\pm1\}$ with dynamics
$P^c(s'\mid s,a)=P(s'\mid s,c\cdot a)$. For a stochastic policy $\pi(a\mid s)$, define the negated policy $-\pi$ by $(-\pi)(a\mid s) \doteq \pi(-a\mid s)$. Assume the return satisfies the actuator-inversion symmetry $J_{-1}(\pi)=J_{+1}(-\pi)$ for all state-only Markovian policies $\pi$. Suppose further that there exists $\Delta\geq 0$ such that, for every $\pi\in\Pi_{\text{unaware}}$,
\begin{align}
\label{eq:DRDelta}
J_{+1}(\pi)+J_{+1}(-\pi)\leq 2\Delta .
\end{align}
Then every every $\pi\in\Pi_{\text{unaware}}$ has domain-randomized objective
$$
\mathbb{E}_{c\sim\mathrm{Unif}\{\pm1\}}[J_c(\pi)]\leq \Delta .
$$
Consequently, if $\Delta$ is small compared with the per-context optima, then
Markovian DR cannot achieve high average return.
\end{lemma}

Condition~\eqref{eq:DRDelta} is a sufficient condition for failure of Markovian DR. It formalizes the intuition that a policy and its negation cannot both make progress when the same action has opposite effects across contexts. This condition holds when reversing effective actuation reverses task progress, as in goal-reaching under actuator inversion. The empirical failure of DR on non-overlapping AIB environments (Table~\ref{tab:aer}) is consistent with this analysis.
Lemma~\ref{eq:DRDelta} does not rule out recurrent or transformer DR policies that infer the hidden context from trajectory history; those are addressed empirically by the Amago baselines; see Appendix~\ref{sec:TxrecurDMASH} and Figure~\ref{fig:iqm-hyper}.

\begin{proof}
For any fixed $\pi\in\Pi_{\mathrm{unaware}}$,
\begin{align*}
J_{\mathrm{DR}}(\pi)
&= \tfrac12\bigl(J_{+1}(\pi) + J_{-1}(\pi)\bigr)
= \tfrac12\bigl(J_{+1}(\pi) + J_{+1}(-\pi)\bigr)
\le \tfrac12\cdot 2\Delta = \Delta,
\end{align*}
where the second equality uses the actuator-inversion symmetry and the inequality
uses condition~\eqref{eq:DRDelta}. Since the bound holds for every such $\pi$,
it also holds after maximizing over $\Pi_{\mathrm{unaware}}$.
\end{proof}

\begin{remark}[Context-unaware policies are epistemic POMDP solvers]
When the policy is context-unaware (as in Markovian DR), the problem becomes an \textit{epistemic POMDP}~\citep{ghosh2021generalization}: the true state is $(s_t, c)$, but the agent only observes $s_t$ and must implicitly maintain a belief over the hidden context $c$. Thus, unknown contexts induce partial observability even when the raw state is fully observed. For overlapping contexts (e.g., small changes in mass or friction), the dynamics $P^c$ vary smoothly. Small belief errors lead to small prediction errors, so the induced belief-MDP remains easy to optimize. For non-overlapping contexts (e.g., actuator inversion $c = \pm 1$), the dynamics for different contexts are mutually incompatible. Even slight uncertainty over $c$ yields drastically different predictions for $s_{t+1}$. 
The context-averaged one-step kernel
$$
\bar{P}(s_{t+1} \mid s_t, a_t) = \mathbb{E}_{c \sim \mathcal{C}_{\text{train}}}\!\left[ P^c(s_{t+1} \mid s_t, a_t) \right],
$$
can be mixture-like (often multimodal or high-variance).
A context-unaware policy is therefore forced to average over contradictory behaviors, often producing low return. Maintaining a high-confidence belief under such conditions requires a sharp separation in the agent’s internal representation, a representation that is both difficult to learn and highly sensitive to noise, leading to higher optimization variance and poorer generalization.  

Providing the agent with an accurately inferred context signal $z_t = g_\phi(\tau_t^c)$ (as in DMA*-SH) sidesteps the epistemic POMDP problem: the policy can condition directly on the correct mode instead of hedging across incompatible ones. This explains the dramatic performance gap of DMA*-SH on non-overlapping benchmarks (Table~\ref{tab:aer}), while the advantage often disappears on overlapping benchmarks where even context-unaware baselines can succeed.
\end{remark}

\subsubsection{Context-Inferred (Our Method, DMA*-SH)}\label{app:overnonover3}

\begin{itemize}[leftmargin=*]
    \item \textit{Overlapping} contexts (\textbf{Moderately difficult})

    The \textit{context-inferred} policy class $\Pi_{\text{inferred}} = \{\pi(a \mid s, z)\}$ depends on the inferred context $z$.
    The agent must solve two coupled problems: 
    \begin{enumerate}
        \item \emph{Context inference}: infer $z$ from a window of past $K$ transitions $\tau = \{(s_k, a_k, \delta s_{k+1})\}$ via the encoder $g_\phi$, and
        \item \emph{Control}: learn the policy $\pi(s, z)$.
    \end{enumerate}
    Inference difficulty scales inversely with context distinguishability. Since the dynamics differ only mildly across contexts, the inferred representation $z$ may be noisy or weakly informative. However, the control problem is comparatively easier: small errors in $z$ induce only small policy deviations, so errors degrade performance smoothly.

    \item \textit{Non-overlapping} contexts (\textbf{Very difficult})

    In our benchmark family, this is the most brittle setting for joint inference and control.  
    Non-overlapping contexts provide strong statistical signals for inference (high Informativeness, e.g., large $I(\tau; c)$), so in principle the encoder can recover $c$ from few transitions. In practice, however, even tiny inference errors are catastrophic: misclassifying $c=+1$ as $c=-1$ induces the \emph{opposite} control law, and the agent immediately fails. Learning is therefore brittle unless the encoder $g_\phi(\tau)$ produces accurate embeddings on most episodes.
    This creates a difficult credit-assignment loop during joint training.

    Encoder imprecision may arise from finite window size $K$ (partial observability), stochasticity in $P^c$ (e.g., sensor noise), or approximation limits of $g_\phi$. In non-overlapping regimes, such small errors are amplified severely in RL performance (e.g., through error propagation in value targets or large policy regret), since the failure modes are binary with no ``graceful'' degradation. The brittleness is worse for concatenation-based baselines, which must learn hard boundaries in their inputs, whereas hypernetwork-conditioning (as in DMA*-SH) effectively captures the multiplicative structure of actuator inversion (Theorem~\ref{thm:hyper_expressiveness}). 
\end{itemize}

\begin{remark}[The smoothness inductive bias of latent-variable models in meta-RL]\label{rem:VariBAD}
In variational latent-variable meta-RL methods for POMDPs (e.g., VariBAD \citep{zintgraf2020varibad}), the agent infers a latent context $z$ from experience and regularizes this inference with an explicit prior.
Typically, the optimization involves an ELBO of the form
$$
\mathcal{L}_{\mathrm{ELBO}} = \mathbb{E}_{q_\phi(z \mid \tau)}[\log p_\theta(\tau \mid z)] - \mathrm{KL}\!\left(q_\phi(z \mid \tau)\;\|\;p(z)\right),
$$
where both the posterior $q_\phi(z \mid \tau)$ and the prior $p(z)$ are \textit{unimodal} Gaussians.  
The KL term encourages $q_\phi(z\mid\tau)$ to stay close to the prior $p(z)$, which in common instantiations is a unimodal Gaussian (e.g., $p(z)=\mathcal{N}(0,I)$). 

In actuator inversion, good representations benefit from mapping the two modes to well-separated latent regions (e.g., approximately opposite signs). A unimodal Gaussian prior places most mass near the origin, so KL regularization biases the encoder toward posteriors that allocate non-negligible probability to intermediate latent values that do not correspond cleanly to either mode.
This mismatch can yield unstable training signals: small errors in $z_t$ can induce the wrong action sign and large performance drops. This is consistent with the empirical failures of VariBAD on ODE and DI tasks (Figure~\ref{fig:returns_varibad}).

In contrast, DMA*-SH avoids smooth latent priors by representing context through multiplicative hypernetwork modulation. Given an inferred embedding $z_t$, the hypernetwork generates adapter weights $\omega=h_\eta(z_t)$. This allows the policy and critic to implement sharp mode-dependent transformations, e.g., approximating a sign flip between the two mode clusters,
$$
\pi_{\xi,\omega}(s) \approx - \pi_{\xi,\omega'}(s),
$$
which matches the multiplicative structure of actuator inversion. These multiplicative interactions provide an inductive bias that is often helpful in non-overlapping regimes where ELBO-regularized latent models struggle empirically.
\end{remark}

\subsection{Informativeness, Variability, and Gradient Variance}
\label{app:varcompress}

We analyze the geometry of learned context representations using two complementary diagnostics: \emph{Variability} (representation spread) and \emph{Informativeness} (mutual information between representation and context). We then link Variability to learning stability by deriving a policy-gradient variance bound controlled by Variability. For binary--continuous contexts, this variance-based lens also motivates an \emph{implicit structural information bottleneck} interpretation: preserving mode information while selectively compressing within-mode variation.

\subsubsection{Decomposing Variability into Noise, Compression, and Separation} \label{app:VariabilityTh}

We relate the Variability of inferred embeddings $z$ 
(Section~\ref{sec:variability}) to the geometry of the learned representation space. In non-overlapping AIB tasks, the context can be factored as $C=(S,U)$ 
where $S\in\{\pm1\}$ is the actuator-inversion mode and $U$ is a continuous parameter (e.g., mass in DI). Below, we treat $(S,U,\mathcal{T},Z)$ as random variables under the rollout distribution induced by a fixed policy (i.e., at a fixed training iterate), with lowercase denoting realizations (e.g., $\tau$ is a realization of $\mathcal{T}$ and $z$ of $Z$). $\mathcal{T}$ denotes a trajectory window of length $K$ and $Z \doteq g_\phi(\mathcal{T})\in\mathbb{R}^{d_z}$ is the inferred embedding.

Fix one dataset split $M\in\{\text{train},\text{eval-in},\text{eval-out}\}$, and let $\mathcal{D}_M$ denote the induced empirical distribution over trajectory windows obtained by pooling all contexts in that split. Each window $\mathcal{T}\sim\mathcal{D}_M$ is generated by the policy interacting with the CMDP under ground-truth context $(S,U)$. All expectations below are with respect to this split-induced joint distribution. Variability is the average per-coordinate variance of embeddings pooled across the split:

\begin{definition}[Variability]\label{def:Variability}
The Variability of split $M$ is
\begin{align}
\label{eq:variability_def}
\mathrm{Variability}(M) \doteq \frac{1}{d_z}\sum_{i=1}^{d_z}\mathrm{Var}[Z_i] 
= \frac{1}{d_z}\operatorname{tr}\!\big(\mathrm{Cov}[Z]\big).
\end{align}
\end{definition}

\begin{theorem}[Variability decomposition for binary-continuous contexts]\label{thm:var_decomp_SU}
Let $C=(S,U)$ be the ground-truth context with $S\in\{\pm1\}$ binary and $U$ a continuous (possibly vector-valued) random variable. Let $Z=g_\phi(\mathcal{T})\in\mathbb{R}^{d_z}$ be the embedding of a trajectory window $\mathcal{T}\sim\mathcal{D}_M$. Assume $\mathbb{E}\|Z\|_2^2<\infty$. Define $\bar z(s,u)\doteq\mathbb{E}[Z\mid S=s,U=u]$, $\mu_s\doteq\mathbb{E}_{U\mid S=s}[\bar z(s,U)]$, and $p\doteq\mathbb{P}(S=+1)$. Then
\begin{align}\label{eq:thm_decomp}
\mathrm{Variability}(M) &= \underbrace{\tfrac{1}{d_z}\mathbb{E}_{S,U}\!\Big[\operatorname{tr}\big(\mathrm{Cov}(Z\mid S,U)\big)\Big]}_{\text{$\mathrm{Term~1}$: within-context noise}}\nonumber\\
&+\underbrace{\tfrac{1}{d_z}\operatorname{tr}\!\Big(\mathbb{E}_{S}\big[\mathrm{Cov}_{U\mid S}(\bar z(S,U))\big]\Big)}_{\text{$\mathrm{Term~2}$: within-mode variation (continuous)}}\nonumber\\
&+\underbrace{\tfrac{p(1-p)}{d_z}\|\mu_+-\mu_-\|_2^2}_{\text{$\mathrm{Term~3}$: between-mode separation (binary)}}.
\end{align}
Moreover, if $u\mapsto\bar z(s,u)$ is $L_s$-Lipschitz on $\mathrm{supp}(U\mid S=s)$, then
\begin{align}\label{eq:lipschitz_bound}
\frac{1}{d_z}\operatorname{tr}\!\Big(\mathrm{Cov}_{U\mid S=s}(\bar z(s,U))\Big)
\le\frac{L_s^2}{d_z}\,\mathbb{E}\!\left[\|U-\mathbb{E}[U\mid S=s]\|_2^2\mid S=s\right].
\end{align}
If, additionally, $\bar z(s,\cdot)$ is differentiable on a convex set $\mathcal{U}_s\supseteq \mathrm{supp}(U\mid S=s)$, then $L_s\le\sup_u\|\partial\bar z(s,u)/\partial u\|_{\mathrm{op}}$.
\end{theorem}
\begin{remark}[Geometric Intuition]\label{rem:Varbinterpret}
The decomposition~\eqref{eq:thm_decomp} is a statistical identity relating total embedding variance to three interpretable components. Its validity is independent of the encoder $g_{\phi}$ or the learning algorithm, and holds for any random variable $Z$ with the $(S,U,Z)$ joint structure. $\mathrm{Term~1}$ represents the \textit{representation noise floor}, $\mathrm{Term~2}$ quantifies the \textit{within-mode variation} induced by $U$ (and hence the extent to which a representation does \emph{not} collapse within-mode differences), and $\mathrm{Term~3}$ measures the \textit{separation} between binary actuator modes. Later, in Theorem~\ref{thm:pg_var_bound}, we formalize how $\mathrm{Variability}(M)$ (and hence its decomposition) enters an upper bound on policy-gradient variance.
\end{remark}

Our proofs rely on the law of total covariance. For any random vector $X$ and conditioning random variable $Y$:
\begin{align}\label{eq:total_cov}
\mathrm{Cov}(X) = \mathbb{E}_Y\big[\mathrm{Cov}(X \mid Y)\big] + \mathrm{Cov}_Y\big(\mathbb{E}[X \mid Y]\big).
\end{align}
Furthermore, we utilize the following independent-copy identity for the trace of the covariance: If $X \in \mathbb{R}^d$ is a random vector with finite second moments ($\mathbb{E}\|X\|_2^2 < \infty$), and $X'$ is an independent copy of $X$, then
\begin{equation}\label{eq:independent_copy_identity}
\operatorname{tr}(\mathrm{Cov}(X)) = \mathbb{E}\|X-\mathbb{E}X\|_2^2 = \tfrac{1}{2} \mathbb{E}\big[\|X - X'\|_2^2\big].
\end{equation}
To verify this, note that since $X'$ is independent of $X$ and identically distributed:
\begin{align*}
\mathbb{E}\big[\|X - X'\|_2^2\big] 
&= \mathbb{E}\big[\|X\|_2^2\big] - 2\mathbb{E}[X]^\top\mathbb{E}[X'] + \mathbb{E}\big[\|X'\|_2^2\big]\\ 
&= 2\mathbb{E}\big[\|X\|_2^2\big] - 2\|\mathbb{E}[X]\|_2^2\\
&= 2\operatorname{tr}(\mathrm{Cov}(X)).
\end{align*}
$\|\cdot \|_{\mathrm{op}}$ denotes the induced $L_{2}$ operator norm, and $\mathrm{supp}(P)$ denotes the support of the distribution $P$.

\begin{proof}
All expectations and covariances are under the distribution $\mathcal{T}\sim\mathcal{D}_M$ and induced $(S,U)$.

Apply \eqref{eq:total_cov} to $Z$ conditioned on $(S,U)$.
With $\bar z(S,U)=\mathbb{E}[Z\mid S,U]$,
\begin{align*}
\mathrm{Cov}(Z)=\mathbb{E}_{S,U}\big[\mathrm{Cov}(Z\mid S,U)\big]+\mathrm{Cov}_{S,U}\big(\bar z(S,U)\big).
\end{align*}
Taking $\frac{1}{d_z}\operatorname{tr}(\cdot)$ and using $\operatorname{tr}(\mathbb{E}[A])=\mathbb{E}[\operatorname{tr}(A)]$,
\begin{align}\label{eq:var_step1}
\mathrm{Variability}(M)=\frac{1}{d_z}\mathbb{E}_{S,U}\!\Big[\operatorname{tr}\big(\mathrm{Cov}(Z\mid S,U)\big)\Big]+\frac{1}{d_z}\operatorname{tr}\!\Big(\mathrm{Cov}_{S,U}\big(\bar z(S,U)\big)\Big).
\end{align}
The first term in~\eqref{eq:var_step1} is $\mathrm{Term~1}$ in~\eqref{eq:thm_decomp}. 

Apply \eqref{eq:total_cov} to $\bar z(S,U)$ conditioned on $S$:
\begin{align*}
\mathrm{Cov}_{S,U}\big(\bar z(S,U)\big)=\mathbb{E}_S\big[\mathrm{Cov}_{U\mid S}\big(\bar z(S,U)\big)\big]+\mathrm{Cov}_S(\mu_S),
\end{align*}
where $\mu_S=\mathbb{E}_{U\mid S}[\bar z(S,U)]$. Note that for each $s$, $\mu_s=\mathbb{E}_{U\mid S=s}[\bar z(s,U)]
=\mathbb{E}[Z\mid S=s]$ by iterated expectation. 
Taking $\frac{1}{d_z}\operatorname{tr}(\cdot)$ gives $\mathrm{Term~2}$ in~\eqref{eq:thm_decomp} plus the binary term.

For $S\in\{+1,-1\}$ with $\mathbb{P}(S=+1)=p$, let $\mu_\pm=\mu_{S=\pm1}$ and $m=p\mu_++(1-p)\mu_-$. Then $\mu_+-m=(1-p)(\mu_+-\mu_-)$ and $\mu_--m=-p(\mu_+-\mu_-)$, so
\begin{align*}
\mathrm{Cov}_S(\mu_S)&=p(\mu_+-m)(\mu_+-m)^\top+(1-p)(\mu_--m)(\mu_--m)^\top\\
&=\big[p(1-p)^2+(1-p)p^2\big](\mu_+-\mu_-)(\mu_+-\mu_-)^\top\\
&=p(1-p)(\mu_+-\mu_-)(\mu_+-\mu_-)^\top.
\end{align*}
Hence $\frac{1}{d_z}\operatorname{tr}(\mathrm{Cov}_S(\mu_S))=\frac{p(1-p)}{d_z}\|\mu_+-\mu_-\|_2^2$, which is $\mathrm{Term~3}$ in~\eqref{eq:thm_decomp}.

\paragraph{Lipschitz bound.}
Fix $s$ and let $f(u)=\bar z(s,u)$. By assumption, $f$ is $L_s$-Lipschitz on $\mathrm{supp}(U\mid S=s)$. Thus for all $u,v$, $\|f(u)-f(v)\|_2 \le L_s\|u-v\|_2$. 
Let $U'$ be an independent copy of $U$ under $U\mid S=s$. Applying the independent-copy identity~\eqref{eq:independent_copy_identity} and the Lipschitz property, we obtain:
\begin{align*}
\operatorname{tr}\!\Big(\mathrm{Cov}_{U\mid S=s}(f(U))\Big)
&= \tfrac{1}{2} \mathbb{E}\!\left[\|f(U)-f(U')\|_2^2\mid S=s\right] \\
&\le \frac{L_s^2}{2}\,\mathbb{E}\!\left[\|U-U'\|_2^2\mid S=s\right]
= L_s^2\,\mathbb{E}\!\left[\|U-\mathbb{E}[U\mid S=s]\|_2^2\mid S=s\right].
\end{align*}
Dividing by $d_z$, we obtain \eqref{eq:lipschitz_bound}.
When $\bar z(s,\cdot)$ is differentiable on a convex $\mathcal U_s\supseteq \mathrm{supp}(U\mid S=s)$, the mean value inequality implies
$L_s \le \sup_{u\in\mathcal U_s}\big\|\partial \bar z(s,u)/\partial u\big\|_{\mathrm{op}}$.
\end{proof}

\textbf{Within-mode compression and between-mode separation.}
Theorem~\ref{thm:var_decomp_SU} decomposes $\mathrm{Variability}(M)$ into three nonnegative contributions with distinct geometric meanings:
\begin{itemize}[leftmargin=*]
    \item \textbf{$\mathrm{Term~1}$: within-context noise (window-level spread)}. The first term $\tfrac{1}{d_z}\mathbb{E}_{S,U}\big[\operatorname{tr}\big(\mathrm{Cov}(Z\mid S,U)\big)\big]$ measures the average dispersion of embeddings across different trajectory windows at \emph{fixed} ground-truth context $(S,U)$. It captures stochasticity arising from partial observability, rollout randomness, finite window length, and encoder noise. So this term is the \emph{noise floor} that blurs any geometric organization induced by $\bar z(S,U)$.   Operationally, reducing $\mathrm{Term~1}$ can improve consistency of the context signal across trajectory windows.
    
    \item \textbf{$\mathrm{Term~2}$: within-mode mean-variation (compression of $U$)}. The second term $\frac{1}{d_z}\operatorname{tr}\big(\mathbb{E}_{S}\big[\mathrm{Cov}_{U\mid S}(\bar z(S,U))\big]\big)$ quantifies how the mean embeddings $\bar z(S,U)$ move as the continuous parameter $U$ varies \emph{within} each actuator mode $S$. 
    By~\eqref{eq:lipschitz_bound}, small Lipschitz sensitivity $L_s$ of $u\mapsto \bar z(s,u)$ guarantees that the within-mode mean-variation is controlled by $L_s^2$ times the intrinsic spread of $U\mid S=s$. Intuitively, when $L_s$ is small, varying $u$ within a fixed mode $s$ can shift the mean embedding $\bar z(s,u)$ only marginally, so the mean representation is nearly ``flat'' in the $U$-direction. In this sense, \textit{compressing} $U$ amounts to making the within-mode mean embedding insensitive to changes in $u$.
    Separately, Proposition~\ref{prop:boundary} shows that perfect information-theoretic compression $I(Z;U\mid S)=0$, is sufficient but not necessary for $\mathrm{Term~2}=0$.

    \item \textbf{$\mathrm{Term~3}$: between-mode mean-separation (separation across $S$)}. 
    The third term $\tfrac{p(1-p)}{d_z}\|\mu_+-\mu_-\|_2^2$  (where $\mu_s=\mathbb{E}_{U\mid S=s}[\bar z(s,U)]$) is the binary mode \textit{separation} induced by actuator inversion. Operationally, larger $\mathrm{Term~3}$ signifies stronger mode separation, increasing as the actuator modes occupy more distinct regions in the representation space. $\mathrm{Terms~2}$ and~$3$ can move in opposite directions: the encoder can simultaneously compress $U$ within each mode (small $\mathrm{Term~2}$) while separating the actuator modes (large $\mathrm{Term~3}$).
\end{itemize}

\subsubsection{Variability controls policy-gradient variance}\label{sec:varcompressPG}
We establish a formal connection between the Variability decomposition (Theorem~\ref{thm:var_decomp_SU}) and policy-gradient variance, and use this connection to resolve the apparent Informativeness--Variability paradox in Section~\ref{sec:variability} (Figure~\ref{fig:variability}).

\begin{theorem}[Policy-gradient variance bound controlled by Variability]\label{thm:pg_var_bound}
Let the ground-truth context be $C=(S,U)$ with $S\in\{\pm1\}$ and $U$ a continuous random variable. Let $Z = g_\phi(\mathcal T) \in\mathbb{R}^{d_z}$ be the inferred embedding and let $G\in\mathbb{R}^{d_\xi}$ denote a single-sample policy-gradient estimator. Assume:
\begin{enumerate}
\item[\textup{(A1)}] The embedding $Z$ and gradient $G$ have finite second moments: $\mathbb{E}\|Z\|_2^2 < \infty$ and $\mathbb{E}\|G\|_2^2 < \infty$.
\item[\textup{(A2)}] The conditional gradient variance is bounded: there exists $V_0 \ge 0$ such that
$\mathbb{E}\big[\operatorname{tr}(\mathrm{Cov}(G \mid Z))\big] \le V_0$.
\item[\textup{(A3)}] The conditional mean gradient $m(z) \doteq \mathbb{E}[G \mid Z = z]$ is $L_m$-Lipschitz on a compact set $\mathcal{Z}$ containing $\mathrm{supp}(Z)$.
\end{enumerate}
Then the total gradient variance satisfies:
\begin{align}\label{eq:pg_var_bound_main}
\operatorname{tr}(\mathrm{Cov}(G)) \le V_0 + L_m^2 \cdot d_z \cdot \mathrm{Variability}(M).
\end{align}
\end{theorem}

Under Assumptions (A1)--(A3), Theorem~\ref{thm:pg_var_bound} upper-bounds the variance of the single-sample policy-gradient estimator by an irreducible term $V_0$ (capturing conditional sampling noise given $Z$) plus a term proportional to $\mathrm{Variability}(M)$. These are regularity conditions under the split-induced distribution $\mathcal{D}_M$: (A1) is plausible since SimNorm constrains embeddings to a compact set \citep{obando2025simplicial} and SAC uses bounded actions with gradient clipping \citep{haarnoja2018soft}; (A2) captures conditional stochasticity of single-sample gradients; (A3) assumes bounded sensitivity of $m(z)$ to changes in $z$, a standard regularization perspective \citep{gouk2021regularisation}. Accordingly, we  treat $V_0$ and $L_m$ as problem-dependent constants. We use Theorem~\ref{thm:pg_var_bound} as an interpretive upper bound linking representation spread to policy-gradient variance, not as a tight characterization of SAC's optimization dynamics. Figure~\ref{fig:gradcon} reports proxy diagnostics consistent with lower gradient-norm variability under DMA*-SH.

\begin{proof}
All expectations and covariances are taken under the split-induced distribution $\mathcal T\sim\mathcal{D}_M$ and the induced random variables $(S,U,Z,G)$ (conditional on fixed current parameters).

By Assumption~(A1), $\mathrm{Cov}(G)$ is well-defined. Similarly, \(\mathrm{Cov}(m(Z))\) is well-defined because $\|m(Z)\|_{2}^{2}$ is integrable; specifically, by Jensen's inequality and Assumption~(A1), 
$$
\mathbb{E}\|m(Z)\|_2^2=\mathbb{E}\|\mathbb{E}[G\mid Z]\|_2^2\le \mathbb{E}\mathbb{E}[\|G\|_2^2\mid Z]=\mathbb{E}\|G\|_2^2<\infty.
$$
Applying the law of total covariance to $G$ conditioned on $Z$:
\begin{align}\label{eq:Gradvardecomp}
\mathrm{Cov}(G) &= \underbrace{\mathbb{E}[\mathrm{Cov}(G\mid Z)]}_{\text{irreducible sampling noise even if } Z \text{ is fixed}} 
    +
    \underbrace{\mathrm{Cov}(\mathbb{E}[G\mid Z])}_{\text{variance coming from } Z \text{ moving around}} \nonumber\\
    &= \qquad \,\,\quad \mathbb{E}\big[\mathrm{Cov}(G \mid Z)\big] \quad\quad\,\,\,\, + \qquad \,\,\quad \mathrm{Cov}(m(Z)).
\end{align}
Taking traces and using linearity:
$\operatorname{tr}(\mathrm{Cov}(G)) = \mathbb{E}\big[\operatorname{tr}(\mathrm{Cov}(G \mid Z))\big] + \operatorname{tr}(\mathrm{Cov}(m(Z)))$.

By Assumption~(A2):
\begin{align}\label{eq:step1_bound}
\operatorname{tr}(\mathrm{Cov}(G)) \le V_0 + \operatorname{tr}(\mathrm{Cov}(m(Z))).
\end{align}
Let $Z'$ be an independent copy of $Z$. 
Applying the independent-copy identity~\eqref{eq:independent_copy_identity} with $X = m(Z)$: 
$$
\operatorname{tr}(\mathrm{Cov}(m(Z))) = \tfrac{1}{2}\mathbb{E}\big[\|m(Z) - m(Z')\|_2^2\big].
$$
By the Lipschitz Assumption~(A3):
$\|m(Z) - m(Z')\|_2^2 \le L_m^2\|Z - Z'\|_2^2$, hence
$\operatorname{tr}(\mathrm{Cov}(m(Z))) \le \frac{L_m^2}{2}\mathbb{E}\big[\|Z - Z'\|_2^2\big]$. 

Since $\mathbb{E}\|Z\|_2^2<\infty$ by Assumption~(A1), we may apply \eqref{eq:independent_copy_identity} with $X=Z$:
\begin{align}\label{eq:lipschitz_contraction}
\operatorname{tr}(\mathrm{Cov}(m(Z))) \le L_m^2 \cdot \operatorname{tr}(\mathrm{Cov}(Z)).
\end{align}

Substituting into~\eqref{eq:step1_bound}:
\begin{align}\label{eq:step2_bound}
\operatorname{tr}(\mathrm{Cov}(G)) \le V_0 + L_m^2 \cdot\operatorname{tr}(\mathrm{Cov}(Z)).
\end{align}
By Definition~\ref{def:Variability},
$\operatorname{tr}(\mathrm{Cov}(Z))=d_{z}\cdot \mathrm{Variability}(M)$. Substituting into~\eqref{eq:step2_bound} yields the stated bound~\eqref{eq:pg_var_bound_main}.
\end{proof}

\subsubsection{The Informativeness--Variability paradox: When less information improves learning}
\label{sec:varcompressParadox}

Informativeness $I(Z;S,U)$ measures how much information $Z$ carries about the \textit{full} context $C=(S,U)$. In non-overlapping AIB tasks, only part of that context may be \emph{decision-relevant} for stable control and learning:
\begin{enumerate}[leftmargin=*]
\item The binary mode $S$ is decision-critical (wrong sign can catastrophically flip the control law).
\item The continuous coordinate $U$ (e.g., mass) may be weakly relevant or even largely irrelevant \emph{within each mode} over the benchmark range, especially if robust control can handle that variation.
\end{enumerate}
Intuitively, one might expect that embeddings capturing more information about the full context $(S,U)$ would yield better RL performance. Proposition~\ref{prop:paradox} shows this intuition can fail: increasing informativeness can increase gradient variance by encoding within-mode variation that is not needed for control. This creates a paradoxical regime where an embedding with strictly \textit{lower} informativeness can achieve a strictly \textit{tighter} policy-gradient variance bound and can therefore learn more stably.
Proposition~\ref{prop:paradox} formalizes this possibility: 

\begin{proposition}[The Informativeness--Variability paradox]\label{prop:paradox}
Let the setting be as in Theorems~\ref{thm:var_decomp_SU} and~\ref{thm:pg_var_bound}. Assume:
\begin{enumerate}
\item[\textup{(B1)}] The binary mode is recoverable from the embedding: $H(S \mid Z) = 0$.
\item[\textup{(B2)}] The embedding carries nontrivial information about the continuous parameter beyond the mode: $I(Z; U \mid S) > 0$.
\end{enumerate}
Define the mode-collapsed embedding $\widetilde{Z} \in \mathbb{R}^{d_z}$ by $\widetilde{Z} \doteq \mu_S$, where $\mu_s$ is as defined in Theorem~\ref{thm:var_decomp_SU}.
Then:
\begin{enumerate}[leftmargin=*]
    \item \emph{Informativeness strictly \textit{decreases}:} $I(\widetilde{Z}; S, U) < I(Z; S, U)$.
    \item \emph{Variability strictly \textit{decreases} by removing the within-mode term:} with $\widetilde{\bar z}(s,u)\doteq \mathbb{E}[\widetilde Z\mid S=s,U=u]=\mu_s$,
    \begin{align*}
    \operatorname{tr}(\mathrm{Cov}(\widetilde Z)) &= p(1-p)\|\mu_+ - \mu_-\|_2^2 \le \operatorname{tr}(\mathrm{Cov}(Z)),
    \end{align*}
    and the inequality is strict whenever $\mathrm{Term~1}$ or $\mathrm{Term~2}$ in~\eqref{eq:thm_decomp} is positive.
    \item \emph{The policy-gradient variance upper bound improves despite lower Informativeness:} If the policy-gradient estimator computed with $\widetilde Z$ satisfies the same assumptions of Theorem~\ref{thm:pg_var_bound} with the same constants $V_0$ and $L_m$, then
    $$
    \operatorname{tr}(\mathrm{Cov}(G_{\widetilde Z})) \le V_0 + L_m^2 \operatorname{tr}(\mathrm{Cov}(\widetilde Z)) < V_0 + L_m^2 \operatorname{tr}(\mathrm{Cov}(Z)),
    $$
    whenever the inequality in Part~2 is strict.
    Thus, a representation can be \emph{less informative about the full context} $(S,U)$ while yielding a \emph{strictly smaller} policy-gradient variance upper bound.
\end{enumerate}
\end{proposition}

Proposition~\ref{prop:paradox} demonstrates that the Informativeness--Variability paradox is mathematically possible: an embedding with strictly lower Informativeness can yield a strictly tighter gradient-variance bound. This provides one explanation consistent with our empirical observations in Section~\ref{sec:variability} (Figure~\ref{fig:variability}), where DMA*-SH achieves lower Informativeness than DMA* and DMA yet superior RL performance. The proposition constructs an idealized oracle embedding $\widetilde{Z} = \mu_S$ that perfectly collapses within-mode variation. We do not claim DMA*-SH learns this oracle exactly; rather, architectural constraints (SimNorm, shared hypernetwork, dynamics-only training) create pressure toward similar compression (Appendix~\ref{app:sib}). The empirical MI trajectories in Figure~\ref{fig:variability2} are consistent with this interpretation and with Assumptions (B1)--(B2): $I(Z;S)$ saturates early near $H(S)$ (high mode recoverability, equivalently small $H(S\mid Z)$), while $I(Z;U\mid S)$ remains nonzero but relatively reduced for DMA*-SH (partial within-mode compression).

\begin{proof}
\textbf{Part 1: Informativeness decreases.}
Since $\widetilde{Z} = \mu_S$ is a deterministic function of $S$, $I(\widetilde{Z}; U \mid S) = 0$. By the chain rule for mutual information, $I(\widetilde{Z}; S, U) = I(\widetilde{Z}; S) + I(\widetilde{Z}; U \mid S) = I(\widetilde{Z}; S) \le H(S)$. By Assumption~(B1), $H(S \mid Z) = 0$, hence $I(Z; S) = H(S) - H(S \mid Z) = H(S)$ and $I(Z; S, U) = I(Z; S) + I(Z; U \mid S) = H(S) + I(Z; U \mid S)$.
By Assumption~(B2), $I(Z; U \mid S) > 0$. So $I(\widetilde{Z}; S, U) \le H(S) < H(S) + I(Z; U \mid S) = I(Z; S, U)$.

\textbf{Part 2: Variability decreases.}
By construction, $\widetilde Z=\mu_S$ depends only on $S$, hence for all $(s,u)$,
$$
\mathrm{Cov}(\widetilde{Z} \mid S = s, U = u) = 0, \qquad \widetilde{\bar{z}}(s, u) \doteq \mathbb{E}[\widetilde{Z} \mid S = s, U = u] = \mu_s.
$$
Since $\widetilde{\bar{z}}(s, u) = \mu_s$ is constant in $u$, $\mathrm{Cov}_{U \mid S = s}(\widetilde{\bar{z}}(s, U)) = 0$. Applying the decomposition (Theorem~\ref{thm:var_decomp_SU}) to $\widetilde{Z}$:
$$
\operatorname{tr}(\mathrm{Cov}(\widetilde{Z})) = 0 + 0 + p(1-p)\|\mu_+ - \mu_-\|_2^2 = p(1-p)\|\mu_+ - \mu_-\|_2^2.
$$
Applying the same decomposition to $Z$:
\begin{align*}
\operatorname{tr}(\mathrm{Cov}(Z)) 
&= \mathbb{E}_{S,U}\Big[\operatorname{tr}\big(\mathrm{Cov}(Z \mid S, U)\big)\Big] + \operatorname{tr}\Big(\mathbb{E}_S\big[\mathrm{Cov}_{U \mid S}(\bar{z}(S, U))\big]\Big) + p(1-p)\|\mu_+ - \mu_-\|_2^2 \\
&= d_z \cdot \mathrm{Term~1} + d_z \cdot \mathrm{Term~2} + p(1-p)\|\mu_+ - \mu_-\|_2^2.
\end{align*}
Since $\mathrm{Term~1}$ and $\mathrm{Term~2}$ are non-negative, $\operatorname{tr}(\mathrm{Cov}(\widetilde{Z})) \le \operatorname{tr}(\mathrm{Cov}(Z))$, with strict inequality whenever 
$\mathrm{Term~1} >0$ or $\mathrm{Term~2} > 0$.

\paragraph{Part 3: Variance bound improves although Informativeness drops.}
By Theorem~\ref{thm:pg_var_bound}, the upper bound for $G_{\widetilde Z}$ is $V_0 + L_m^2 \operatorname{tr}(\mathrm{Cov}(\widetilde{Z}))$, and the upper bound for $G_Z$ is $V_0 + L_m^2 \operatorname{tr}(\mathrm{Cov}(Z))$. By Part~2, $\operatorname{tr}(\mathrm{Cov}(\widetilde{Z})) < \operatorname{tr}(\mathrm{Cov}(Z))$ whenever $\mathrm{Term~1}$ or $\mathrm{Term~2}$ is positive. Therefore:
$$
V_0 + L_m^2 \operatorname{tr}(\mathrm{Cov}(\widetilde{Z})) 
< V_0 + L_m^2 \operatorname{tr}(\mathrm{Cov}(Z)),
$$
proving that the policy-gradient variance upper bound strictly improves for $\widetilde{Z}$ despite its lower informativeness.
\end{proof}

\subsubsection{DMA*-SH as an Implicit Structural Information Bottleneck}\label{app:sib}

We formalize the notion of a \emph{structural information bottleneck} (SIB) for binary--continuous contexts, and explain how DMA*-SH can be viewed as \textit{implicitly} implementing an \textit{approximate SIB} via dynamics prediction under architectural constraints. We emphasize that SIB serves here as an interpretive framework, not an optimization target: DMA*-SH does \emph{not} explicitly optimize an SIB objective. Rather, the SIB perspective helps explain the emergent representation geometry analyzed in Appendix~\ref{sec:SIBemp}: In particular, the observed signatures (fast mode acquisition and reduced within-mode spread in certain tasks) are consistent with a capacity-constrained representation that prioritizes mode information $S$ over within-mode detail $U$ when this is sufficient for dynamics prediction and control.

\paragraph{Two complementary notions of compression.} We first formalize what we mean by ``compression''. 

\paragraph{(A) Variance: compression of within-mode mean dependence on $U$.} 
Recall the conditional mean embedding $\bar{z}(s, u) \doteq \mathbb{E}[Z \mid S = s, U = u]$ and mode-conditional mean $\mu_s \doteq \mathbb{E}_{U \mid S = s}[\bar{z}(s, U)]$. $\mathrm{Term~2}$ in Theorem~\ref{thm:var_decomp_SU} measures how much the conditional mean embedding $\bar z(s,u)$ varies with $U$ within each mode $s$. Small $\mathrm{Term~2}$ indicates that $\bar{z}(s, u)$ is approximately ``flat'' along $U$-directions within a fixed mode, a variance-based signature of compression. Moreover, Theorem~\ref{thm:var_decomp_SU} provides a quantitative sufficient condition: if $u \mapsto \bar{z}(s, u)$ is $L_s$-Lipschitz, then
\begin{align}\label{eq:lipschitz_bound_app}
\mathrm{Term~2} \le \frac{1}{d_z}\sum_{s \in \{\pm 1\}} \mathbb{P}(S = s) \cdot L_s^2 \cdot \operatorname{tr}\!\bigl(\mathrm{Cov}(U \mid S = s)\bigr).
\end{align}
Thus, small Lipschitz constants imply small within-mode mean variation.

\paragraph{(B) Information: compression of $U$ given $S$.}
A complementary notion is the conditional mutual information $I(Z; U \mid S)$. Smaller $I(Z; U \mid S)$ means that, after conditioning on the mode $S$, the representation $Z$ retains less information about the continuous parameter $U$.

\paragraph{Relationship between the two notions.}
Proposition~\ref{prop:boundary} below establishes that these two notions are related but not equivalent: perfect information-theoretic compression ($I(Z;U\mid S)=0$) implies zero within-mode mean variation ($\mathrm{Term~2}=0$), but the converse fails because $\mathrm{Term~2}$ captures only \emph{mean} dependence on $U$, not higher-moment dependence.

\begin{proposition}[Boundary behavior of within-mode variation]\label{prop:boundary}
Let the setting be as in Theorem~\ref{thm:var_decomp_SU}.
\begin{enumerate}
    \item If $I(Z; U \mid S) = 0$, then $\mathrm{Term~2} = 0$.
    \item The converse does not hold: $\mathrm{Term~2} = 0$ does not imply $I(Z; U \mid S) = 0$.
\end{enumerate}
\end{proposition}
\begin{proof}
\textbf{Part 1.}
By definition, $I(Z; U \mid S) = 0$ if and only if $Z \perp U \mid S$, meaning $p(Z \mid S = s, U = u) = p(Z \mid S = s)$ for all $s, u$. This implies $\bar{z}(s, u) = \mathbb{E}[Z \mid S = s, U = u] = \mathbb{E}[Z \mid S = s] = \mu_s$ for all $u$. Since $\bar{z}(s, u) = \mu_s$ is constant in $u$, we have $\mathrm{Cov}_{U \mid S = s}(\bar{z}(s, U)) = 0$ for each $s$. Averaging over $S$ and taking $\frac{1}{d_z}\operatorname{tr}(\cdot)$ yields $\mathrm{Term~2}=0$.

\textbf{Part 2.}
We construct a counterexample. Let $\bar{z}(s, u) = \mu_s$ be constant in $u$, but let the conditional variance depend on $u$. Suppose $U \sim \mathrm{Uniform}[1, 2]$ and $Z \mid (S = s, U = u) \sim \mathcal{N}(\mu_s, u^2 I_{d_z})$. Then $\bar{z}(s, u) = \mu_s$ is constant in $u$, so $\mathrm{Cov}_{U \mid S}(\bar{z}(s, U)) = 0$ and $\mathrm{Term~2} = 0$. However, since $p(Z \mid S = s, U = u)$ depends on $u$ through the variance, we have $Z \not\perp U \mid S$, hence $I(Z; U \mid S) > 0$.
\end{proof}

Thus $\mathrm{Term~2}$ is best interpreted as a \emph{mean-flatness} signature rather than a complete proxy for $I(Z;U\mid S)$.

\paragraph{Structural information bottleneck (SIB).} The classical information bottleneck (IB) seeks representations that preserve task-relevant information while compressing other variation~\citep{tishby2000information,alemi2017deep,banerjee2020variational}. In actuator-inversion regimes, the context $C=(S,U)$ has an asymmetric structure: flipping the mode $S$ typically changes the \emph{qualitative} controller structure (e.g., action sign), while $U$ modulates \emph{within-mode} quantitative details (e.g., gains). This motivates a conditional compression objective that prioritizes mode information while penalizing within-mode sensitivity to $U$.

\begin{definition}[Structural information bottleneck (SIB) for actuator inversion]\label{def:sib_formal}
Let $C=(S,U)$ with $S \in \{\pm 1\}$ binary and $U$ continuous and $Z=g_\phi(\mathcal T)$. 
The \emph{SIB objective} is
\begin{align}\label{eq:sib_objective}
\mathcal{L}_{\mathrm{SIB}}(Z) = I(Z; S) - \beta\, I(Z; U \mid S), \quad \beta > 0,
\end{align}
or in constrained form: maximize $I(Z;S)$ over $g_\phi$ subject to $I(Z;U\mid S)\le R$. We say that $Z$ implements an \emph{exact SIB} if $I(Z; S) = H(S)$ and $I(Z; U \mid S) = 0$, and an \emph{approximate SIB} if $I(Z; S)/H(S)$ (assuming $H(S)>0$, as in our non-degenerate splits) is high and $I(Z; U \mid S)$ is relatively low.
\end{definition}

The first term $I(Z; S)$ measures \emph{mode sufficiency}: how much information the representation retains about the binary mode. The second term $I(Z; U \mid S)$ measures \emph{within-mode compression}: how much information about $U$ is retained beyond mode identification. The SIB objective prioritizes mode information while penalizing within-mode encoding.

\paragraph{Implicit SIB via dynamics prediction and architectural constraints.}
DMA*-SH does not optimize~\eqref{eq:sib_objective} explicitly. The encoder $g_\phi$ is trained solely via the dynamics prediction loss~\eqref{eq:hnopt},
$$L_{\phi,\theta,\eta} = \left\lVert \delta \hat{s}_{t+1} - \delta s_{t+1} \right\rVert_2^2,$$
where $\delta s_{t+1}$ is the next-state change and $\delta \hat{s}_{t+1}$ is its prediction. This objective contains no information-theoretic regularizer.
Nevertheless, the combination of dynamics prediction with architectural constraints creates an \emph{implicit} bottleneck that can approximate the SIB structure. The architectural constraints in DMA*-SH are shown in Table~\ref{tab:SIB-constraints}.

\begin{table}[ht]
  \caption{Architectural constraint and effect on $Z$.}
  \label{tab:SIB-constraints}
  \centering
\begin{tabular}{ll}
\toprule
\textbf{Constraint} & \textbf{Effect} \\
\midrule
SimNorm & Constrains $Z$ to product of $L=2$ simplices of dimension $V=4$ \\
Low $d_z$ & Limited embedding dimensionality ($d_z = 8$) \\
LSTM bottleneck & Sequential processing compresses temporal information \\
Random masking & Forces robustness, prevents overfitting to specific features \\
\bottomrule
\end{tabular}
\end{table}

These constraints limit the effective capacity of the representation. Without additional assumptions, minimizing squared prediction error is \emph{not} generally equivalent to maximizing mutual information. We use the following only as an \emph{interpretive} lens: Under restricted representation classes, optimizing dynamics prediction tends to favor encoding those context factors whose recovery most reduces prediction error. In actuator-inversion environments, the mode $S$ typically governs the sign/compatibility of the action-to-dynamics response, whereas $U$ often modulates within-mode scaling. Under strong capacity constraints, it is therefore plausible that optimization pressure prioritizes (or learns earlier) the mode information yielding high $I(Z;S)$ while compressing fine-grained within-mode dependence on $U$, yielding relatively low $I(Z; U \mid S)$. The key insight is that DMA*-SH’s compression is not imposed by the loss, but by the architecture: the dynamics loss determines \emph{what} information is useful; the architectural constraints determine \emph{how much} can be retained. When mode information is more predictive for dynamics than within-mode variation, the resulting representation can be biased toward high $I(Z;S)$ and reduced $I(Z;U\mid S)$, thereby approximating the SIB objective in effect. 

\paragraph{SIB links representation geometry to policy-gradient variance.}
The SIB framework connects to our main theoretical results as follows. Theorem~\ref{thm:var_decomp_SU} decomposes total $\mathrm{Variability}(M)$ into $\mathrm{Term~1}$ (within-context noise), $\mathrm{Term~2}$ (within-mode mean-variation of $\bar z(s,u)=\mathbb{E}[Z\mid S=s,U=u]$ across $U\mid S$), and $\mathrm{Term~3}$ (between-mode mean-separation). 

By Proposition~\ref{prop:boundary}, $I(Z;U\mid S)=0$ implies $\mathrm{Term~2}=0$; more generally, reduced conditional informativeness about $U$ is compatible with (though does not in itself imply) smaller within-mode mean-variation. 
By Theorem~\ref{thm:pg_var_bound}, under Lipschitz regularity of $m(z) = \mathbb{E}[G \mid Z = z]$, the policy-gradient variance satisfies $\operatorname{tr}(\mathrm{Cov}(G)) \leq V_0 + L_m^2 d_z \cdot \mathrm{Variability}(M)$; since this bound depends on $\mathrm{Variability}(M)$, reducing $\mathrm{Term~2}$ while preserving $\mathrm{Term~3}$ tightens this bound, provided $\mathrm{Term~1}$ does not increase to offset the reduction. Proposition~\ref{prop:paradox} formalizes the resulting tradeoff: an embedding with lower total informativeness $I(Z; S, U)$ can achieve a tighter gradient-variance upper bound when the reduction comes from compressing within-mode structure rather than degrading mode information. 

In this sense, DMA*-SH can be viewed as approximately solving a \emph{structural IB} that prioritizes mode information while compressing within-mode variation in a way that benefits the policy-gradient variance bound.

\paragraph{Why DMA*-SH outperforms Concat on non-overlapping eval-out.} The SIB framework highlights three properties. To state them precisely, we distinguish between the data-generating process and what the agent observes. The latent context $C=(S,U)$ determines the MDP dynamics under which rollouts are generated. For fixed parameters, this induces the causal structure $(S,U)\to \mathcal{T} \to Z \to \omega$, where $Z=g_\phi(\mathcal{T})$ and $\omega=h_\eta(Z)$. Let $X$ denote additional non-context inputs to downstream modules (e.g., state for the policy or state--action pairs for the critic). Downstream modules receive $\omega$ together with $X$ without ever observing
$(S,U)$.

\begin{enumerate}[label=(\roman*),leftmargin=*,nosep]
    \item \textit{Mode sufficiency:} $H(S \mid Z) \approx 0$, i.e., the latent mode $S$ is approximately recoverable from $Z$.
    \item \textit{Within-mode compression:} $I(Z;U\mid S)$ is relatively small, i.e., $Z$ carries limited information about $U$ beyond what $S$ provides.
    \item \textit{Downstream context-channel restriction (architectural):} any downstream output $\widehat{Y}$ produced by a module that receives context only through $(Z,\omega)$ satisfies 
    $$\widehat{Y} \perp (S,U) \mid (Z, X).$$
    Consequently, any sensitivity to within-mode shifts in $U$ must be mediated through how $U$ affects the observed rollout distribution and the inferred embedding.
\end{enumerate}

Condition~(iii) is an architectural property of DMA*-SH: the shared adapter pathway ensures that downstream modules receive context only through $\omega=h_\eta(Z)$, with no explicit $(S,U)$ input. Conditions~(i) and~(ii) are learned properties; gradient detachment of $\omega$ in the RL losses ensures that the encoder is shaped only by dynamics prediction, creating pressure to encode features predictive for dynamics (primarily mode $S$) while architectural bottlenecks limit capacity for within-mode detail. Downstream outputs can still depend on $U$ marginally through the observed non-context inputs $X$ (e.g., the state), since the rollout distribution under $P^{(S,U)}$ depends on $U$. We therefore interpret (ii)--(iii) as a structural bias against \emph{explicit} and potentially brittle $U$-conditioning.

In non-overlapping tasks, the binary mode $S$ acts as a compatibility variable: correct behavior must commit to the appropriate mode, while in our benchmark ranges $U$ predominantly modulates within-mode quantitative details. The standard context-aware baseline Concat receives $(S,U)$ directly as input to the policy, so it has no architectural restriction analogous to (iii) and no intrinsic pressure to \emph{ignore $U$}. In practice, this can entangle policy/value features with fine-grained $U$-dependence learned on $\mathcal{C}_{\mathrm{train}}$, which may be brittle under the held-out split $\mathcal{C}_{\mathrm{eval\text{-}out}}$. Concat is thus context-aware but \emph{not context-selective}. DMA*-SH's architectural restriction (iii), combined with approximate satisfaction of (i)--(ii), provides a structural bias toward mode-based rather than fine-grained $U$-based conditioning, which we interpret as contributing to its improved robustness on non-overlapping $\mathcal{C}_{\mathrm{eval\text{-}out}}$
(Table~\ref{tab:aer-eval-out}).

\paragraph{Non-overlapping versus overlapping environments.}
The SIB interpretation applies most naturally to non-overlapping (actuator-inversion) environments where the binary mode $S$ has a discontinuous, dominant effect on dynamics. In overlapping environments with purely continuous context $C = (U_1, U_2)$ and no binary mode, there is no $\mathrm{Term~3}$ (between-mode separation) since no discrete mode exists. In such environments, compression can discard task-relevant variation, and the actuator-inversion-specific inductive bias provides less benefit. This is consistent with DMA*-SH achieving competitive but not dominant performance on overlapping environments (Table~\ref{tab:consolidatedAER}).

\subsubsection{Regime shifts and selective-compression signatures on $\mathcal{C}_{\mathrm{eval\text{-}out}}$}
\label{sec:SIBemp}
Figure~\ref{fig:variability2} reports $I(Z;S)$, $I(Z;U\mid S)$, Variability, and returns on the held-out split $\mathcal{C}_{\mathrm{eval\text{-}out}}$ for three non-overlapping environments: DI, Cartpole, and Reacher (Hard). Mutual information (MI) is measured in nats; for a binary mode $S\in\{\pm1\}$ with balanced splits, the maximum is $I(Z;S)=H(S)=\ln 2 \approx 0.693$ nats. We treat MI estimates as comparative diagnostics (relative ordering across methods under the same estimator and sample budget), not calibrated absolute values.

Across all three environments, the evolution exhibits a consistent qualitative ordering:
\begin{enumerate}[leftmargin=*]
    \item \textbf{Rapid mode acquisition.} $I(Z;S)$ increases quickly and enters a near-saturated regime (relative to $H(S)$) early in optimization, indicating rapid acquisition of mode information. In Cartpole and Reacher, $I(Z;S)$ approaches values close to $\ln 2$, whereas in DI the attained level is lower but still substantial.
    \item \textbf{Slower within-mode refinement.} $I(Z;U\mid S)$ increases more gradually and typically continues to rise after $I(Z;S)$ has effectively plateaued. The eventual level of $I(Z;U\mid S)$ is task-dependent, with Reacher exhibiting markedly larger within-mode information than DI or Cartpole, consistent with arm-length scaling being more consequential for Reacher's kinematics/dynamics than mass (DI) or pole length (Cartpole) for the other two tasks.
    \item \textbf{Regime changes aligned with mode saturation.} Visible kinks (regime changes) occur around the time $I(Z;S)$ enters its near-saturated regime. The accompanying behavior of $\mathrm{Variability}(M)$ and returns is environment-specific.
\end{enumerate}

\begin{figure}[ht]
    \centering    
    \includegraphics[width=0.6\textwidth]{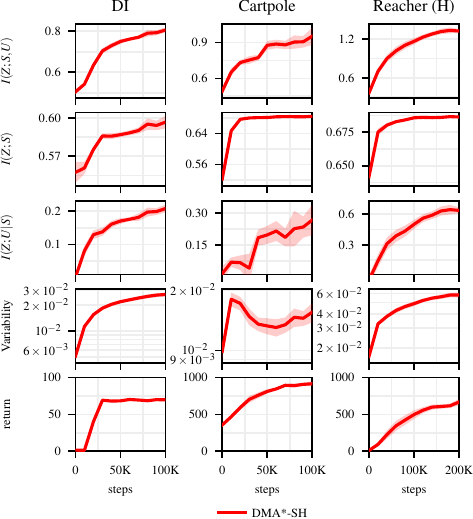}
    \caption{\textbf{Held-out dynamics of mode sufficiency vs. within-mode refinement.} Evolution of $I(Z;S,U)$, $I(Z;S)$, $I(Z;U\mid S)$, Variability, and returns for DI, Cartpole, and Reacher (Hard). Across tasks, mode information $I(Z;S)$ rises fast and approaches the binary ceiling $H(S)=\ln 2$, while within-mode information $I(Z;U\mid S)$ accumulates more slowly; kinks near mode saturation coincide with regime changes in Variability, consistent with a \textit{mode-first, task-adaptive structural bottleneck}.
    }
    \label{fig:variability2}
\end{figure}

\paragraph{DI.}
DI exhibits the sharpest association between mode acquisition and task success. Returns rise sharply and reach their maximum precisely when $I(Z;S)$ transitions from rapid growth to a near-plateau. This temporal alignment is consistent with mode identification being a key enabler of success in DI: under actuator inversion, incorrect mode identification induces acceleration in the wrong direction, which (under sparse reward) yields near-zero return. After this mode-acquisition transition, returns remain saturated, while $I(Z;U\mid S)$ and Variability continue to evolve on $\mathcal{C}_{\mathrm{eval\text{-}out}}$, indicating that additional within-mode information is not required to improve returns in this environment. While our results do not claim a monotone relationship between returns and mutual information, this pattern is compatible with the tradeoff highlighted by Proposition~\ref{prop:paradox}: increasing within-mode information can, in principle, increase representational variability (and hence loosen the gradient-variance upper bound) without yielding additional control benefit once mode information is reliable.

\paragraph{Cartpole.}
In Cartpole, $I(Z;S)$ reaches a near-saturated regime early, while returns continue improving over a longer horizon. The Variability curve exhibits a regime change aligned with the mode-acquisition transition, followed by non-monotone evolution thereafter. This is consistent with a two-stage process: (i) early acquisition of mode information sufficient to avoid catastrophic inversion errors, followed by (ii) longer-horizon refinement associated with within-mode adaptation to $U$ (pole length) under feedback control. Since $I(Z;S)$ does not decrease (and in fact increases slightly) across this transition, a post-transition decrease in Variability cannot be attributed to a loss of mode information; rather, it indicates that non-mode contributions to Variability must be decreasing over that interval (i.e., a reduction in $\mathrm{Term~1}$ and/or $\mathrm{Term~2}$ in Theorem~\ref{thm:var_decomp_SU}). Moreover, the observation that $I(Z;U\mid S)$ can increase while Variability decreases is compatible with Proposition~\ref{prop:boundary} (Part~2): within-mode information about $U$ can be expressed through aspects of the conditional distribution beyond the conditional mean map $\bar z(s,u)$, so $\mathrm{Term~2}$ need not track $I(Z;U\mid S)$. Unlike DI, returns in Cartpole rise gradually rather than jumping at the mode transition, reflecting that balancing provides continuous reinforcement even when mode discrimination remains imperfect, and feedback control can partially offset suboptimal gains.

\paragraph{Reacher (Hard).}
In Reacher (Hard), $I(Z;S)$ again enters a near-saturated regime early, while returns continue to improve substantially afterwards, in parallel with a sustained rise of $I(Z;U\mid S)$. This pattern is consistent with a regime where mode identification resolves the actuator-inversion incompatibility, yet accurate within-mode control remains sensitive to arm-length scaling ($U$), which affects the reachable workspace, the joint-to-end-effector geometry, and the torques required to move limbs of varying length and inertia. Correspondingly, within-mode information appears more behaviorally consequential than in DI. The Variability curve exhibits a regime change aligned with the kink in $I(Z;S)$ and continues to increase before stabilizing late in optimization, consistent with extended refinement on held-out contexts.

Taken together, the eval-out trajectories are qualitatively consistent with a \textit{mode-first} optimization pattern and an \emph{approximate}, \textit{task-adaptive} structural bottleneck: $I(Z;S)$ approaches $H(S)$ early (mode sufficiency), while $I(Z;U\mid S)$ continues to grow later (within-mode refinement) in a task-dependent manner. When $U$ is strongly control-relevant, DMA*-SH does \emph{not} compress it aggressively; instead it learns it. The alignment between mode-acquisition transitions and regime changes in Variability is consistent with the variance decomposition of Theorem~\ref{thm:var_decomp_SU} and the representation-dependent component of the policy-gradient variance bound in Theorem~\ref{thm:pg_var_bound}. Finally, the differing coupling between $I(Z;U\mid S)$ and returns across environments is compatible with the tradeoff articulated by Proposition~\ref{prop:paradox}. We emphasize that this SIB interpretation is a post-hoc explanatory framework supported by these empirical patterns.

\subsection{Scale Control and Directional Geometry in the Context Space}
\label{sec:directional_bias}

We record two empirically relevant properties of DMA* and DMA*-SH that motivate cosine-based geometric summaries used later (Definition~\ref{def:RO}) and help interpret the DI visualizations (Appendix~\ref{sec:more_variability}) through the decomposition in Theorem~\ref{thm:var_decomp_SU} (Appendix~\ref{app:varcompress}).

\paragraph{The DMA* normalization pipeline controls representation scale.}
DMA* and DMA*-SH apply (i) per-sample input normalization via AvgL1Norm and (ii) output normalization via SimNorm (Section~\ref{sec:DMA-schemes}). AvgL1Norm rescales each input vector $x\in\mathbb{R}^N$ to $\mathrm{AvgL1Norm}(x)=\frac{N x}{\sum_i |x_i|}$, which removes arbitrary per-sample magnitude drift in the encoder input. SimNorm maps $z_t$ to a product of simplices via group-wise softmax, yielding a bounded output space with controlled scale. Together, these operators substantially reduce uncontrolled radial scaling of embeddings compared to an unnormalized encoder, so that angular geometry (cosine similarity) becomes a more interpretable notion of similarity for per-context mean embeddings.

\paragraph{Shared hypernetwork conditioning is associated with stronger directional concentration in DI.}
DMA*-SH conditions a single hypernetwork $h_\eta$ on the normalized embedding $z_t$ to generate adapter weights used across dynamics, policy, and value networks. In DI, the cosine heatmaps and RO scores (Figures~\ref{fig:tsnero} and \ref{fig:variability}) show that moving from DMA* to DMA*-SH is associated with higher within-mode cosine similarity across mass values, while the actuator-inversion modes remain distinguishable in both cosine structure and t-SNE. This empirical pattern is consistent with Theorem~\ref{thm:var_decomp_SU}: reduced continuous within-mode spread (compression) can co-occur with nontrivial separation across the binary actuator mode.

\subsubsection{Representation-Overlap ($\mathrm{RO}$).} \label{sec:ROTh}

We introduce Representation-Overlap ($\mathrm{RO}$) as a directional summary of how per-context mean embeddings align in latent space. This is particularly natural in DMA*-SH as SimNorm controls the scale of $z_t$ and the hypernetwork $h_\eta$ is conditioned on the normalized embedding (Appendix~\ref{sec:directional_bias}), so cosine-based comparisons are less confounded by uncontrolled radial scaling. Accordingly, cosine similarity provides an invariant notion of similarity under positive rescalings:
$$
\operatorname{cos}(\alpha u,\beta v)=\operatorname{cos}(u,v)
\qquad \forall \alpha,\beta>0.
$$
When norms are controlled by normalization, higher cosine similarity between mean embeddings is a useful proxy for more directionally concentrated representations, and therefore for more similar hypernetwork conditioning inputs.

\begin{definition}[Representation-Overlap ($\mathrm{RO}$)]\label{def:RO}
Let $\{c^{(1)},\dots,c^{(n)}\}$ be $n$ distinct contexts.
For each $c^{(i)}$, given $B$ embeddings $\{z_{c^{(i)}}^{(b)}\}_{b=1}^B$, define the mean embedding
$\mu_{c^{(i)}}=\frac{1}{B}\sum_{b=1}^B z_{c^{(i)}}^{(b)}$.
The pairwise cosine similarity is
\begin{align}
\label{eq:pairwise_cosim}
\operatorname{cos}(\mu_{c^{(i)}},\mu_{c^{(j)}}) = \frac{\mu_{c^{(i)}}^\top \mu_{c^{(j)}}}{\|\mu_{c^{(i)}}\|\,\|\mu_{c^{(j)}}\|}.
\end{align}
The global $\mathrm{RO}$ score averages over all $n^2$ pairs (including $i=j$):
\begin{align}
\label{eq:RO_cos}
\mathrm{RO} = \frac{1}{n^{2}} \sum_{i=1}^{n} \sum_{j=1}^{n} \operatorname{cos}(\mu_{c^{(i)}},\mu_{c^{(j)}}).
\end{align}
\end{definition}

$\mathrm{RO}$ summarizes global directional concentration of the per-context mean embeddings. In the DI setting, Theorem~\ref{thm:var_decomp_SU} (Appendix~\ref{app:varcompress}) distinguishes (i) within-mode spread over the continuous parameter and (ii) separation across actuator-inversion modes. Empirically, higher $\mathrm{RO}$ in DMA*-SH co-occurs with stronger within-mode alignment across mass (a directional signature consistent with reduced continuous within-mode spread), while still preserving a clear distinction between actuator modes (Appendix~\ref{sec:more_variability}).

\subsubsection{t-SNE Visualization and $\mathrm{RO}$ Cosine Similarity Analysis}
\label{sec:more_variability}

We visualize the geometry of inferred embeddings $z_t$ in DI using t-SNE \citep{van2008visualizing} and cosine similarity of per-context means (Definition~\ref{def:RO}). These views complement the scalar diagnostics in Figure~\ref{fig:variability} by illustrating how the mean structure across contexts changes, which is central to the decomposition in Theorem~\ref{thm:var_decomp_SU} (Appendix~\ref{app:varcompress}).

Figure~\ref{fig:tsnero} shows t-SNE visualizations for DI comparing DMA, DMA*, and DMA*-SH. Each point corresponds to an embedding $z_t = g_\phi(\tau)$, and colors indicate the underlying context pair (actuator-inversion mode and mass). Across methods, embeddings organize into two actuator-mode groups, while DMA*-SH exhibits visibly weaker stratification by mass within each mode. This qualitative pattern is consistent with Theorem~\ref{thm:var_decomp_SU}: lowering the continuous within-mode contribution (compression of $U$ within fixed $S$) can reduce overall Variability, while retaining a nontrivial between-mode separation term for actuator inversion.

We further examine cosine similarities between per-context mean embeddings using Equation~\ref{eq:pairwise_cosim} (Figure~\ref{fig:tsnero}, bottom).

\paragraph{DMA $\rightarrow$ DMA*: Effect of normalization.}
In DI, DMA exhibits strongly negative cosines between means from opposite actuator modes (approximately $-0.4$ to $-0.97$), indicating near sign-opposition in mean directions across modes.
With input/output normalization (DMA*), these cross-mode cosines move to near-orthogonality (about $0.01$--$0.03$), yielding a less extreme directional relationship between the two actuator-mode groups.

\paragraph{DMA* $\rightarrow$ DMA*-SH: Effect of shared hypernetwork conditioning.}
Moving from DMA* to DMA*-SH, within-mode cosines across different mass values increase (roughly from $0.53$--$1.0$ to values near $1.0$ in DI), consistent with stronger within-mode directional concentration and hence reduced continuous within-mode spread (the compression term in Theorem~\ref{thm:var_decomp_SU}). At the same time, cross-mode cosines remain distinct from within-mode values (e.g., around $0.14$--$0.23$ in DI), consistent with maintaining a separable actuator-inversion structure rather than collapsing the discrete distinction.

Overall, the t-SNE and cosine patterns provide a geometric complement to Figure~\ref{fig:variability} and are consistent with the compression/separation decomposition in Theorem~\ref{thm:var_decomp_SU}.

\subsection{Implicit Regularization via Shared Hypernetwork Gradients} \label{app:gradcon}

\paragraph{Shared hypernetworks (DMA*-SH).}
We examine the effect of sharing a single hypernetwork across the dynamics model, policy, and Q-function in DMA*-SH. The training loop in Algorithm~\ref{alg:dmash} interleaves (i) RL updates of the actor and critic base parameters and (ii) dynamics-driven updates of the representation and hypernetwork parameters:
\begin{itemize}[leftmargin=*]
\item RL updates:
$$
\xi \gets \xi - \alpha_1 \nabla_\xi \sum_c L^c_\xi,
\qquad
\zeta \gets \zeta - \alpha_2 \nabla_\zeta \sum_c L^c_\zeta,
$$
where $L^c_\xi$ and $L^c_\zeta$ are the actor and critic losses computed using $\pi_{\xi,\omega}$ and $Q_{\zeta,\omega}$ with $\omega=h_\eta(z_t)$ (treated as a constant in $L^c_\xi$ and $L^c_\zeta$ via stop-gradient).

\item Dynamics updates:
$$
\phi \gets \phi - \alpha_3 \nabla_\phi \sum_c L^c_{\phi,\theta,\eta},
\qquad
\theta \gets \theta - \alpha_3 \nabla_\theta \sum_c L^c_{\phi,\theta,\eta},
\qquad
\eta \gets \eta - \alpha_3 \nabla_\eta \sum_c L^c_{\phi,\theta,\eta},
$$
where $L^c_{\phi,\theta,\eta}$ is the forward-dynamics (FD) reconstruction objective~\eqref{eq:hnopt} and depends on $f_{\theta,\omega}(s_t,a_t)$ with $\omega=h_\eta(z_t)$.
\end{itemize}

Since $\omega$ is detached in the RL losses, the RL objectives update only the base parameters $\xi$ and $\zeta$; gradients from $L_\xi$ and $L_\zeta$ do not backpropagate through $\omega$ to $\eta$ (or to $z_t$) (Figure~\ref{fig:architectureB}). Nonetheless, the objectives remain coupled through the \emph{shared adapter pathway}, namely the forward mapping
\begin{align}
\label{eq:shared-adapter-pathway}
\tau_t^c \xrightarrow{\,g_\phi\,} z_t \xrightarrow{\,h_\eta\,} \omega_t \xrightarrow{\mathrm{shared\ adapters}} \{f_{\theta,\omega_t},\,\pi_{\xi,\omega_t},\,Q_{\zeta,\omega_t}\}.
\end{align}
The same $\omega_t=h_\eta(z_t)$ is reused by the dynamics model, actor, and critic. During RL updates, $\xi$ and $\zeta$ are optimized under the adapter configuration produced by the current hypernetwork, but gradients through $\omega_t$ are stopped. During dynamics updates, $\phi$, $\theta$, and $\eta$ are refined by the reconstruction objective using data collected by the agent. This sharing injects a dynamics-trained structural constraint into RL optimization: the actor and critic can use context only through adapters shaped by the dynamics objective. Appendix~\ref{sec:TxrecurDMASH} contrasts this factorization with history-based recurrent and Transformer agents.

\paragraph{Separate hypernetworks (DMA*-H).}
We compare DMA*-SH to a separate-hypernetwork design, denoted DMA*-H, where the dynamics model, policy, and Q-function have separate hypernetwork parameters $\eta^f,\eta^\pi,\eta^Q$ and produce separate adapter weights $\omega_t^f,\omega_t^\pi,\omega_t^Q$:
\begin{align}
\label{eq:adapter-pathway}
z_t \xrightarrow{\,h_{\eta^m}\,} \omega_t^m
\xrightarrow{\mathrm{adapter}} M^m_{\omega_t^m},
\qquad m\in\{f,\pi,Q\}.
\end{align}
Here $M^m$ denotes the corresponding module: $f_\theta$, $\pi_\xi$, or $Q_\zeta$. The FD loss updates only $\eta^f$, while the actor and critic losses update $\eta^\pi$ and $\eta^Q$. This decouples dynamics-driven and reward-driven adaptation, which can introduce objective mismatch when reward-driven adapters move in directions not supported by the dynamics-trained pathway.

\paragraph{Shadow gradients (diagnostics in $\eta$-space and $z$-space).}
Since $\omega$ is detached in the RL losses in DMA*-SH, the true training-time gradients through the shared adapter pathway~\eqref{eq:shared-adapter-pathway} satisfy $\nabla_\eta L_\pi = 0$ and $\nabla_z L_\pi = 0$ by construction. For analysis only, we define a \emph{shadow} computation that temporarily removes the stop-gradient through $\omega$ and differentiates RL losses along this pathway. These shadow quantities are \emph{not} used for training and do not affect the optimization trajectory; they are diagnostic probes of how the RL objective would prefer to change the shared adapter pathway if the detachment were removed.

In $\eta$-space, the shadow sensitivity of $L_{\mathrm{RL}}$ to the shared hypernetwork parameters follows
$\frac{\partial L_{\mathrm{RL}}}{\partial \eta}
= \frac{\partial L_{\mathrm{RL}}}{\partial \omega}\cdot \frac{\partial \omega}{\partial \eta}$, 
where $\frac{\partial \omega}{\partial \eta} = \frac{\partial h_{\eta}(z_t)}{\partial \eta}$. The factor $\frac{\partial L_{\mathrm{RL}}}{\partial \omega}$ chains through both $\pi_{\xi,\omega}$ and $Q_{\zeta,\omega}$.
For instance, for the actor, the relevant component is
$-\frac{\partial}{\partial \omega}\mathbb{E}_{a_t\sim \pi_{\xi,\omega}}\left[Q_{\zeta,\omega}-\alpha \log \pi_{\xi,\omega}\right]$.
Although $\eta$ is shielded from RL updates to avoid direct interference with the FD-trained adapter mapping, the shadow gradients quantify how the RL objective would (hypothetically) prefer to adjust the shared mapping $z\mapsto \omega$ if such updates were enabled, while training instead proceeds by adapting $(\xi,\zeta)$ under the dynamics-trained adapters.

In $z$-space, the analogous shadow sensitivity is
$\frac{\partial L_{\mathrm{RL}}}{\partial z_t} = \frac{\partial L_{\mathrm{RL}}}{\partial \omega} \cdot \frac{\partial \omega}{\partial z_t}$, where $\frac{\partial \omega}{\partial z_t}=\frac{\partial h_\eta(z_t)}{\partial z_t}$. The factor $\frac{\partial L_{\mathrm{RL}}}{\partial \omega}$ chains through both $\pi_{\xi,\omega}$ and $Q_{\zeta,\omega}$, so $\|\nabla_{z_t} L_\pi\|$ and $\|\nabla_{z_t} L_Q\|$ quantify how strongly the policy and critic objectives would (hypothetically) prefer to change the inferred context signal along the adapter pathway. These $z$-space norms provide a direct measure of \textit{context sensitivity} in the shared embedding space $\mathbb{R}^{d_z}$.

\paragraph{Gradient analysis (Figure~\ref{fig:gradcon}).}

We report gradient-based diagnostics across DI, DI-Friction, ODE, and Cartpole to characterize how DMA*-SH and DMA*-H differ in (i) effective context sensitivity along the adapter pathway and (ii) the variance of selected gradient-norm signals.

In DMA*-SH, all quantities involving $\nabla_\eta L_\pi$, $\nabla_z L_\pi$, or the cosines with these terms are computed using the shadow construction described above.

\begin{enumerate}[leftmargin=*, topsep=2pt, itemsep=2pt]

\item Mean policy hypernetwork sensitivity ($\eta$-space): $\mathbb{E}\|\nabla_{\eta} L_\pi\|$ (DMA*-SH) and $\mathbb{E}\|\nabla_{\eta^\pi} L_\pi\|$ (DMA*-H).
Across environments, the DMA*-SH curve remains substantially larger than DMA*-H, indicating a persistent hypothetical tendency of the policy objective to reshape the shared hypernetwork mapping if such updates were enabled. This is consistent with the actor and critic adapting their base parameters $(\xi,\zeta)$ under a dynamics-trained adapter mapping rather than directly rewriting it.

\item Mean policy context sensitivity ($z$-space): $\mathbb{E}\|\nabla_{z} L_\pi\|$.
DMA*-SH exhibits consistently larger values than DMA*-H, indicating stronger dependence of the policy objective on the inferred context through the adapter pathway. In DMA*-H, the smaller values are consistent with weaker effective context utilization along this route.

\item Variance of policy gradients w.r.t.\ context ($z$-space): $\mathrm{Var}\|\nabla_{z} L_\pi\|$.
Unlike the mean sensitivity, which captures persistent dependence on the adapter pathway, the relative ordering varies across environments, suggesting that the variance is influenced by environment-specific nonstationarity and mode-switching.

\item Variance of policy base-parameter gradient norms ($\xi$-space): $\mathrm{Var}\|\nabla_\xi L_\pi\|$.
DMA*-SH typically shows slightly lower variance than DMA*-H, consistent with a modestly more stable optimization signal for the policy base parameters under shared adapters.

\item Variance of context-encoder gradients under dynamics ($\phi$-space): $\mathrm{Var}\|\nabla_\phi L_d\|$.
This measures variance of the dynamics-driven learning signal entering the context encoder. Values are small in magnitude, but DMA*-SH often shows slightly lower variance than DMA*-H.

\item Variance of dynamics gradients w.r.t.\ embedding ($z$-space): $\mathrm{Var}\|\nabla_z L_d\|$.
This measures variance of the dynamics objective's sensitivity to the inferred embedding in $\mathbb{R}^{d_z}$. Magnitudes are small, so this metric serves mainly as a weak supporting diagnostic.

\item Dynamics--policy alignment ($\eta$-space): $\mathrm{cos}(\nabla_{\eta} L_d, \nabla_{\eta} L_\pi)$ (DMA*-SH) and a heuristic analogue in DMA*-H comparing $\eta^f$ and $\eta^\pi$.
In DMA*-SH, this cosine tracks whether the (hypothetical) reward-driven direction in the shared $\eta$ coordinates tends to align with or oppose the dynamics-driven direction. In DMA*-H, the corresponding cosine is a heuristic as it compares different parameter spaces.

\item Policy--critic alignment ($\eta$-space): $\mathrm{cos}(\nabla_{\eta} L_\pi, \nabla_{\eta} L_Q)$ (DMA*-SH) and a heuristic analogue in DMA*-H comparing $\eta^\pi$ and $\eta^Q$.
In DMA*-SH, this cosine summarizes whether actor and critic objectives would push the shared hypernetwork parameters in similar directions under the hypothetical update. In DMA*-H, the corresponding cosine is a heuristic as it compares different parameter spaces.

\item Returns.
DMA*-SH attains faster learning and higher returns than DMA*-H across the considered environments.
This performance gap co-occurs with the separation in $\mathbb{E}\|\nabla_{\eta} L_\pi\|$ and $\mathbb{E}\|\nabla_{z} L_\pi\|$, consistent with stronger and more persistent context dependence through the shared adapter pathway.
\end{enumerate}

Overall, the most consistent separation is in the mean context-sensitivity diagnostics, both $\mathbb{E}\|\nabla_{\eta} L_\pi\|$ and $\mathbb{E}\|\nabla_{z} L_\pi\|$. This suggests that DMA*-SH maintains a sustained context-dependent adapter pathway shaped by the dynamics-trained hypernetwork, whereas DMA*-H exhibits substantially weaker sensitivity along the corresponding policy pathway.

\subsection{Comparison to History-Based Recurrent and Transformer Agents}
\label{sec:TxrecurDMASH}

A history-based agent such as Amago-2~\citep{grigsby2024bamago} uses a trajectory encoder or sequence model to produce a latent state shared by actor and critic heads. Abstractly, this adaptation pathway can be written as
$$\tau_t \xrightarrow{\mathrm{SeqModel}_\psi}
h_t \xrightarrow{\mathrm{policy/value\ heads}} (a_t,Q_t).
$$
Thus task inference and control are coupled inside a single hidden state $h_t$ optimized by RL objectives. Such models can in principle implement multiplicative interactions through attention or gating, and our expressivity theorem is not a separation against Transformers. The distinction is structural rather than purely expressive.

In DMA*-SH, the corresponding pathway is the dynamics-grounded factorization in~\eqref{eq:shared-adapter-pathway}. History first determines a context embedding $z_t$, then a hypernetwork maps this embedding to a shared adapter configuration $\omega_t$, and the same $\omega_t$ is reused by the dynamics model, actor, and critic. The map $z_t\mapsto\omega_t$ is trained only by the forward-dynamics objective, while actor and critic gradients are stopped through $\omega_t$. Thus all context dependence seen by control is forced through one dynamics-trained intermediate object, rather than an arbitrary RL-trained hidden state.

This factorization separates two roles that are often conflated in recurrent or Transformer baselines: \textit{mode identification} and \textit{mode-conditioned computation}. A sequence model may infer latent task information from history, but it does not require the inferred representation to parameterize the transition model and the control networks through the same operator. DMA*-SH imposes this constraint explicitly: the multiplicative operator used by control is grounded in transition prediction and shared across model, actor, and critic.

This distinction is particularly important in non-overlapping settings. A wrong mode estimate can make the policy catastrophically wrong, so early reward gradients may be sparse, high-variance, or misleading. Dynamics prediction provides a dense supervised signal from the first transition. Detaching $\omega_t$ forces the actor and critic to learn under a dynamics-trained adapter configuration rather than rewriting the context-to-operator map themselves. Hence detachment is part of the structural prior, not merely an implementation detail.

\paragraph{No-detach ablation.}
Our no-detach ablations in Figure~\ref{fig:iqm-hyper} support this interpretation. Allowing critic gradients through the shared hypernetwork degrades performance, while allowing actor gradients through the shared hypernetwork causes near-zero returns on the tested non-overlapping environments. In contrast, analogous actor-gradient flow in DMA* without a shared hypernetwork has only modest effect. This indicates that the critical failure mode is not simply that RL gradients touch a context encoder, but that RL gradients rewrite the shared context-to-operator map used by all modules.

\begin{figure}[ht]
    \centering
    \includegraphics[width=0.73\textwidth]{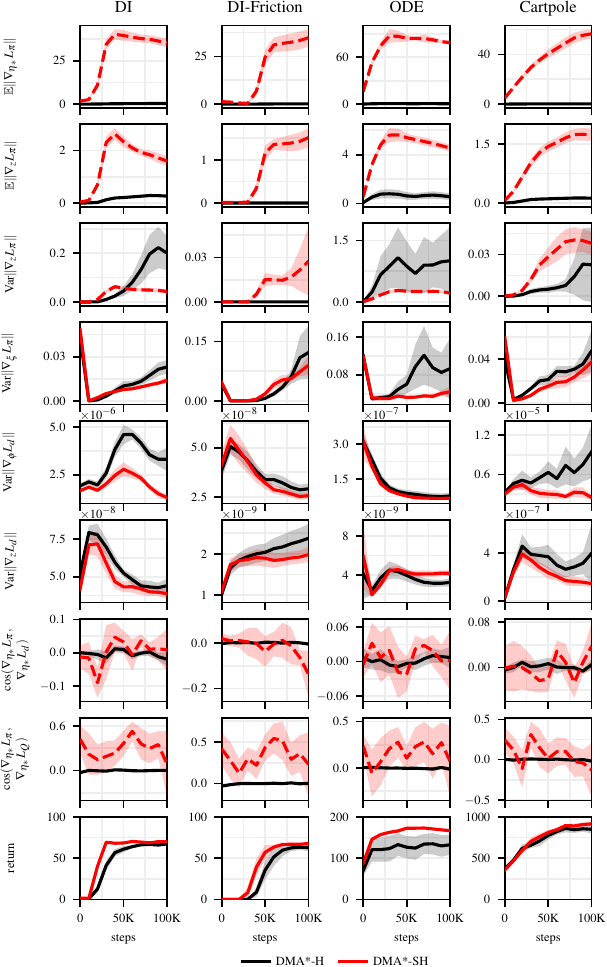}
    \caption{
    Gradient analysis comparing shared hypernetworks (DMA*-SH, red) vs.\ separate hypernetworks (DMA*-H, black) across DI, DI-Friction, ODE, and Cartpole environments. Dashed lines indicate quantities computed via a shadow graph obtained by temporarily removing the stop-gradient through the adapter pathway (e.g., $\nabla_{\eta} L_{\pi}$ or $\nabla_{z} L_{\pi}$ in DMA*-SH, where neither $\eta$ nor $z$ receives policy gradients during training due to detachment of $\omega$). This enables hypothetical gradient evaluation without altering the training loop. Here, $\eta_*$ denotes the relevant hypernetwork parameters: in shared, the single $\eta$ (optimized solely via the dynamics loss $L_d$); in separate, the module-specific hypernetworks (e.g., $\eta^\pi$ for policy-related gradients). $L_d$, $L_{\pi}$ and $L_Q$ correspond respectively to the FD loss, $L_\xi$ (actor) and $L_\zeta$ (Q-function).    }
    \label{fig:gradcon}
\end{figure}

\clearpage

\section{Extended Related Work}
\label{sec:related_extended}

\textbf{Zero-shot generalization in contextual RL.} Contextual RL has been studied from multiple perspectives, including CMDPs, domain randomization, and meta-RL \citep{hallak2015contextual,modi2018markov,beck2023survey}. The survey by \citet{kirk2023survey} highlights its importance for zero-shot generalization, noting that separating training and evaluation context sets enables systematic evaluation. Broadly, two directions are distinguished: 1) explicit context is observable as privileged information \citep{chen2018hardware,seyed2019smile,ball2021augmented,eghbal2021context,sodhani2021multi,mu2022domino,benjamins2023contextualize,prasanna2024dreaming}, and 2) context must be inferred implicitly from past experience \citep{chen2018hardware,xu2019densephysnet,lee2020context,seo2020trajectory,xian2021hyperdynamics,sodhani2022block,melo2022transformers,evans2022context,ndir2024inferring,roder2025dynamics}.
Our work follows the second approach, focusing on self-supervised context inference via dynamics-model alignment. Recurrent agents may also learn internal context representations \citep{grigsby2024amago,grigsby2024bamago,luo2024efficient,hafner2019learning,hafner2025mastering}, though these are typically not aligned with the underlying dynamics.
Closely related, \citet{beukman2023dynamics} utilize hypernetworks \citep{ha2017hypernetworks} to incorporate context into RL models. Our approach differs fundamentally: We do not assume access to explicit context, and rather than employing separate hypernetworks for the policy/Q-function, we train a single hypernetwork jointly with the dynamics model, which is then shared across policy and action-value networks.

\textbf{Meta-RL.} Meta-RL aims to enable agents to rapidly adapt to unseen tasks with minimal data \citep{beck2023survey}, often by learning policies that infer task structure from prior interactions, sometimes using hypernetworks \citep{beck2022hyper,beck2023recurrent}. Most meta-RL approaches require fine-tuning on new tasks {across multiple episode rollouts} \citep{duan2016rl,finn2017model,nagabandi2018learning,rakelly2019efficient,zintgraf2019fast}, which is incompatible with our zero-shot generalization setting. VariBAD \citep{zintgraf2020varibad} and TrMRL \citep{melo2022transformers} are not subject to this limitation, as they have been shown to adapt to the task within the first rollout. Recent advances in context-based offline meta-RL are also promising, as these algorithms aim to infer latent task information, both reward- and dynamics-based, from static datasets \citep{li2021focal,dorfman2021offline,yuan2022robust,li2024Unicorn,li2024efficient}. However, since these methods typically rely on assumptions inherent to the offline setting, it remains an open question whether they can be robustly applied to online RL \citep{li2024Unicorn}.

\textbf{Context in cognition.} Beyond RL, cognitive modeling suggests that humans segment the environment into context-like events \citep{Zacks:2001,Zacks:2007,Butz:2016}. For instance, the recurrent REPRISE model learns latent context representations from scratch, distinguishing dynamic regimes \citep{Butz:2019}. More recent work differentiates event segmentation from context inference, showing that contextual priors support learning of sensorimotor repertoires and memory structures \citep{Heald:2021,Heald:2023}. 
Bayesian active inference models indicate that context can reduce computational effort while accurately modeling human behavior \citep{Markovic:2021,Schwoebel:2021,Butz:2022,CuevasRivera:2023,Parr:2023,Mittenbuehler:2024}. In cognitive modeling-inspired deep learning, contextualized hypernetworks have been introduced in various forms, showing superior generalization and emergent compositionality \citep{sugita2011simultaneously}, the emergence of affordance maps \citep{Scholz:2022}, as well as the possibility to focus object-oriented encoding pipelines \citep{Traub:2024}. At the intersection of neuroscience, developmental psychology, cognitive modeling, and machine learning, context inference and context-conditioned learning appear to be critical to enabling robust behavioral learning in complex environments \citep{Butz:2024}.

\clearpage
\section{Algorithms}
\label{sec:algorithms}

For DMA/DMA* we isolate encoder training to the dynamics objective by detaching $z_t$ in RL losses (Algorithm~\ref{alg:dma}); analogously, DMA*-SH detaches $\omega$ in RL losses (Algorithm~\ref{alg:dmash}). See Figure~\ref{fig:architecture}.

\begin{minipage}{0.47\textwidth}
\begin{algorithm}[H]
\centering
\begin{algorithmic}[1]
\REQUIRE Context set $\mathcal{C}_{\text{train}} = \{c_i\}_{i=1 \ldots n_c}$ sampled from context range for training, learning rates $\alpha_1, \alpha_2, \alpha_3$
\STATE Init. replay buffers $\mathcal{B}^c$ for each context
\STATE Init. context windows (deque) $\tau^c_{t=0} = \{(s_{t-k}, a_{t-k}, \delta s_{t+1-k})\}_{k:1\ldots K} \sim \pi_{random}$ for each context
\FOR{step in training steps}
\STATE \textit{// Collect data in environment} 
\FOR{$c$ in $\mathcal{C}_{\text{train}}$}
\STATE Encode past transitions $z_t = g_\phi(\tau^c_t)$
\STATE \phantom{Compute weights $\omega = h_\eta(z_t)$}
\STATE Gather data from environment interaction with $a_t \sim \pi_{\xi}(\cdot | s_t, z_t)$
\STATE Add data to $\mathcal{B}^c$ and update $\tau^c_t$
\ENDFOR
\STATE \textit{// Training} 
\FOR{$c$ in $\mathcal{C}_{\text{train}}$}
\STATE Sample RL batch $b^c \sim \mathcal{B}^c$ with corresponding context windows $\tau^c_t$
\STATE Encode past transitions $z_t = g_\phi(\tau^c_t)$
\STATE \phantom{}
\STATE Predict $\delta \hat{s}_{t+1} = f_\theta(s_t, a_t, z_t)$
\STATE $L^c_\xi = L_\xi(\pi_\xi, b^c, z_t.\mathrm{detach}())$
\STATE $L^c_\zeta = L_\zeta(Q_\zeta, b^c, z_t.\mathrm{detach}())$
\STATE $L^c_{\phi,\theta} = \left\lVert \delta \hat{s}_{t+1} - \delta s_{t+1} \right\rVert_2^2$
\ENDFOR
\STATE $\xi \gets \xi - \alpha_1 \nabla_\xi \sum_c L^c_\xi$
\STATE $\zeta \gets \zeta - \alpha_2 \nabla_\zeta \sum_c L^c_\zeta$
\STATE $\phi \gets \phi - \alpha_3 \nabla_\phi \sum_c L^c_{\phi,\theta}$
\STATE $\theta \gets \theta - \alpha_3 \nabla_\theta \sum_c L^c_{\phi,\theta}$
\STATE \phantom{$\eta \gets \eta - \alpha_3 \nabla_\eta \sum_c L^c_{\phi,\theta,\eta}$}
\ENDFOR
\end{algorithmic}
\caption{{Training loop \textbf{DMA/DMA*}}}
\label{alg:dma}
\end{algorithm}
\end{minipage}
\hfill
\begin{minipage}{0.47\textwidth}
\begin{algorithm}[H]
\centering
\begin{algorithmic}[1]
\REQUIRE Context set $\mathcal{C}_{\text{train}} = \{c_i\}_{i=1 \ldots n_c}$ sampled from context range for training, learning rates $\alpha_1, \alpha_2, \alpha_3$
\STATE Init. replay buffers $\mathcal{B}^c$ for each context
\STATE Init. context windows (deque) $\tau^c_{t=0} = \{(s_{t-k}, a_{t-k}, \delta s_{t+1-k})\}_{k:1\ldots K} \sim \pi_{random}$ for each context
\FOR{step in training steps}
\STATE \textit{// Collect data in environment} 
\FOR{$c$ in $\mathcal{C}_{\text{train}}$}
\STATE Encode past transitions $z_t = g_\phi(\tau^c_t)$
\STATE {\color{red} Compute weights $\omega = h_\eta(z_t)$}
\STATE Gather data from environment interaction with $a_t \sim \textcolor{red}{\pi_{\xi,\omega}(\cdot | s_t)}$
\STATE Add data to $\mathcal{B}^c$ and update $\tau^c_t$
\ENDFOR
\STATE \textit{// Training} 
\FOR{$c$ in $\mathcal{C}_{\text{train}}$}
\STATE Sample RL batch $b^c \sim \mathcal{B}^c$ with corresponding context windows $\tau^c_t$
\STATE Encode past transitions $z_t = g_\phi(\tau^c_t)$
\STATE {\color{red} Compute weights: $\omega = h_\eta(z_t)$}
\STATE Predict $\delta \hat{s}_{t+1} = \textcolor{red}{f_{\theta,\omega}(s_t,a_t)}$
\STATE $L^c_\xi = \textcolor{red}{L_\xi(\pi_{\xi,\omega.\mathrm{detach}()}, b^c)}$
\STATE $L^c_\zeta = \textcolor{red}{L_\zeta(Q_{\zeta,\omega.\mathrm{detach}()}, b^c)}$
\STATE $\textcolor{red}{L^c_{\phi,\theta,\eta}} = \left\lVert \delta \hat{s}_{t+1} - \delta s_{t+1} \right\rVert_2^2$
\ENDFOR
\STATE $\xi \gets \xi - \alpha_1 \nabla_\xi \sum_c L^c_\xi$
\STATE $\zeta \gets \zeta - \alpha_2 \nabla_\zeta \sum_c L^c_\zeta$
\STATE $\phi \gets \phi - \alpha_3 \nabla_\phi \sum_c \textcolor{red}{L^c_{\phi,\theta,\eta}}$
\STATE $\theta \gets \theta - \alpha_3 \nabla_\theta \sum_c \textcolor{red}{L^c_{\phi,\theta,\eta}}$
\STATE {\color{red} $\eta \gets \eta - \alpha_3 \nabla_\eta \sum_c L^c_{\phi,\theta,\eta}$}
\ENDFOR
\end{algorithmic}
\caption{{Training loop \textbf{DMA*-SH}}}
\label{alg:dmash}
\end{algorithm}
\end{minipage}

\newpage
\section{Hyperparameters and Implementation Details}
\label{sec:hyperparameters}
Table~\ref{tab:hyperparameters} provides an overview of the used hyperparameters of the SAC agent, the context encoder, the dynamic model and the hypernetwork. We did not perform any tuning for SAC and kept hyperparameters standard as provided in CleanRL \citep{huang2022cleanrl}. 

At the core of the context encoder, we use an LSTM layer whose final hidden state serves as the context representation $z_t$. This follows prior work employing MLPs, RNNs, or Transformer encoder layers as context encoders \citep{rakelly2019efficient, evans2022context}. Prior work \citep{rakelly2019efficient} used permutation-invariant encoders so that context inference does not depend on the temporal ordering of transitions within the context window. Motivated by this, we process a uniformly subsampled fraction of transitions from the context window in random order before feeding them to the LSTM. The context window size $K$ depends on the environment: tasks derived from the DM Control Suite require a larger $K$ than others. The context encoder samples a random $20\%$ of the $K$ transitions as input $\tau_t^c$; for example, in the DM Control Suite it observes only $128 \times 0.2 \approx 25$ transitions.
 
For our hypernetworks, we use the framework of \citet{oshg2019hypercl}. The adapter introduces a bottleneck, and importantly, we do not apply an activation function before it. Our design also allows the adapter to be bypassed via a skip connection. The design choices regarding the hypernetworks and adapters match those in DA \citep{beukman2023dynamics}, where the placement of activation functions is likewise crucial. We reimplemented DA and verified that its performance is comparable to the original implementation.

\begin{figure}[ht]
    \centering
    \includegraphics[width=0.98\textwidth]{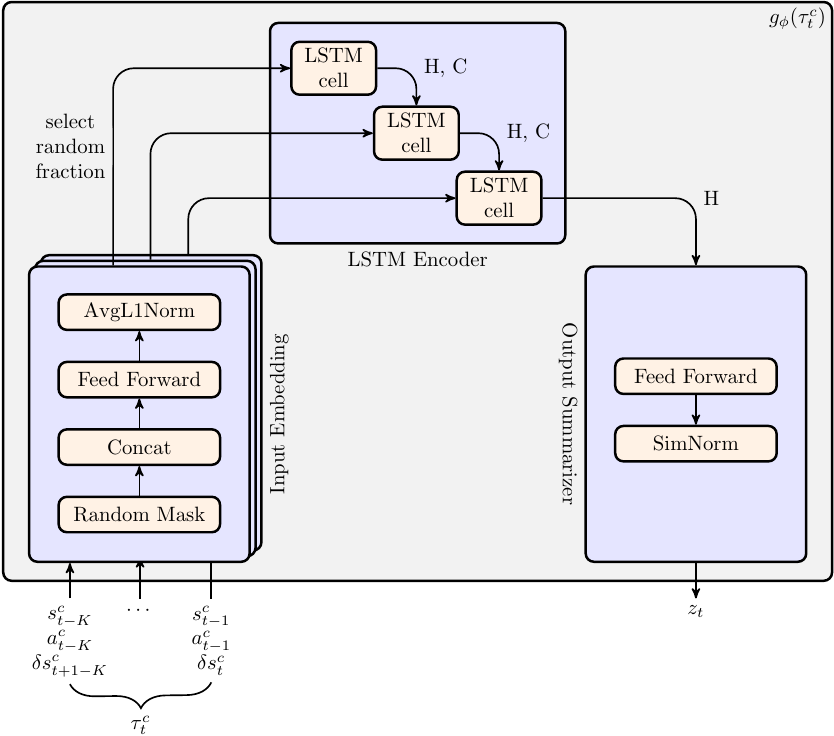}
    \caption{Architecture of the context encoder.}
    \label{fig:encoder}
\end{figure}

\begin{table}[ht]
  \caption{Hyperparameters.}
  \label{tab:hyperparameters}
  \centering
  \begin{tabular}{lll}
    \toprule

    Module & Name & Value \\
    \midrule
    SAC&Buffer capacity (per context) &100 000\\
    &Batch size &256\\
    &Discount $\gamma$ &0.99\\
    &Optimizer &Adam\\
    &Critic LR &0.001\\
    &Actor LR &0.0003\\
    &Temperature LR &0.0003\\
    &Critic soft target update $\tau$ &0.005\\
    &Init temperature & 1.0\\
    &Hidden dims & (256, 256) \\
    &Activation function & ReLU \\
    \midrule
    Context encoder&LR &0.0003\\
    &Model dim  & 32 \\
    &Context dim & 8 \\
    &Context window size $K$ (DI, ODE) & 24 \\
    &Context window size $K$ (DMC- and Gymnasium based) & 128 \\
    &Context window fraction & 0.2 \\
    &Context encoder type & LSTM \\
    &LSTM representation state & hidden H\\
    &Activation function & ReLU \\
    &Input Norm & AvgL1Norm \\
    &Output Norm & SimNorm \\
    &Input Masking (DMA*) & 0.2 \\
    &Input Masking (DMA*-SH) & 0.4 \\ 
    \midrule
    Dynamic model &LR &0.0003\\
    &Hidden dims & (256, 256) \\
    &Activation function & ReLU \\
    \midrule
    Hypernetwork &LR&0.0003\\
    &Hidden dims & (64, 64) \\
    &Activation function & ReLU \\
    \midrule
    Adapter & Bottleneck & 32\\
    &Skip connection & True \\
    &Pre adapter activation function & None \\
    &Post adapter activation function & ReLU \\
    \bottomrule
  \end{tabular}
\end{table}

\clearpage
\section{AIB: Discontinuous Context-to-Dynamics Environment Suite}\label{app:env_details}

This section gives concise descriptions of the environments summarized in Table~\ref{tab:envs_summary}, together with the context variables and our overlap classification (Definition~\ref{def:OverlapNonoverlap}). Unless noted otherwise, each environment uses two context dimensions sampled once per episode and held fixed throughout. The context does not vary within an episode, it varies across episodes.

At a high level, environments are \emph{overlapping} when a single robust context-unaware policy can often perform reasonably across the context range, as is typical for continuous physical parameters inducing gradual changes in optimal behavior. Environments are \emph{non-overlapping} when context instances induce mutually incompatible optimal control laws. AIB creates such structure through discontinuous action-effect changes, including actuator inversion, actuator permutations, and continuous weakly non-overlapping dynamics. Actuator inversion (Definition~\ref{def:actInv}) is the minimal binary case: a context $c\in\{\pm1\}$ multiplies the action before it affects the dynamics, $P^c(s'|s,a)=P(s'|s,c\cdot a)$. This makes identical actions induce opposite state transitions across modes. Actuator permutations similarly swap control channels, while ODE/ODE-$k$ provide continuous contexts with empirically low policy overlap over the benchmark ranges.

\subsection{Environments\label{sec:AIBenviron}}

Several tasks are taken from the DeepMind Control Suite (DMC) \citep{tassa2018deepmind} and Gymnasium MuJoCo \citep{towers2024gymnasium}. The ODE/ODE-$k$ environments follow prior work \citep{beukman2023dynamics}. The suite contains three types of non-overlap: binary sign inversions, actuator permutations, and continuous weakly non-overlapping dynamics.

\paragraph{DI (Custom; non-overlapping).}
We consider a custom two-dimensional double-integrator point-mass task. The state consists of planar position and velocity, and actions apply planar forces in the $x$ and $y$ directions. The agent is initialized away from the goal and receives a sparse reward: $+1$ on reaching the goal, $0$ otherwise. Context variables are (i) the mass and (ii) an actuator inversion factor $c\in\{\pm 1\}$ that flips the sign of the applied force. The inversion produces incompatible control laws related approximately by a sign flip, yielding non-overlapping structure.

\paragraph{DI-Friction (Custom; overlapping).}
This environment matches DI but includes frictional damping. Context variables are (i) mass and (ii) friction coefficient. Both are continuous physical parameters and typically induce gradual shifts in the optimal control law. Robust context-unaware agents can often handle a wide range, so we treat this family as overlapping.

\paragraph{DI-Perm (Custom; non-overlapping).} This environment considers an actuator permutation-variant of the DI. Context variables are (i) the mass and (ii) an actuator permutation factor $q \in \{0,1\}$ that swaps the $x$- and $y$-components of the applied force. For $q=0$, actions are applied normally; for $q=1$, the force components are exchanged, i.e. $(a_x,a_y)$ is applied as $(a_y,a_x)$. The permutation produces incompatible control laws related approximately by a coordinate permutation, yielding non-overlapping structure.

\paragraph{ODE (Custom; weakly non-overlapping).}
We include a scalar regulation task where the state evolves as
$$
x_{t+1}=x_t+\dot x_t\,dt,\qquad \dot x_t = c_1 a_t + c_2 a_t^2,
$$
with context $c=(c_1,c_2)\in\mathbb{R}^2$. The objective is to choose $a_t$ to keep $x_t$ near zero. Although the context variables are continuous (no explicit binary inversion), the action-to-state map depends multiplicatively and nonlinearly on $c$. Regulating $x_t > 0$ toward zero requires $\dot x_t < 0$. To first order, this requires $c_1 a_t < 0$: when $c_1 > 0$, the agent must choose $a_t < 0$, whereas when $c_1 < 0$, the agent must choose $a_t > 0$. Thus, the required action direction flips with $\mathrm{sign}(c_1)$. This acts as a continuous analog of actuator inversion induced by $c_1$ crossing $0$ rather than an explicit binary flag. The quadratic term $c_2 a_t^2$ adds context-dependent magnitude constraints, e.g., when $c_2 > 0$, excessively large $|a|$ makes the quadratic term dominate and can prevent $\dot x_t<0$ even if $c_1 a_t<0$. The primary action direction, however, is still dictated by $\mathrm{sign}(c_1)$. 
Empirically, \citet{beukman2023dynamics} report poor performance of context-unaware baselines on this task over their benchmark ranges, consistent with low policy overlap; we therefore classify ODE as {weakly non-overlapping}.

\paragraph{ODE-$k$ (Custom; weakly non-overlapping).}
This environment extends the ODE task to $k$ context dimensions via a higher-order polynomial
$$
x_{t+1}=x_t+\dot x_t\,dt,\qquad \dot x_t = \sum_{j=1}^{k} c_j a_t^{j},
$$
with context $c=(c_1,\ldots,c_k)$. In ODE-$k$, sign variation in odd-order coefficients ($c_1, c_3, \dots$) can enforce directional action flips analogous to the $c_1$ effect in ODE. Even with fixed signs, varying the coefficients can substantially alter the polynomial's shape, introducing multiple local minima or divergent directions that can cause the optimal control strategy to shift abruptly between disparate regimes. This increased structural complexity reduces policy overlap. Empirically, context-unaware agents perform poorly in these settings \citep{beukman2023dynamics}. We therefore classify ODE-$k$ as {weakly non-overlapping}.

\paragraph{Cartpole (DMC; non-overlapping).}
We use the cartpole-balance-v0 task, where the agent balances a pole by applying horizontal forces to its base. The environment is contextualized by (i) pole length and (ii) an actuator inversion factor $c\in\{\pm1\}$ that flips the sign of the applied force. This creates a discontinuous action-effect change and yields non-overlapping structure.

\paragraph{Cheetah (DMC; non-overlapping).}
We use the cheetah-run-v0 task, where a planar biped is rewarded for moving forward. The environment is contextualized by (i) leg-length scaling, which significantly affects gait dynamics, and (ii) an actuator inversion factor $c \in \{\pm 1\}$ applied to action effects. The inversion produces incompatible gaits/control commands across modes, yielding non-overlapping structure.

\paragraph{Reacher (Easy/Hard) (DMC; non-overlapping).}
We use the reacher-easy/hard-v0 tasks (easy and hard variants), where a two-link planar arm is rewarded for reaching a target sphere (larger in the easy variant than in the hard variant). The environment is contextualized by (i) an arm-length scaling factor, which changes the kinematics and dynamics, and (ii) an actuator inversion factor $c\in \{\pm 1\}$ that flips the effect of actions. The inversion induces incompatible reaching behaviors across modes, so we classify these contexts as non-overlapping.

\paragraph{Reacher (Easy/Hard)-Perm (DMC; non-overlapping).}
These environments consider actuator permutation-variants of Reacher (Easy/Hard). Context variables are (i) arm-length scaling and (ii) an actuator permutation factor $q\in\{0,1\}$ that exchanges the actuation of the two reacher links shoulder and wrist. For $q=0$, actions are applied normally; for $q=1$, the action dimensions are swapped. The permutation produces incompatible control laws, yielding non-overlapping structure.

\paragraph{BallInCup (DMC; overlapping).}
We use the ball\_in\_cup-catch-v0 task, where an actuated cup moves in the vertical plane to swing and catch a ball attached by a tendon. The reward signal is sparse, i.e., $+1$ if the ball is in the cup, $0$ otherwise. The environment is contextualized by (i) the tendon length and (ii) gravity. Since both context variables are continuous physical parameters that typically induce gradual changes in the swing dynamics, robust policies often transfer across ranges; we therefore classify this family as overlapping.

\paragraph{Walker (DMC; overlapping).}
We use the walker-walk-v0 task, where a planar walker is rewarded for moving forward. The environment is contextualized by (i) actuator strength and (ii) gravity. Both are continuous parameters that induce smooth variations in locomotion difficulty and dynamics; we classify this family as overlapping.

\paragraph{WalkerGym (Gymnasium; overlapping).}
We use the Walker2d-v5 task, where a planar walker is rewarded for moving forward. The environment is contextualized by (i) actuator strength and (ii) gravity. Both are continuous and typically yield gradual policy variation; we classify this family as overlapping.

\paragraph{HopperGym (Gymnasium; overlapping).}
We use the Hopper-v5 task, where a planar hopper is rewarded for moving forward without falling over. The environment is contextualized by (i) actuator strength and (ii) gravity. Both are continuous and generally yield gradual variations in feasible hopping gaits; we classify this family as overlapping.

\begin{table}[ht]
\caption{\textbf{Context ranges and return bounds for AIB environments.} 
Each environment samples a context per episode, held fixed throughout. Contexts are two-dimensional for all environments (except ODE-$k$, which uses $k$ coefficients with identical per-coefficient ranges to ODE). Columns Train, Eval-in, Eval-out specify the sampling supports for $\mathcal{C}_{\text{train}}$, $\mathcal{C}_{\text{eval-in}}$, and $\mathcal{C}_{\text{eval-out}}$, respectively; while supports may overlap, the sampled context instances are pairwise disjoint across sets. The last column gives return bounds $[J^{\mathrm{lo}}_E, J^{\mathrm{hi}}_E]$ used for score normalization.}
\label{tab:envs}
\centering
\fontsize{6pt}{6pt}\selectfont
\begin{tabular}{@{}llccccc@{}}
\toprule
& & \multicolumn{3}{c}{\textbf{Context ranges}} & \\
\cmidrule(lr){3-5}
\textbf{Environment} & \textbf{Context} & Train & Eval-in & Eval-out & \textbf{Training steps} &\textbf{Return bounds} \\
\midrule
\multirow{2}{*}{DI} 
  & mass & $[0.5,1.5]$ & $(0.5,1.5)$ & $[0.1,0.5)\cup(1.5,2.0]$ & \multirow{2}{*}{\num{100000}} & \multirow{2}{*}{$[0,100]$} \\
  & actuator inversion & $\{\pm1\}$ & $\{\pm1\}$ & $\{\pm1\}$ & \\
\midrule
\multirow{2}{*}{DI-Friction} 
  & mass & $[0.5,1.5]$ & $(0.5,1.5)$ & $[0.1,0.5)\cup(1.5,2.0]$ & \multirow{2}{*}{\num{100000}} & \multirow{2}{*}{$[0,100]$} \\
  & friction & $[0.5,1.5]$ & $(0.5,1.5)$ & $[0.1,0.5)\cup(1.5,2.0]$ & \\
\midrule
\multirow{2}{*}{DI-Perm} 
  & mass & $[0.5,1.5]$ & $(0.5,1.5)$ & $[0.1,0.5)\cup(1.5,2.0]$ & \multirow{2}{*}{\num{100000}} & \multirow{2}{*}{$[0,100]$} \\
  & actuator permutation & $\{0,1\}$ & $\{0,1\}$ & $\{0,1\}$ & \\
\midrule
\multirow{2}{*}{ODE} 
  & $c_1$ & $[-5,5]$ & $(-5,5)$ & $[-10,-5)\cup(5,10]$ & \multirow{2}{*}{\num{100000}} & \multirow{2}{*}{$[0,200]$} \\
  & $c_2$ & $[-5,5]$ & $(-5,5)$ & $[-10,-5)\cup(5,10]$ & \\
\midrule
ODE-$k$ & all $c_i$ ($i\in\{1,\ldots,k\}$) & $[-5,5]$ & $(-5,5)$ & $[-10,-5)\cup(5,10]$ & \num{100000} & $[0,200]$ \\
\midrule
\multirow{2}{*}{Cartpole} 
  & pole length & $[0.3,0.85]$ & $(0.3,0.85)$ & $[0.1,0.3)\cup(0.85,2.0]$ & \multirow{2}{*}{\num{100000}} & \multirow{2}{*}{$[0,1000]$} \\
  & actuator inversion & $\{\pm1\}$ & $\{\pm1\}$ & $\{\pm1\}$ & \\
\midrule
\multirow{2}{*}{Cheetah} 
  & leg length & $[0.8,1.2]$ & $(0.8,1.2)$ & $[0.4,0.8)\cup(1.2,1.6]$ & \multirow{2}{*}{\num{200000}} & \multirow{2}{*}{$[0,1000]$} \\
  & actuator inversion & $\{\pm1\}$ & $\{\pm1\}$ & $\{\pm1\}$ & \\
\midrule
\multirow{2}{*}{Reacher (E/H)} 
  & arm length & $[0.8,1.2]$ & $(0.8,1.2)$ & $[0.4,0.8)\cup(1.2,1.6]$ & \multirow{2}{*}{\num{200000}} & \multirow{2}{*}{$[0,1000]$} \\
  & actuator inversion & $\{\pm1\}$ & $\{\pm1\}$ & $\{\pm1\}$ & \\
\midrule
\multirow{2}{*}{Reacher (E/H)-Perm} 
  & arm length & $[0.8,1.2]$ & $(0.8,1.2)$ & $[0.4,0.8)\cup(1.2,1.6]$ & \multirow{2}{*}{\num{200000}} & \multirow{2}{*}{$[0,1000]$} \\
  & actuator permutation & $\{0,1\}$ & $\{0,1\}$ & $\{0,1\}$ & \\
\midrule
\multirow{2}{*}{BallInCup} 
  & gravity & $[8.0,12.0]$ & $(8.0,12.0)$ & $[1.0,8.0)\cup(12.0,20.0]$ & \multirow{2}{*}{\num{200000}} & \multirow{2}{*}{$[0,1000]$} \\
  & tendon length & $[0.24,0.36]$ & $(0.24,0.36)$ & $[0.1,0.24)\cup(0.36,0.5]$ & \\
\midrule
\multirow{2}{*}{Walker} 
  & gravity & $[4.9,14.7]$ & $(4.9,14.7)$ & $[1.0,4.9)\cup(14.7,19.6]$ & \multirow{2}{*}{\num{200000}} & \multirow{2}{*}{$[0,1000]$} \\
  & actuator strength & $[0.5,1.5]$ & $(0.5,1.5)$ & $[0.1,0.5)\cup(1.5,2.0]$ & \\
\midrule
\multirow{2}{*}{WalkerGym} 
  & gravity & $[4.9,14.7]$ & $(4.9,14.7)$ & $[1.0,4.9)\cup(14.7,19.6]$ & \multirow{2}{*}{\num{500000}} & \multirow{2}{*}{$[0,5000]$} \\
  & actuator strength & $[0.5,1.5]$ & $(0.5,1.5)$ & $[0.1,0.5)\cup(1.5,2.0]$ & \\
\midrule
\multirow{2}{*}{HopperGym} 
  & gravity & $[4.9,14.7]$ & $(4.9,14.7)$ & $[1.0,4.9)\cup(14.7,19.6]$ & \multirow{2}{*}{\num{500000}} & \multirow{2}{*}{$[0,3800]$} \\
  & actuator strength & $[0.5,1.5]$ & $(0.5,1.5)$ & $[0.1,0.5)\cup(1.5,2.0]$ & \\
\bottomrule
\end{tabular}
\end{table}

\subsection{Comparison with Existing Benchmarks}\label{sec:Compbench}
We distinguish AIB from two widely used contextual/multitask RL benchmarks: CARL~\citep{benjamins2023contextualize} and Meta-World~\citep{yu2020meta,mclean2025meta}.

\textbf{CARL} contextualizes standard RL environments by exposing environment parameters as an explicit context (e.g., gravity, mass, friction, actuator strength)~\citep{benjamins2023contextualize}. These variations are primarily continuous parameter shifts designed to preserve the environment's intended semantics, and the benchmark does not target discontinuous context-to-dynamics changes in how actions affect state transitions.

\textbf{Meta-World} provides 50 qualitatively distinct robotic manipulation tasks (e.g., door opening, drawer closing, button pressing, object pushing) with parametric variation in goal and object positions within each task~\citep{yu2020meta}. The benchmark evaluates whether agents can learn multiple skills and transfer to held-out tasks. Within any given task, the underlying transition dynamics remain fixed; only goals and object configurations change and the benchmark does not target context-induced dynamics shifts.

\textbf{AIB} still encompasses smooth physical parameter variations but complements these benchmarks by explicitly targeting discontinuous context-to-dynamics structure. In non-overlapping AIB settings, identical actions can induce incompatible effects across modes, e.g., opposite forces under inversion or swapped physical effects under actuator permutation, so optimal policies need not vary smoothly with context. Such discontinuities can arise in practice from coordinate-frame mismatches, mirrored joints, swapped motor polarities, swapped control channels, or interface inversions in teleoperation. See Remark~\ref{rem:Nonovpract}. Neither CARL nor Meta-World is designed to isolate such failure modes, making AIB a targeted \textit{stress test} for zero-shot generalization under discontinuous context shifts.

\clearpage
\section{Ablations and Design Rationale}
\label{sec:ablations}

We briefly summarize the key design choices in DMA*-SH. Additional implementation details for the context encoder and hypernetwork appear in Appendix~\ref{sec:hyperparameters}.
See Appendices~\ref{app:varcompress}--\ref{app:gradcon} for extended discussion of how these choices shape the learned context geometry.

\textbf{Input masking.}
Hypernetworks can amplify small perturbations in $z_t$ into large changes in $\omega$. Random masking regularizes this pathway by encouraging $g_\phi$ to learn a redundant, distributed code that spreads information across the context window and across coordinates, so that $\omega=h_\eta(z_t)$ remains context-sensitive but less dependent on any single input component. A masking ratio of $40\%$ performed best in our experiments, and performance is robust across a range of substantial masking ratios (Figure~\ref{fig:iqm-mask}).

\textbf{AvgL1Norm input normalization.}
After masking and a linear projection, we normalize per sample using AvgL1Norm. Among the input-normalization options we tested within the context window, including LayerNorm~\citep{ba2016layer}, AvgL1Norm~\citep{fujimoto2023sale}, SimNorm~\citep{lavoie2023simplicial,hansen2024tdmpc2}, and WindowNorm, AvgL1Norm yielded the most reliable performance (Figure~\ref{fig:iqm-innorm}).

\textbf{SimNorm output normalization.}
Normalizing the context embedding $z_t$ is critical for stable online training. Among LayerNorm, AvgL1Norm, and SimNorm, the best performance in our experiments was achieved with SimNorm (Figure~\ref{fig:iqm-outnorm}).

\textbf{Adapter weight sharing.}
Sharing a single dynamics-trained hypernetwork with the policy and Q-function is more effective than training separate hypernetworks for the RL modules. An ablation using separate hypernetworks is presented in Figure~\ref{fig:iqm-hyper}. Also see Appendix~\ref{app:gradcon} and Figure~\ref{fig:gradcon} for an extended discussion.

We perform a range of ablations that motivate our design choices. For these ablations, results are aggregated across a subset of contextualized environments from the AIB benchmark (Section~\ref{sec:environments} and Appendix~\ref{app:env_details}): DI, DI-Friction, ODE, Cartpole, BallInCup, and Walker.

Figure~\ref{fig:poi} reports probability of improvement (PoI) for DMA* and DMA*-SH, following \citet{agarwal2021deep}. PoI indicates how likely one design choice improves over another, but it does not quantify the magnitude of the improvement. In Figure~\ref{fig:iqm} we compare vanilla DMA to DMA* and DMA*-SH, indicating that the design choices cumulatively have a significant impact.

In Figure~\ref{fig:iqm-mask} we compare IQM scores \citep{agarwal2021deep} for different ratios of random input masking of actions, states, and next-state differences in $\tau_t^c$, suggesting that $20\%$ is beneficial for DMA* and
$40\%$ is beneficial for DMA*-SH. Especially for DMA*-SH, substantial performance degradation appears only at relatively high masking ratios, indicating robustness to variation in the sampled context inputs $\tau_t^c$.

In Figures~\ref{fig:iqm-innorm} and~\ref{fig:iqm-outnorm} we compare different normalization variants for the input and output of the context encoder. The intuition about AvgL1Norm and SimNorm provided in Section~\ref{sec:dma} and in prior work \citep{fujimoto2023sale,lavoie2023simplicial,hansen2024tdmpc2} is also reflected in performance. Our dynamics-alignment loss~\eqref{eq:dma}, which operates on state differences, encourages the encoder to organize $z_t$ according to relative differences between contexts rather than absolute magnitudes, consistent with a scale-insensitive perspective; see
Appendix~\ref{sec:directional_bias} for discussion.

Figure~\ref{fig:iqm-window} motivates the choice of window size $K=24$ for DI, DI-Friction, and ODE, and $K=128$ for the DMC- and Gymnasium-based environments. The impact of $K$ appears modest once a sufficiently large minimum window size is used.

Overall, the ablations indicate that appropriate normalization is important for dynamics-aligned context encoders in zero-shot generalization for contextual RL, and that input masking can further improve performance. This is also reflected in Figure~\ref{fig:iqm-hyper} for DMA*-SH.

Furthermore, Figure~\ref{fig:iqm-hyper} shows that the shared hypernetwork design outperforms an architecture with separate hypernetworks for the dynamics model, policy, and Q-function (DMA*-H; Section~\ref{app:gradcon}). To further isolate the contribution of dynamics alignment for RL, we constructed a variant DMA*-H (RL only) that applies a hypernetwork to the RL modules but not to the dynamics model, so the RL adapter weights are not dynamics-aligned. Apart from a KL-loss term and a contrastive-loss term, DMA*-H (RL only) closely mirrors the hypernetwork structure of R2PGO \citep{li2024efficient} in an online RL setting. 
Allowing actor gradients to propagate directly through the shared hypernetwork leads to near-zero returns (DMA*-SH, actor non-detached). This shows that letting RL rewrite the shared $\omega$ is harmful, and that gradient detachment is necessary; see Appendix~\ref{sec:TxrecurDMASH}. 
Figure~\ref{fig:iqm-hyper} also shows that normalization, masking, hypernetwork sharing, and dynamics-model alignment are all beneficial; see Appendix~\ref{sec:directional_bias} for a detailed discussion.

For the unaware, recurrent baseline Amago-v2 \citep{grigsby2024bamago}, we use a GRU trajectory encoder instead of the default Transformer variant. This is justified by overall slightly stronger performance (Figure~\ref{fig:iqm-amago}).

PEARL \citep{rakelly2019efficient} is used as a baseline to compare against DMA* and DMA*-SH. PEARL was originally proposed for meta-RL settings where the latent context captures task variation via rewards, and the context encoder is trained through RL losses. To better match our dynamics-aligned setting, we consider a variant that couples PEARL's probabilistic context encoder to a dynamics model, which we denote as DMA-Pearl. A comparison of PEARL and DMA-Pearl is provided in Figure~\ref{fig:pearl}. Furthermore, Figure~\ref{fig:pearl_kl} reports performance for different choices of $\beta$ that weight the KL regularization in PEARL. DMA-Pearl improves over vanilla DMA (Table~\ref{tab:aer}), highlighting the benefits of a probabilistic context encoder with KL regularization in this setting. However, integrating these design elements into DMA* and DMA*-SH does not yield further gains (Figure~\ref{fig:pearl}).

We also conducted experiments with VariBAD \citep{zintgraf2020varibad}, testing two different KL-weights $\beta$. While it achieves comparable performance in the overlapping DI-Friction setting, it struggles considerably with the non-overlapping contextualizations in DI and ODE. For this reason, and given that we already include DMA-Pearl in our comparisons, we exclude VariBAD as a baseline in the remainder of the paper (Section~\ref{sec:baselines}).

Methods built on smooth latent dynamics priors such as VariBAD and DMA-Pearl can struggle when their objectives explicitly encourage latent embeddings to vary continuously with respect to context trajectories. This inductive bias can be mismatched to tasks whose true context-to-dynamics map exhibits genuine discontinuities (as in DI with actuator inversion), where the correct representation requires a sign flip rather than a smooth interpolation (Remark~\ref{rem:VariBAD}). In contrast, DMA*-SH can represent such discontinuities through multiplicative modulation in the hypernetwork pathway, which allows sharp directional changes in the induced policy and critic without an explicit ELBO or latent-prior penalty. As a result, DMA*-SH avoids the continuity bias inherent to latent-prior methods and can realize the functional geometry required for discontinuous contextual RL.

\begin{figure}[ht]
    \centering
    \begin{subfigure}[b]{0.95\textwidth}
        \centering
        \includegraphics[height=2.5cm]{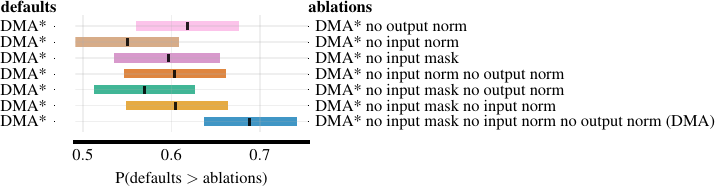}
        \caption{DMA*.}
    \end{subfigure}
    \par\bigskip
    \begin{subfigure}[b]{0.95\textwidth}
        \centering
        \includegraphics[height=2.5cm]{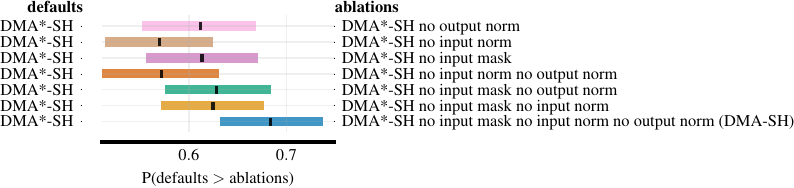}
        \caption{DMA*-SH.}
    \end{subfigure}
    \caption{Probability of improvement (POI) \citep{agarwal2021deep} based on AER scores (Section~\ref{sec:metrics}) aggregated over six contextualized environments and over contexts in the three context sets $\mathcal{C}_{\text{train}}$, $\mathcal{C}_{\text{eval-in}}$ and $\mathcal{C}_{\text{eval-out}}$. For the proposed DMA* and DMA*-SH, we ablate separately the random masking, input and output normalization, a combination of two, or everything at once (DMA, DMA-SH).
    }
    \label{fig:poi}
\end{figure}

\begin{figure}[ht]
    \centering
    \begin{subfigure}[b]{0.5\textwidth}
        \centering
        \includegraphics[width=0.95\textwidth]{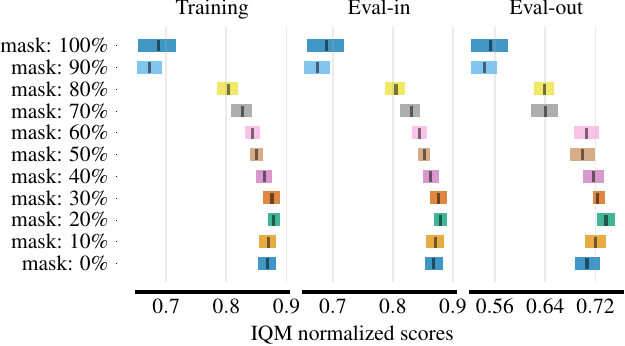}
        \caption{DMA*.}
    \end{subfigure}\hfill
    \begin{subfigure}[b]{0.5\textwidth}
        \centering
        \includegraphics[width=0.95\textwidth]{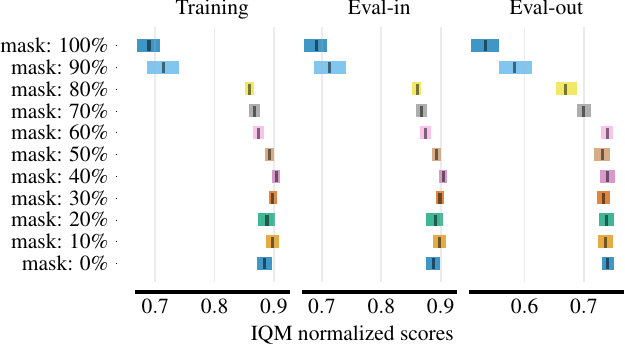}
        \caption{DMA*-SH.}
    \end{subfigure}
    \caption{Interquartile mean (IQM) \citep{agarwal2021deep} based on AER scores (Section~\ref{sec:metrics}) aggregated over six contextualized environments. We distinguish results for contexts in the three context sets $\mathcal{C}_{\text{train}}$, $\mathcal{C}_{\text{eval-in}}$ and $\mathcal{C}_{\text{eval-out}}$. We compare different ratios for the random input masking. When averaging over the three context sets, best performance is achieved using a ratio of $20\%$ for DMA* and $40\%$ for DMA*-SH.
    }
    \label{fig:iqm-mask}
\end{figure}

\begin{figure}[ht]
    \centering
    \begin{subfigure}[b]{0.5\textwidth}
        \centering
        \includegraphics[width=0.98\textwidth]{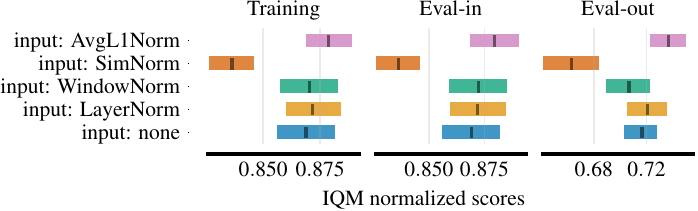}
        \caption{DMA*.}
    \end{subfigure}\hfill
    \begin{subfigure}[b]{0.5\textwidth}
        \centering
        \includegraphics[width=0.98\textwidth]{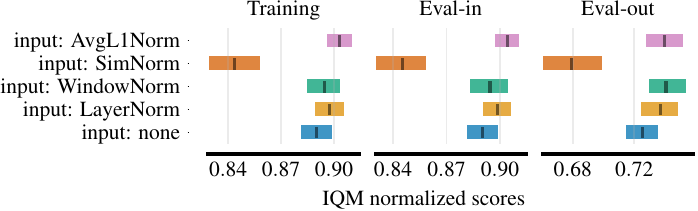}
        \caption{DMA*-SH.}
    \end{subfigure}
    \caption{Interquartile mean (IQM) \citep{agarwal2021deep} based on AER scores (Section~\ref{sec:metrics}) aggregated over six contextualized environments. We distinguish results for contexts in the three context sets $\mathcal{C}_{\text{train}}$, $\mathcal{C}_{\text{eval-in}}$ and $\mathcal{C}_{\text{eval-out}}$. We compare different types of input normalization. When averaging over the three context sets, best performance is achieved using AvgL1Norm in both DMA* and DMA*-SH.
    }
    \label{fig:iqm-innorm}
\end{figure}

\begin{figure}[ht]
    \centering
    \begin{subfigure}[b]{0.5\textwidth}
        \centering
        \includegraphics[width=0.98\textwidth]{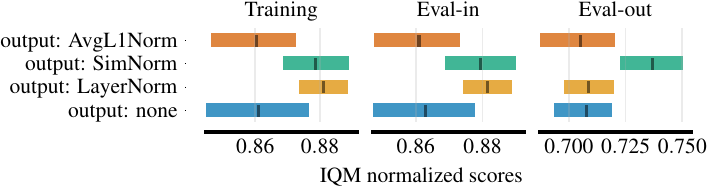}
        \caption{DMA*.}
    \end{subfigure}\hfill
    \begin{subfigure}[b]{0.5\textwidth}
        \centering
        \includegraphics[width=0.98\textwidth]{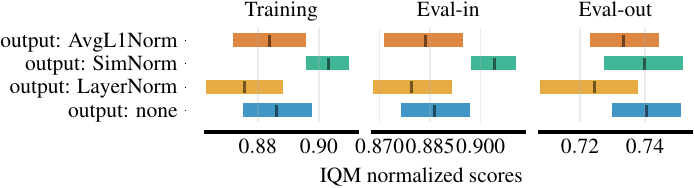}
        \caption{DMA*-SH.}
    \end{subfigure}
    \caption{Interquartile mean (IQM) \citep{agarwal2021deep} based on AER scores (Section~\ref{sec:metrics}) aggregated over six contextualized environments. We distinguish results for contexts in the three context sets $\mathcal{C}_{\text{train}}$, $\mathcal{C}_{\text{eval-in}}$ and $\mathcal{C}_{\text{eval-out}}$. We compare different types of output normalization. When averaging over the three context sets, best performance is achieved using SimNorm in both DMA* and DMA*-SH.
    }
    \label{fig:iqm-outnorm}
\end{figure}

\begin{figure}[ht]
    \centering
    \begin{subfigure}[b]{0.5\textwidth}
        \centering
        \includegraphics[width=0.92\textwidth]{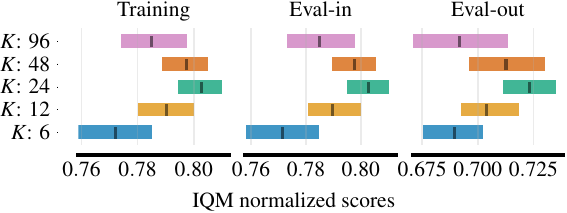}
        \caption{DMA* for DI and ODE.}
    \end{subfigure}\hfill
    \begin{subfigure}[b]{0.5\textwidth}
        \centering
        \includegraphics[width=0.92\textwidth]{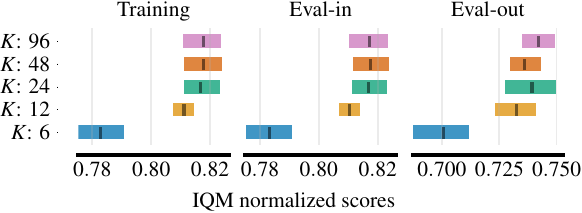}
        \caption{DMA*-SH for DI and ODE.}
    \end{subfigure}
    \par\bigskip
    \begin{subfigure}[b]{0.5\textwidth}
        \centering
        \includegraphics[width=0.92\textwidth]{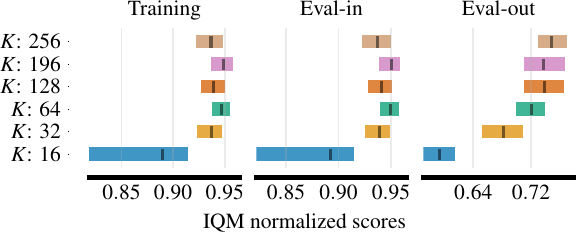}
        \caption{DMA* for DMC and Gymnasium environments.}
    \end{subfigure}\hfill
    \begin{subfigure}[b]{0.5\textwidth}
        \centering
        \includegraphics[width=0.92\textwidth]{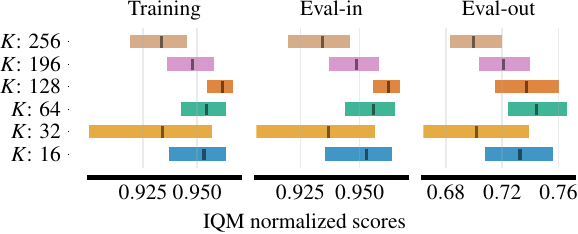}
        \caption{DMA*-SH for DMC and Gymnasium environments.}
    \end{subfigure}
    \caption{{Interquartile mean (IQM) comparing different context window sizes justifying the choice of 24 for DI and ODE environments and 128 for DMC- and Gymnasium-based environments.}
    }
    \label{fig:iqm-window}
\end{figure}

\begin{figure}[ht]
    \centering
\centering
    \includegraphics[width=0.7\textwidth]{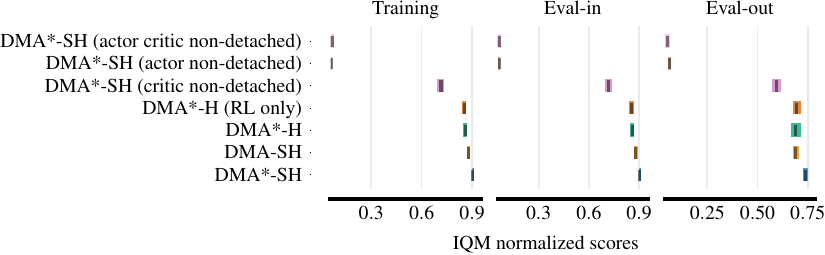}
    \caption{
    Interquartile mean (IQM) \citep{agarwal2021deep} based on AER scores (Section~\ref{sec:metrics}) aggregated over six contextualized environments. We distinguish results for contexts in the three context sets $\mathcal{C}_{\text{train}}$, $\mathcal{C}_{\text{eval-in}}$, and $\mathcal{C}_{\text{eval-out}}$. We compare DMA*-SH to a variant without normalization and masking (DMA-SH) and to an architecture that does not share the hypernetwork (DMA*-H). Instead, DMA*-H uses separate hypernetworks for the dynamics model, policy, and Q-value function. Apart from a KL-loss term and a contrastive-loss term, DMA*-H (RL only) closely resembles R2PGO \citep{li2024efficient} in an online RL setting. It does not employ a hypernetwork for the dynamics model, so the adapter weights for the RL modules are not aligned with the dynamics model. We also ablate gradient flow through the actor and/or critic pathways, denoted DMA*-SH (actor/critic non-detached), by allowing RL gradients to update the shared hypernetwork. Results show that detachment is critical, especially in non-overlapping settings; see Appendix~\ref{sec:TxrecurDMASH}. Our results show that normalization, masking, hypernetwork sharing, and dynamics-model alignment are all beneficial. For a detailed discussion, see Appendix~\ref{sec:directional_bias}.
    }
    \label{fig:iqm-hyper}
\end{figure}

\begin{figure}[ht]
    \centering
\centering
    \includegraphics[width=0.55\textwidth]{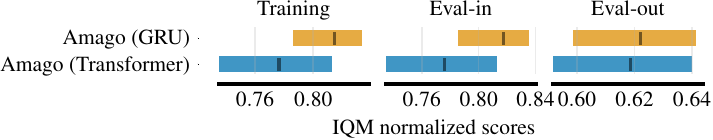}
    \caption{
    Interquartile mean (IQM) \citep{agarwal2021deep} based on AER scores (Section~\ref{sec:metrics}) aggregated over six contextualized environments. We distinguish results for contexts in the three context sets $\mathcal{C}_{\text{train}}$, $\mathcal{C}_{\text{eval-in}}$, and $\mathcal{C}_{\text{eval-out}}$. For Amago, we compare GRU- and Transformer-based trajectory encoders, justifying the use of the GRU variant.
    }
    \label{fig:iqm-amago}
\end{figure}

\begin{figure}[ht]
    \centering
    \begin{subfigure}[b]{0.5\textwidth}
        \centering
        \includegraphics[width=0.98\textwidth]{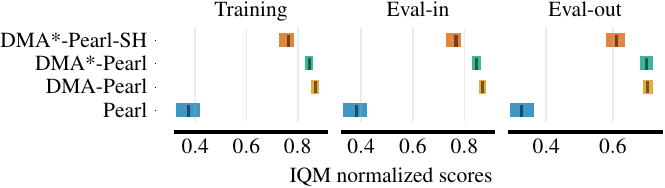}
        \caption{Modified Pearl.}
        \label{fig:pearl}
    \end{subfigure}\hfill
    \begin{subfigure}[b]{0.5\textwidth}
        \centering
        \includegraphics[width=0.98\textwidth]{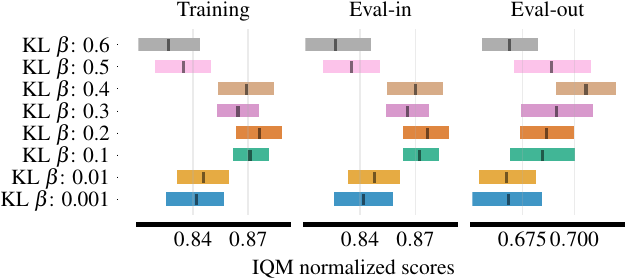}
        \caption{$\beta$ KL weighting.}
        \label{fig:pearl_kl}
    \end{subfigure}
    \caption{Interquartile mean (IQM) \citep{agarwal2021deep} based on AER scores (Section~\ref{sec:metrics}) aggregated over six contextualized environments. We distinguish results for contexts in the three context sets $\mathcal{C}_{\text{train}}$, $\mathcal{C}_{\text{eval-in}}$ and $\mathcal{C}_{\text{eval-out}}$. In a) we compare the original Pearl approach aligned with the Q-function to the dynamic model-aligned variant that we are using as a baseline, DMA-Pearl. Additionally, we incorporate our additions to DMA and the shared hypernetwork context utilization to Pearl. In b) we test different $\beta$ weighting parameters for the KL term in Pearl and decided for $\beta=0.4$ when using DMA-Pearl as a baseline.
    }
    \label{fig:iqm-pearl}
\end{figure}

\begin{figure}[!htb]
    \centering
    \includegraphics[width=0.75\textwidth]{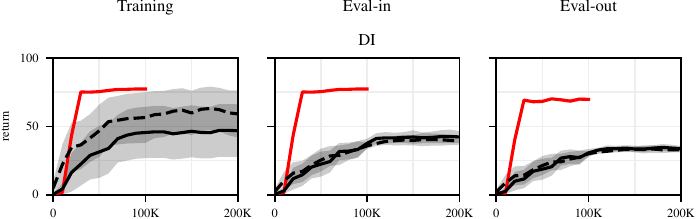}\\\vspace{0.2cm}
    \includegraphics[width=0.75\textwidth]{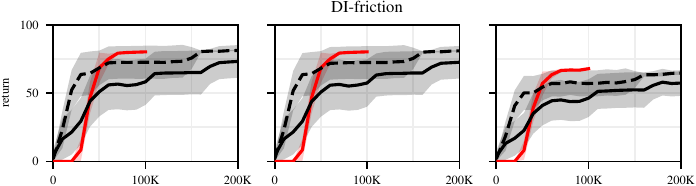}\\\vspace{0.2cm}
    \includegraphics[width=0.75\textwidth]{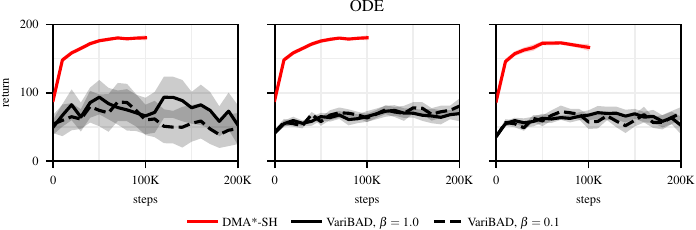}
    \caption{{Returns over steps, averaged over the contexts used for training and evaluation, respectively. Comparison to the meta-RL approach VariBAD \citep{zintgraf2020varibad}. See Remark~\ref{rem:VariBAD}. VariBAD is based on the on-policy PPO, hence we allow for more environment steps. We do not observe any improvement after 200K steps. Two KL-weights ($\beta$) are tested.}}
    \label{fig:returns_varibad}
\end{figure}

\clearpage

\section{Detailed Results}
\label{sec:detailed-results}

\subsection{Separate Result Tables, Learning Curves and Performance Gains}

Table~\ref{tab:aer} aggregates AER scores over the three context sets. For a more detailed view, Tables~\ref{tab:aer-train}--\ref{tab:aer-eval-out} report results separately on $\mathcal{C}_{\text{train}}$, $\mathcal{C}_{\text{eval-in}}$, and $\mathcal{C}_{\text{eval-out}}$. DMA*-SH performs favorably across all three splits.

Figure~\ref{fig:returns} presents learning curves for six environments, shown separately for each context set. Depending on the environment, we run $\num{100000}$--$\num{500000}$ gradient update steps and, for each context instance, the same number of environment steps. Since we train with $n_c=20$ contexts, this corresponds to $\num{2000000}$--$\num{10000000}$ total environment steps (Table~\ref{tab:envs}). DMA*-SH exhibits consistently strong performance across these settings.

\begin{table}[ht]
  \caption{
    AER scores with $95\%$ confidence intervals for all environments. Results for the \textbf{context set $\mathcal{C}_{\text{train}}$}, comparing DMA*, DMA*-SH, and all baselines. Best AER scores are in bold; if multiple methods are highlighted for an environment, their differences are not distinguishable under the probability of improvement with $95\%$ confidence intervals \citep{agarwal2021deep}. \textit{All} aggregates all environments, \textit{Overlap} aggregates those with overlapping contexts, and \textit{Non-overlap} aggregates those with non-overlapping contexts. For the rows \textit{All}, \textit{Overlap}, and \textit{Non-overlap}, scores are aggregated after environment-wise min--max scaling using environment-specific return bounds (Table~\ref{tab:envs}).
  }
  \fontsize{5pt}{5pt}\selectfont
  \centering
  \label{tab:aer-train}
  \begin{tabular}{lrrrrrrrr}
    \toprule
    &\multicolumn{2}{c}{Context-Aware}&\multicolumn{2}{c}{Context-Unaware}&\multicolumn{4}{c}{Context-Inferred} \\
    \cmidrule(l){2-3}\cmidrule(l){4-5}\cmidrule(l){6-9}
    Name & Concat & DA & DR & Amago& DMA & DMA-Pearl & \textbf{DMA*} & \textbf{DMA*-SH} \\
    \midrule

    DI & 75 [72, 77] & \textbf{78 [77, 78]} & 17 [8, 25] & \textbf{66 [55, 75]} & 74 [73, 75] & 73 [71, 74] & 77 [76, 78] & \textbf{78 [78, 79]} \\
    DI-Friction & 71 [55, 79] & 80 [79, 80] & 72 [56, 81] & \textbf{82 [81, 82]} & 62 [45, 74] & 79 [78, 80] & 71 [55, 80] & 81 [80, 81] \\
    DI-Perm & 77 [76, 78] & \textbf{78 [77, 79]} & 54 [43, 63] & \textbf{77 [71, 81]} & 76 [76, 77] & 75 [74, 76] & 76 [74, 77] & \textbf{78 [77, 79]} \\
    ODE & \textbf{180 [175, 184]} & \textbf{183 [179, 188]} & 63 [54, 73] & \textbf{178 [176, 180]} & 173 [170, 176] & \textbf{174 [167, 181]} & \textbf{179 [175, 183]} & \textbf{183 [179, 186]} \\
    Cartpole & 929 [912, 945] & 934 [902, 960] & 658 [611, 706] & 667 [586, 755] & 919 [892, 944] & 904 [854, 943] & 941 [919, 961] & \textbf{972 [959, 981]} \\
    Cheetah & 437 [418, 457] & 452 [438, 469] & 319 [293, 346] & \textbf{488 [475, 505]} & \textbf{461 [429, 491]} & 428 [403, 458] & \textbf{463 [443, 485]} & \textbf{484 [462, 504]} \\
    Reacher (E) & \textbf{923 [887, 954]} & \textbf{916 [870, 955]} & 584 [532, 639] & \textbf{942 [929, 954]} & \textbf{903 [861, 939]} & \textbf{912 [877, 944]} & \textbf{925 [898, 950]} & \textbf{934 [917, 950]} \\
    Reacher (E)-Perm & \textbf{916 [883, 946]} & \textbf{899 [857, 932]} & 691 [655, 731] & \textbf{910 [888, 930]} & \textbf{921 [888, 952]} & \textbf{929 [900, 954]} & \textbf{902 [875, 929]} & \textbf{933 [909, 954]} \\
    Reacher (H) & 724 [606, 823] & 756 [704, 811] & 290 [210, 381] & \textbf{913 [897, 927]} & 700 [616, 780] & 735 [649, 801] & 743 [650, 825] & \textbf{849 [773, 908]} \\
    Reacher (H)-Perm & 768 [666, 859] & 804 [762, 849] & 363 [261, 468] & 807 [753, 852] & 716 [604, 799] & 731 [676, 793] & 727 [623, 813] & \textbf{895 [851, 934]} \\
    BallInCup & \textbf{976 [973, 977]} & 973 [971, 974] & 961 [940, 974] & 728 [577, 868] & \textbf{975 [974, 977]} & \textbf{977 [975, 978]} & \textbf{975 [973, 976]} & \textbf{975 [973, 977]} \\
    Walker & \textbf{900 [890, 907]} & 882 [864, 898] & \textbf{907 [898, 916]} & 845 [831, 857] & 865 [812, 898] & \textbf{908 [901, 915]} & 885 [866, 898] & \textbf{915 [905, 925]} \\
    WalkerGym & 3068 [2751, 3353] & 3834 [3663, 4030] & 3171 [2946, 3400] & \textbf{4551 [4217, 4894]} & 3502 [3264, 3733] & 3571 [3282, 3876] & 3333 [2960, 3718] & 3726 [3606, 3820] \\
    HopperGym & \textbf{2847 [2740, 2947]} & \textbf{2851 [2778, 2926]} & 2664 [2590, 2737] & 2792 [2704, 2877] & \textbf{2939 [2845, 3026]} & \textbf{2938 [2864, 3003]} & \textbf{2956 [2895, 3024]} & 2853 [2829, 2879] \\
    \midrule
    All & 0.79 [0.77, 0.81] & 0.81 [0.81, 0.82] & 0.56 [0.54, 0.58] & 0.79 [0.78, 0.8] & 0.78 [0.76, 0.8] & 0.8 [0.79, 0.8] & 0.8 [0.78, 0.81] & \textbf{0.84 [0.83, 0.84]} \\
    Overlap & 0.79 [0.75, 0.82] & \textbf{0.83 [0.82, 0.84]} & 0.79 [0.75, 0.81] & \textbf{0.81 [0.78, 0.84]} & 0.79 [0.75, 0.82] & \textbf{0.83 [0.82, 0.85]} & 0.8 [0.76, 0.83] & \textbf{0.84 [0.83, 0.84]} \\
    Non-overlap & 0.79 [0.77, 0.81] & 0.8 [0.79, 0.81] & 0.44 [0.41, 0.46] & 0.78 [0.77, 0.8] & 0.78 [0.75, 0.8] & 0.78 [0.77, 0.79] & 0.79 [0.77, 0.81] & \textbf{0.84 [0.83, 0.85]} \\
    
    \bottomrule
  \end{tabular}
\end{table}

\begin{table}[ht]
  \caption{
      AER scores with $95\%$ confidence intervals for all environments. Results for the \textbf{context set $\mathcal{C}_{\text{eval-in}}$}, comparing DMA*, DMA*-SH, and all baselines. Best AER scores are in bold; if multiple methods are highlighted for an environment, their differences are not distinguishable under the probability of improvement with $95\%$ confidence intervals \citep{agarwal2021deep}. \textit{All} aggregates all environments, \textit{Overlap} aggregates those with overlapping contexts, and \textit{Non-overlap} aggregates those with non-overlapping contexts. For the rows \textit{All}, \textit{Overlap}, and \textit{Non-overlap}, scores are aggregated after environment-wise min--max scaling using environment-specific return bounds (Table~\ref{tab:envs}).
  }
  \fontsize{5pt}{5pt}\selectfont
  \centering
  \label{tab:aer-eval-in}
  \begin{tabular}{lrrrrrrrr}
    \toprule
    &\multicolumn{2}{c}{Context-Aware}&\multicolumn{2}{c}{Context-Unaware}&\multicolumn{4}{c}{Context-Inferred} \\
    \cmidrule(l){2-3}\cmidrule(l){4-5}\cmidrule(l){6-9}
    Name & Concat & DA & DR & Amago& DMA & DMA-Pearl & \textbf{DMA*} & \textbf{DMA*-SH} \\
    \midrule

    DI & 75 [72, 77] & \textbf{78 [77, 78]} & 17 [8, 26] & \textbf{66 [56, 75]} & 74 [73, 75] & 73 [72, 74] & 77 [76, 78] & \textbf{78 [78, 79]} \\
    DI-Friction & 71 [55, 79] & 80 [79, 80] & 72 [56, 81] & \textbf{82 [81, 82]} & 62 [46, 75] & 79 [78, 80] & 71 [55, 80] & 81 [80, 81] \\
    DI-Perm & 77 [76, 78] & \textbf{78 [77, 79]} & 54 [43, 63] & \textbf{77 [71, 81]} & 76 [76, 77] & 75 [74, 76] & \textbf{76 [74, 77]} & \textbf{78 [77, 79]} \\
    ODE & \textbf{181 [175, 185]} & \textbf{183 [178, 187]} & 63 [54, 73] & \textbf{178 [176, 180]} & 173 [170, 176] & \textbf{173 [166, 180]} & \textbf{178 [174, 183]} & \textbf{182 [179, 185]} \\
    Cartpole & 930 [912, 945] & 935 [905, 959] & 659 [611, 706] & 668 [589, 752] & 919 [894, 945] & 905 [855, 943] & 941 [919, 961] & \textbf{972 [961, 981]} \\
    Cheetah & 440 [423, 455] & 457 [444, 472] & 319 [295, 347] & \textbf{486 [471, 504]} & \textbf{463 [436, 493]} & 430 [403, 460] & \textbf{461 [441, 483]} & \textbf{487 [466, 507]} \\
    Reacher (E) & \textbf{924 [887, 954]} & \textbf{917 [869, 954]} & 582 [529, 636] & \textbf{942 [930, 953]} & \textbf{905 [864, 941]} & \textbf{913 [878, 946]} & \textbf{925 [897, 951]} & \textbf{931 [916, 947]} \\
    Reacher (E)-Perm & \textbf{916 [881, 945]} & \textbf{900 [857, 933]} & 692 [655, 734] & \textbf{915 [897, 932]} & \textbf{922 [888, 950]} & \textbf{929 [899, 955]} & \textbf{902 [873, 930]} & \textbf{932 [906, 953]} \\
    Reacher (H) & 724 [613, 822] & 757 [705, 818] & 291 [209, 380] & \textbf{916 [901, 928]} & 702 [620, 787] & 736 [648, 804] & 745 [654, 825] & \textbf{848 [772, 905]} \\
    Reacher (H)-Perm & 769 [672, 860] & 805 [762, 853] & 364 [266, 467] & 812 [755, 856] & 715 [599, 801] & 732 [680, 792] & 728 [622, 821] & \textbf{894 [848, 932]} \\
    BallInCup & \textbf{976 [974, 978]} & 974 [972, 975] & 957 [924, 975] & 731 [586, 868] & \textbf{976 [974, 977]} & \textbf{977 [975, 979]} & \textbf{975 [973, 977]} & \textbf{976 [974, 977]} \\
    Walker & \textbf{904 [894, 911]} & 887 [865, 903] & \textbf{909 [901, 916]} & 849 [836, 861] & 866 [813, 901] & \textbf{912 [905, 919]} & 888 [867, 903] & \textbf{918 [909, 926]} \\
    WalkerGym & 3093 [2760, 3400] & 3845 [3671, 4048] & 3204 [2985, 3424] & \textbf{4601 [4273, 4959]} & 3519 [3279, 3764] & 3585 [3303, 3878] & 3377 [2986, 3748] & 3753 [3619, 3876] \\
    HopperGym & \textbf{2844 [2733, 2942]} & \textbf{2857 [2777, 2944]} & 2673 [2610, 2736] & 2795 [2706, 2883] & \textbf{2947 [2845, 3039]} & \textbf{2954 [2886, 3016]} & \textbf{2976 [2912, 3047]} & 2853 [2830, 2875] \\
    \midrule
    All & 0.79 [0.77, 0.81] & 0.82 [0.81, 0.82] & 0.56 [0.54, 0.59] & 0.79 [0.78, 0.81] & 0.78 [0.76, 0.8] & 0.8 [0.79, 0.81] & 0.8 [0.78, 0.81] & \textbf{0.84 [0.83, 0.84]} \\
    Overlap & 0.79 [0.75, 0.82] & \textbf{0.84 [0.82, 0.85]} & 0.79 [0.75, 0.81] & \textbf{0.81 [0.78, 0.84]} & 0.79 [0.75, 0.82] & \textbf{0.83 [0.82, 0.85]} & 0.81 [0.76, 0.84] & \textbf{0.84 [0.84, 0.84]} \\
    Non-overlap & 0.79 [0.77, 0.81] & 0.8 [0.79, 0.81] & 0.44 [0.41, 0.46] & 0.78 [0.77, 0.8] & 0.78 [0.76, 0.8] & 0.78 [0.77, 0.79] & 0.79 [0.77, 0.81] & \textbf{0.84 [0.83, 0.85]} \\

    \bottomrule
  \end{tabular}
\end{table}

\begin{table}[ht]
  \caption{
      AER scores with $95\%$ confidence intervals for all environments. Results for the \textbf{context set $\mathcal{C}_{\text{eval-out}}$}, comparing DMA*, DMA*-SH, and all baselines. Best AER scores are in bold; if multiple methods are highlighted for an environment, their differences are not distinguishable under the probability of improvement with $95\%$ confidence intervals \citep{agarwal2021deep}. \textit{All} aggregates all environments, \textit{Overlap} aggregates those with overlapping contexts, and \textit{Non-overlap} aggregates those with non-overlapping contexts. For the rows \textit{All}, \textit{Overlap}, and \textit{Non-overlap}, scores are aggregated after environment-wise min--max scaling using environment-specific return bounds (Table~\ref{tab:envs}).
  }
  \fontsize{5pt}{5pt}\selectfont
  \centering
  \label{tab:aer-eval-out}
  \begin{tabular}{lrrrrrrrr}
    \toprule
    &\multicolumn{2}{c}{Context-Aware}&\multicolumn{2}{c}{Context-Unaware}&\multicolumn{4}{c}{Context-Inferred} \\
    \cmidrule(l){2-3}\cmidrule(l){4-5}\cmidrule(l){6-9}
    Name & Concat & DA & DR & Amago& DMA & DMA-Pearl & \textbf{DMA*} & \textbf{DMA*-SH} \\
    \midrule

    DI & 65 [62, 68] & \textbf{70 [69, 71]} & 16 [11, 21] & 52 [44, 59] & 42 [39, 46] & 58 [55, 62] & \textbf{70 [69, 71]} & \textbf{71 [70, 73]} \\
    DI-Friction & 54 [39, 64] & 68 [65, 70] & 61 [47, 70] & \textbf{73 [72, 74]} & 45 [33, 53] & 65 [62, 68] & 62 [47, 70] & 69 [67, 70] \\
    DI-Perm & 65 [61, 68] & \textbf{71 [69, 73]} & 43 [34, 50] & 63 [57, 68] & 50 [47, 55] & 61 [58, 65] & 64 [61, 67] & \textbf{69 [65, 72]} \\
    ODE & 126 [118, 134] & \textbf{172 [165, 179]} & 63 [54, 71] & 148 [147, 148] & 152 [148, 157] & 165 [160, 169] & \textbf{168 [162, 175]} & \textbf{173 [170, 176]} \\
    Cartpole & 731 [700, 762] & 808 [754, 862] & 613 [564, 660] & 581 [529, 637] & 862 [842, 880] & 842 [800, 874] & 901 [873, 930] & \textbf{958 [941, 972]} \\
    Cheetah & \textbf{279 [264, 297]} & 241 [226, 255] & 205 [189, 222] & \textbf{268 [255, 283]} & 239 [221, 258] & 208 [195, 223] & 225 [212, 239] & \textbf{252 [232, 272]} \\
    Reacher (E) & \textbf{824 [779, 872]} & 801 [748, 851] & 526 [478, 577] & \textbf{858 [842, 873]} & 804 [772, 836] & 799 [767, 832] & \textbf{833 [809, 859]} & 829 [802, 858] \\
    Reacher (E)-Perm & \textbf{817 [782, 845]} & \textbf{801 [746, 845]} & 593 [558, 634] & 761 [724, 798] & \textbf{800 [758, 841]} & \textbf{818 [787, 844]} & \textbf{802 [768, 833]} & \textbf{843 [823, 867]} \\
    Reacher (H) & \textbf{601 [506, 692]} & 627 [591, 665] & 217 [140, 305] & \textbf{729 [692, 766]} & 561 [483, 635] & 588 [515, 644] & 603 [523, 668] & \textbf{719 [674, 763]} \\
    Reacher (H)-Perm & 600 [531, 665] & 638 [604, 673] & 257 [201, 313] & 598 [554, 640] & 573 [467, 660] & 541 [466, 619] & 581 [488, 664] & \textbf{731 [690, 773]} \\
    BallInCup & \textbf{821 [792, 846]} & 691 [659, 730] & 667 [611, 714] & 493 [433, 558] & \textbf{784 [760, 811]} & 756 [733, 782] & 749 [705, 790] & 719 [684, 755] \\
    Walker & 546 [521, 571] & \textbf{582 [561, 605]} & \textbf{578 [557, 598]} & \textbf{565 [551, 578]} & 557 [531, 575] & \textbf{591 [577, 605]} & 566 [551, 581] & \textbf{585 [572, 598]} \\
    WalkerGym & 1944 [1773, 2158] & 2283 [2163, 2422] & 2058 [1954, 2162] & \textbf{2619 [2488, 2757]} & 2285 [2177, 2397] & 2329 [2162, 2518] & 2170 [1985, 2344] & 2296 [2186, 2389] \\
    HopperGym & 1872 [1795, 1959] & 1918 [1866, 1969] & 1795 [1737, 1853] & 1917 [1845, 1989] & \textbf{2001 [1948, 2045]} & \textbf{2046 [2003, 2086]} & \textbf{1973 [1917, 2026]} & 1982 [1935, 2024] \\
    \midrule
    All & 0.61 [0.6, 0.63] & 0.65 [0.63, 0.66] & 0.43 [0.41, 0.45] & 0.61 [0.6, 0.61] & 0.59 [0.57, 0.61] & 0.63 [0.62, 0.64] & 0.64 [0.63, 0.65] & \textbf{0.68 [0.68, 0.69]} \\
    Overlap & 0.56 [0.52, 0.59] & \textbf{0.58 [0.57, 0.6]} & 0.55 [0.51, 0.57] & 0.56 [0.55, 0.58] & 0.55 [0.53, 0.57] & \textbf{0.6 [0.59, 0.61]} & \textbf{0.58 [0.54, 0.6]} & \textbf{0.59 [0.59, 0.6]} \\
    Non-overlap & 0.64 [0.63, 0.66] & 0.69 [0.67, 0.71] & 0.37 [0.35, 0.39] & 0.63 [0.62, 0.64] & 0.61 [0.59, 0.64] & 0.65 [0.64, 0.66] & 0.68 [0.66, 0.7] & \textbf{0.73 [0.72, 0.74]} \\
    
    \bottomrule
  \end{tabular}
\end{table}

\begin{table}[ht]
\centering
\caption{Relative performance gains (+\%) of DMA*-SH compared to DR, Concat, and DA across regimes and environment types. The \textit{Aggregated} regime considers results aggregated across context sets $\mathcal{C}_{\text{train}}$, $\mathcal{C}_{\text{eval-in}}$, and $\mathcal{C}_{\text{eval-out}}$ (Table~\ref{tab:aer}). The \textit{Train} regime considers results for the context set $\mathcal{C}_{\text{train}}$ (Table~\ref{tab:aer-train}). The \textit{Eval-out} regime considers results for the context set $\mathcal{C}_{\text{eval-out}}$ (Table~\ref{tab:aer-eval-out}). Relative gain is computed as $(\text{DMA*-SH} - \text{baseline}) / \text{baseline} \times 100$. Values are based on AER scores aggregated across environment types (see last three columns \textit{All}, \textit{Overlap}, and \textit{Non-overlap} in the respective Tables~\ref{tab:aer},~\ref{tab:aer-train} and~\ref{tab:aer-eval-out}).}
\label{tab:consolidatedAER}
\small
\begin{tabular}{llrrr}
\toprule
Regime & Type & vs DR & vs Concat & vs DA \\
\midrule
\multirow{3}{*}{Train}
 & All          & 50.0\%  & 6.3\%  & 3.7\% \\
 & Overlap      & 6.3\%   & 6.3\%  & 1.2\% \\
 & Non-overlap  & 90.9\%  & 6.3\%  & 5.0\% \\
\midrule
\multirow{3}{*}{Eval-out}
 & All          & 58.1\%  & 11.5\% & 4.6\% \\
 & Overlap      & 7.3\%   & 5.4\%  & 1.7\% \\
 & Non-overlap  & 97.3\%  & 14.1\% & 5.8\% \\
\midrule
\multirow{3}{*}{Aggregated}
 & All          & 51.9\%  & 8.2\%  & 3.9\% \\
 & Overlap      & 7.0\%   & 7.0\%  & 1.3\% \\
 & Non-overlap  & 95.1\%  & 8.1\%  & 3.9\% \\
\bottomrule
\end{tabular}
\end{table}

\begin{figure}[!htb]
    \centering
    \includegraphics[width=0.68\textwidth]{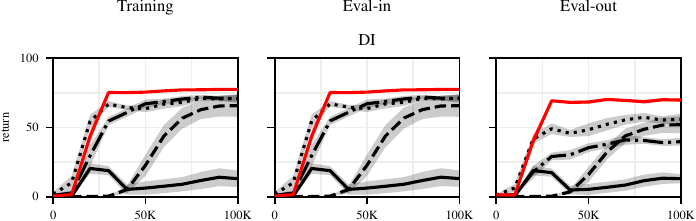}\\\vspace{0.2cm}
    \includegraphics[width=0.68\textwidth]{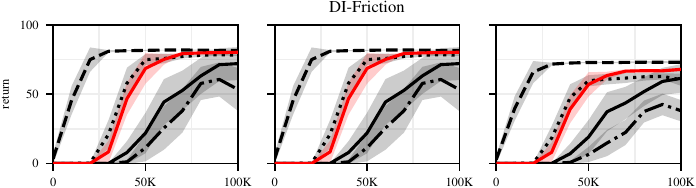}\\\vspace{0.2cm}
    \includegraphics[width=0.68\textwidth]{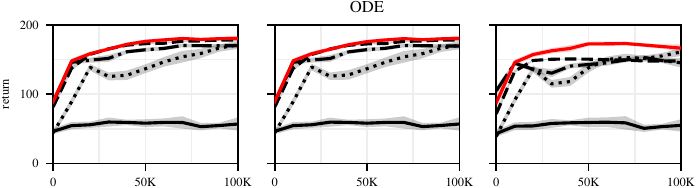}\\\vspace{0.2cm}
    \includegraphics[width=0.68\textwidth]{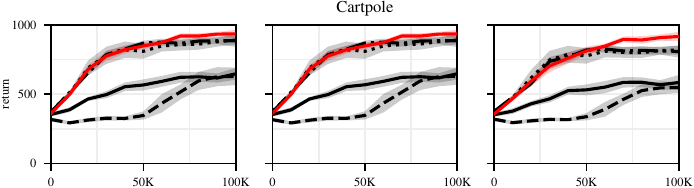}\\\vspace{0.2cm}
    \includegraphics[width=0.68\textwidth]{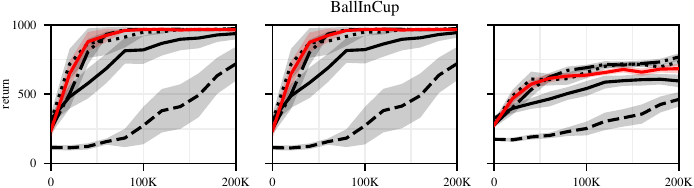}\\\vspace{0.2cm}
    \includegraphics[width=0.68\textwidth]{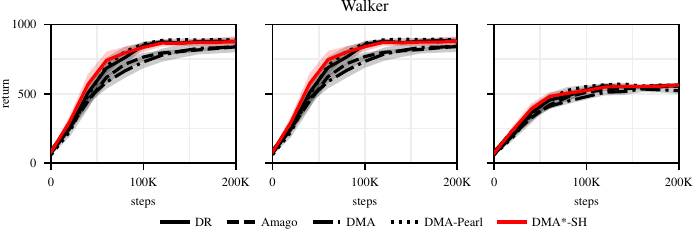}
    \caption{Returns over steps, averaged over the contexts used for training and evaluation, respectively. Comparison to baselines where context information is not explicitly available (Section~\ref{sec:baselines}).}
    \label{fig:returns}
\end{figure}

\clearpage

\subsection{Context-Instance Generalization Analysis} 

While the aggregated IQM provide a convenient high-level summary of performance across environments and context sets, contextual RL introduces an additional axis of variation that requires finer granularity \citep{benjamins2023contextualize,ndir2024inferring,prasanna2024dreaming}. To make generalization behavior explicit, we complement the aggregated metrics with detailed visualizations at the level of individual context instances.

For each environment, we evaluate our proposed DMA*-SH against baselines across the full grid of training and evaluation contexts and visualize the results using context-wise bar plots (Figures~\ref{fig:bar0}--\ref{fig:bar1}) and heatmaps (Figures~\ref{fig:heatmap0}--\ref{fig:heatmap5}).
These plots reveal how performance changes as evaluation contexts drift from the training distribution and highlight failure cases, such as the inability of the context-unaware DR baseline to handle non-overlapping dynamics (e.g., DI), or the challenges faced by the context-aware Concat method when dealing with extreme values of the context $c_0$ in the out-of-distribution regime of the ODE environment (Figure~\ref{fig:heatmap2} and Table~\ref{tab:aer-eval-out}), indicating difficulties with both positive and negative values of $c_0$.
In Cartpole, DMA*-SH achieves impressive consistency across all context instances (Figure~\ref{fig:bar1}). This instance-level view exposes trends that aggregate statistics may obscure and provides a clearer understanding of where and how generalization breaks down.

\begin{figure}[ht]
    \centering
\centering
    \includegraphics[width=0.9\textwidth]{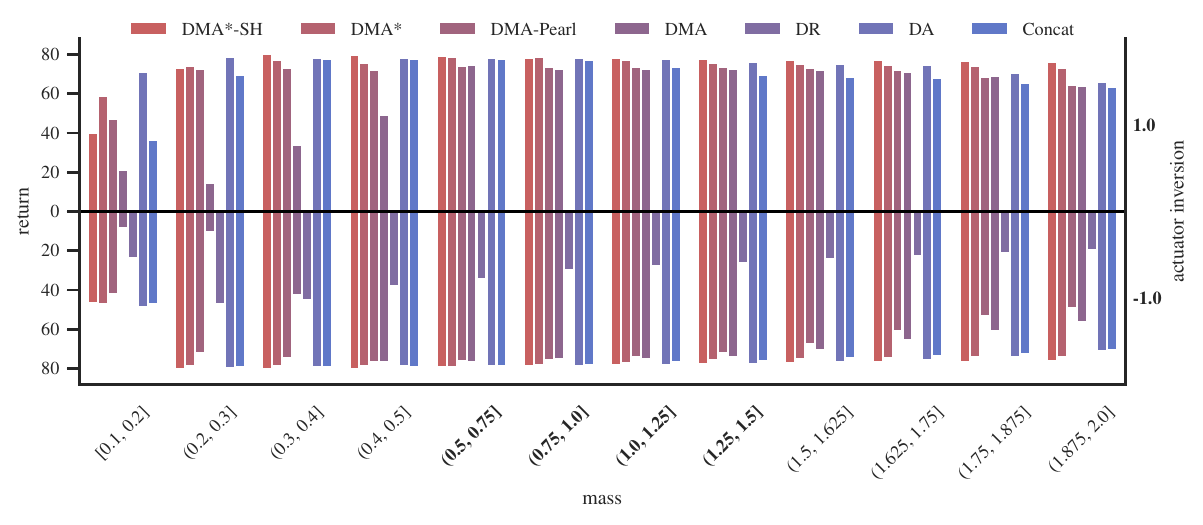}
    \caption{{Bar plot for DI to visualize AER for individual context instances and different methods. Bold labels refer to contexts used during training.}
    }
    \label{fig:bar0}
\end{figure}

\begin{figure}[ht]
    \centering
\centering
    \includegraphics[width=0.9\textwidth]{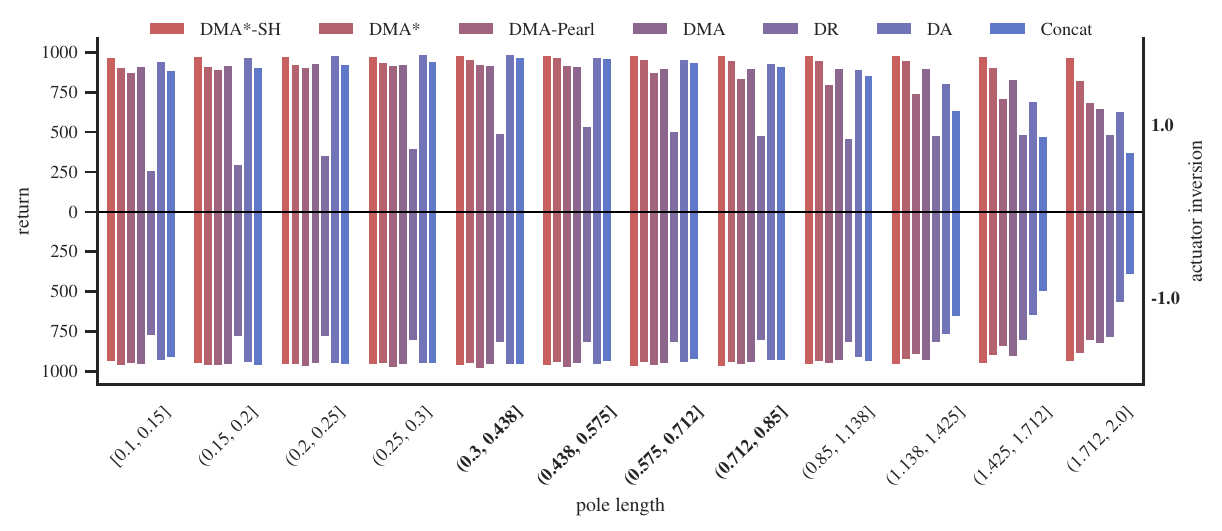}
    \caption{{Bar plot for Cartpole to visualize AER for individual context instances and different methods. Bold labels refer to contexts used during training.}
    }
    \label{fig:bar1}
\end{figure}

\begin{figure}[ht]
    \centering
    \begin{subfigure}[b]{0.5\textwidth}
        \centering
        \includegraphics[width=0.90\textwidth]{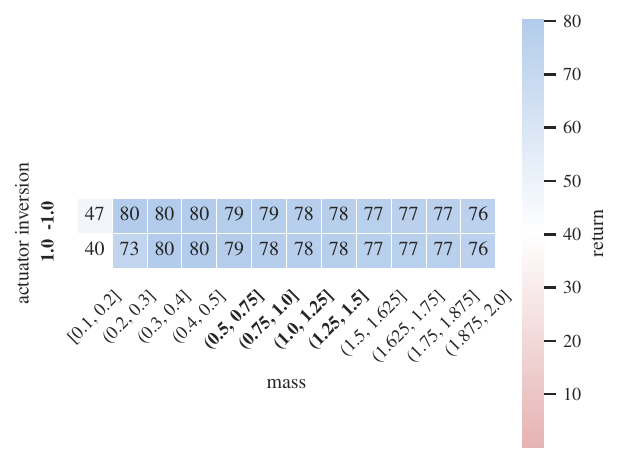}
        \caption{DMA*-SH.}
    \end{subfigure}\hfill
    \begin{subfigure}[b]{0.5\textwidth}
        \centering
        \includegraphics[width=0.90\textwidth]{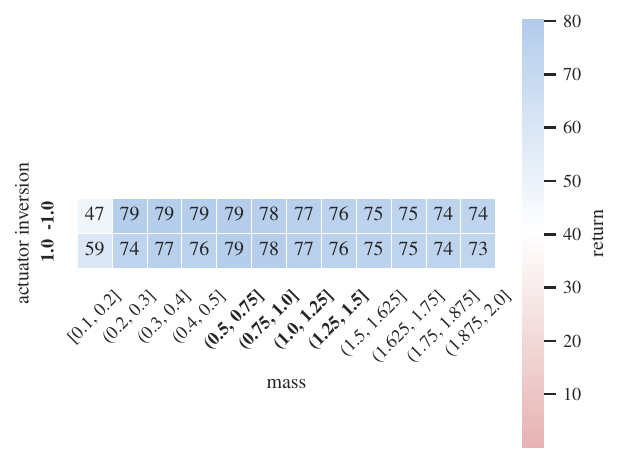}
        \caption{DMA*.}
    \end{subfigure}
    \par\bigskip
    \begin{subfigure}[b]{0.5\textwidth}
        \centering
        \includegraphics[width=0.90\textwidth]{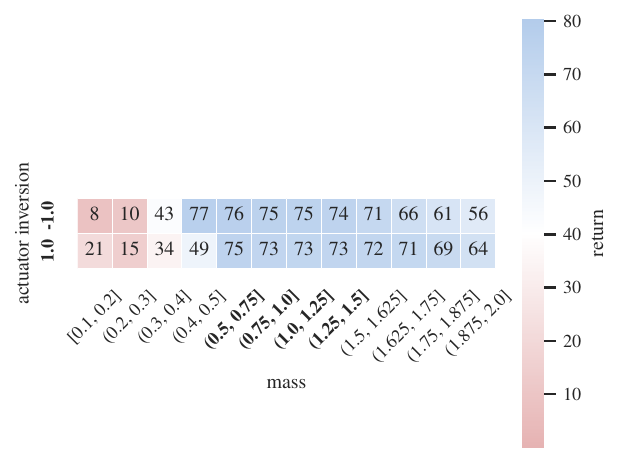}
        \caption{DMA.}
    \end{subfigure}\hfill
    \begin{subfigure}[b]{0.5\textwidth}
        \centering
        \includegraphics[width=0.90\textwidth]{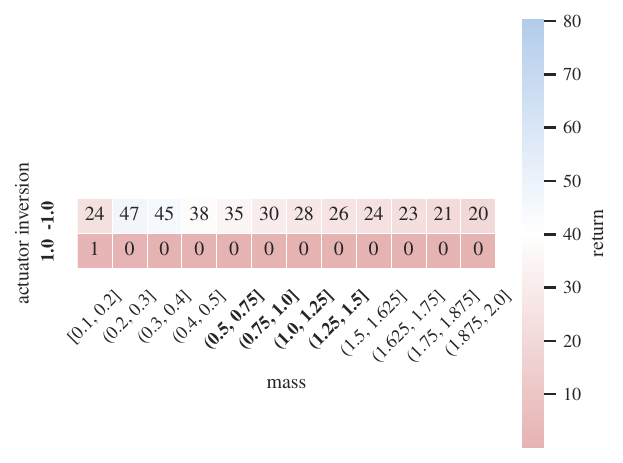}
        \caption{DR.}
    \end{subfigure}
    \par\bigskip
    \begin{subfigure}[b]{0.5\textwidth}
        \centering
        \includegraphics[width=0.90\textwidth]{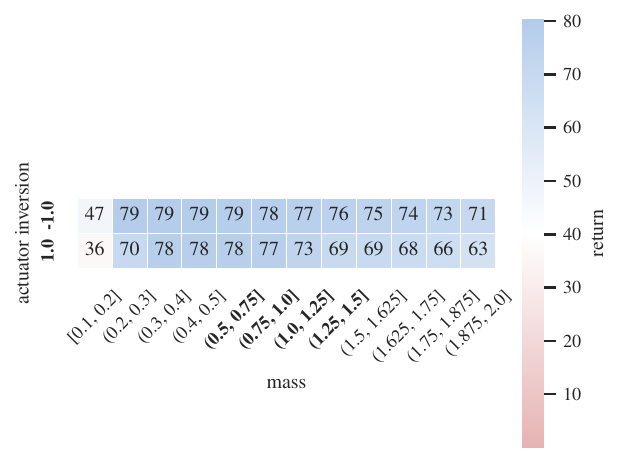}
        \caption{Concat.}
    \end{subfigure}\hfill
    \begin{subfigure}[b]{0.5\textwidth}
        \centering
        \includegraphics[width=0.90\textwidth]{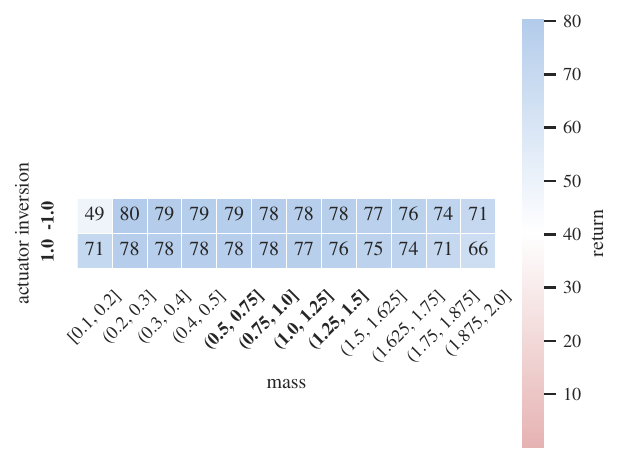}
        \caption{DA.}
    \end{subfigure}
    \caption{{Heatmaps for DI to visualize AER for individual context instances. Bold labels refer to contexts used during training.}}
    \label{fig:heatmap0}
\end{figure}

\begin{figure}[ht]
    \centering
    \begin{subfigure}[b]{0.5\textwidth}
        \centering
        \includegraphics[width=0.90\textwidth]{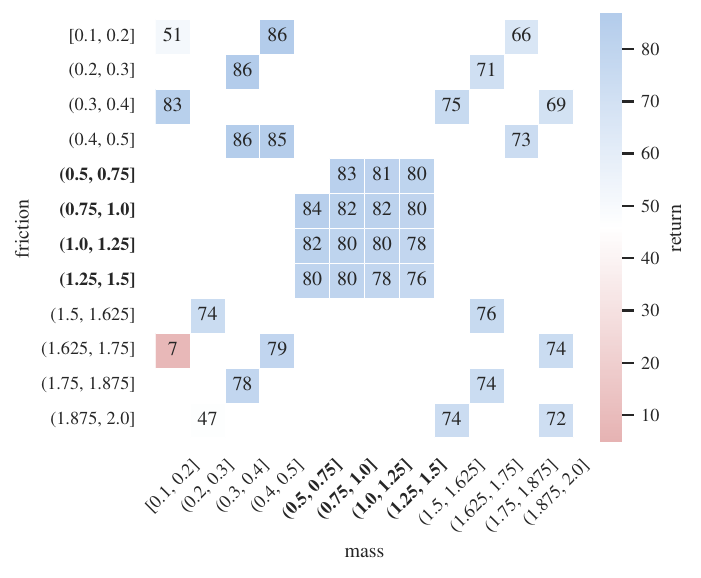}
        \caption{DMA*-SH.}
    \end{subfigure}\hfill
    \begin{subfigure}[b]{0.5\textwidth}
        \centering
        \includegraphics[width=0.90\textwidth]{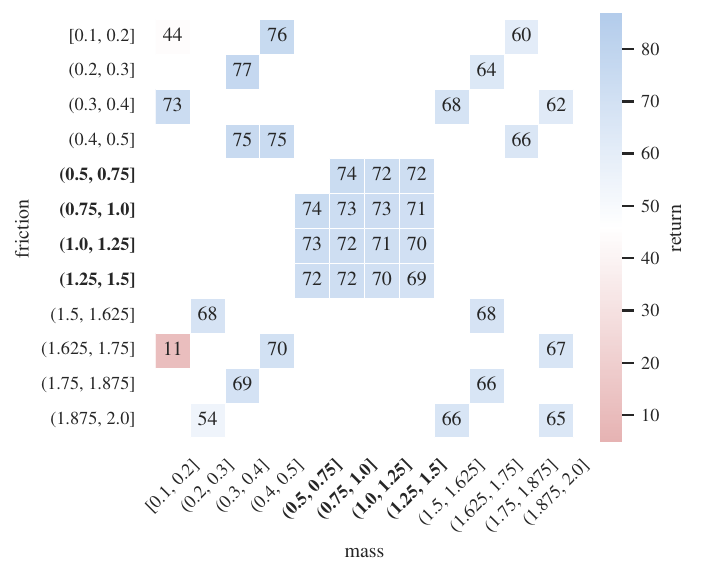}
        \caption{DMA*.}
    \end{subfigure}
    \par\bigskip
    \begin{subfigure}[b]{0.5\textwidth}
        \centering
        \includegraphics[width=0.90\textwidth]{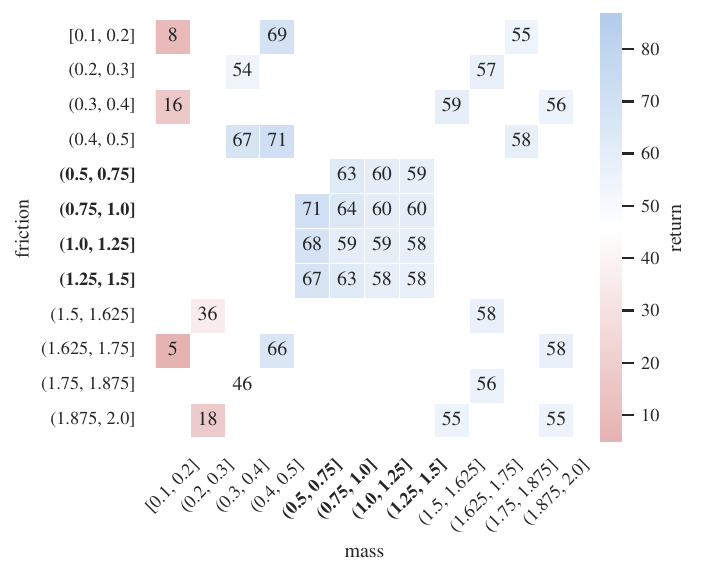}
        \caption{DMA.}
    \end{subfigure}\hfill
    \begin{subfigure}[b]{0.5\textwidth}
        \centering
        \includegraphics[width=0.90\textwidth]{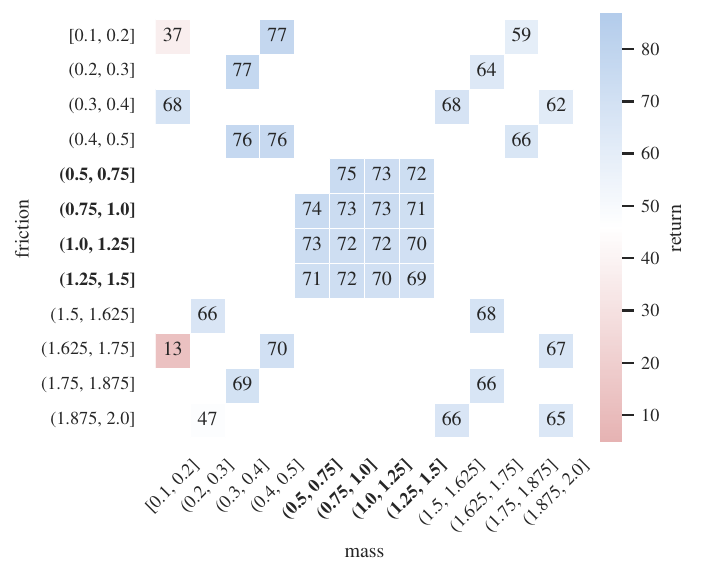}
        \caption{DR.}
    \end{subfigure}
    \par\bigskip
    \begin{subfigure}[b]{0.5\textwidth}
        \centering
        \includegraphics[width=0.90\textwidth]{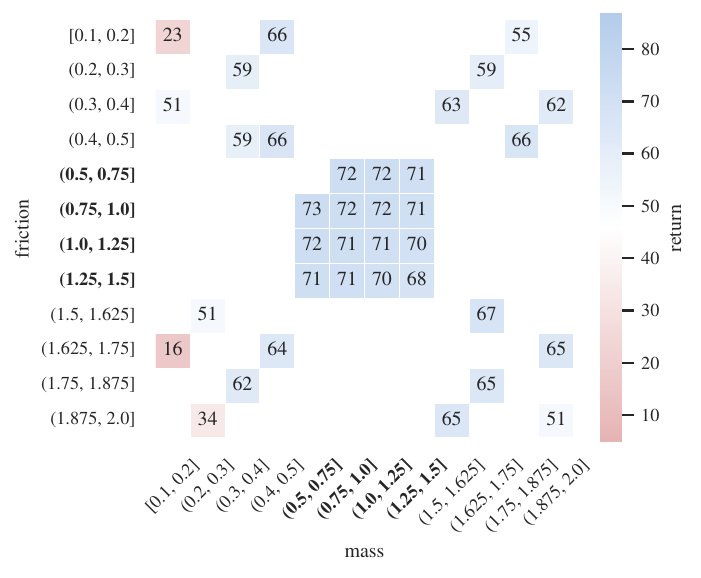}
        \caption{Concat.}
    \end{subfigure}\hfill
    \begin{subfigure}[b]{0.5\textwidth}
        \centering
        \includegraphics[width=0.90\textwidth]{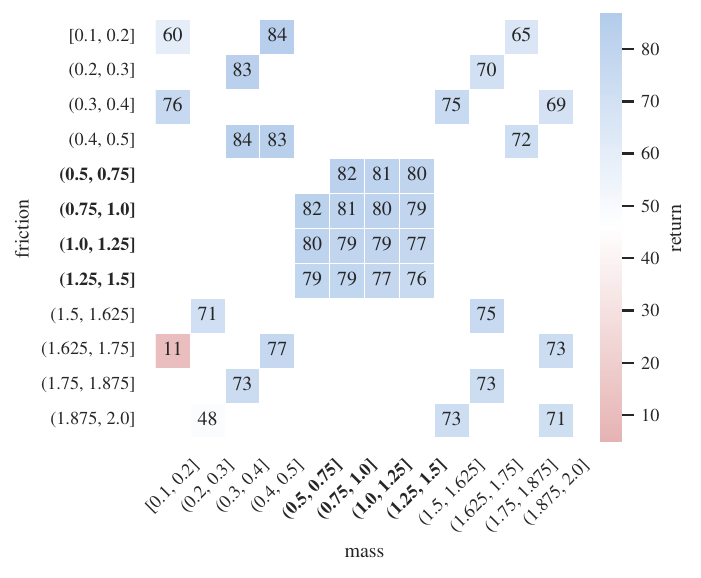}
        \caption{DA.}
    \end{subfigure}
    \caption{{Heatmaps for DI-Friction to visualize AER for individual context instances. Bold labels refer to contexts used during training.}
    }
    \label{fig:heatmap1}
\end{figure}

\begin{figure}[ht]
    \centering
    \begin{subfigure}[b]{0.5\textwidth}
        \centering
        \includegraphics[width=0.90\textwidth]{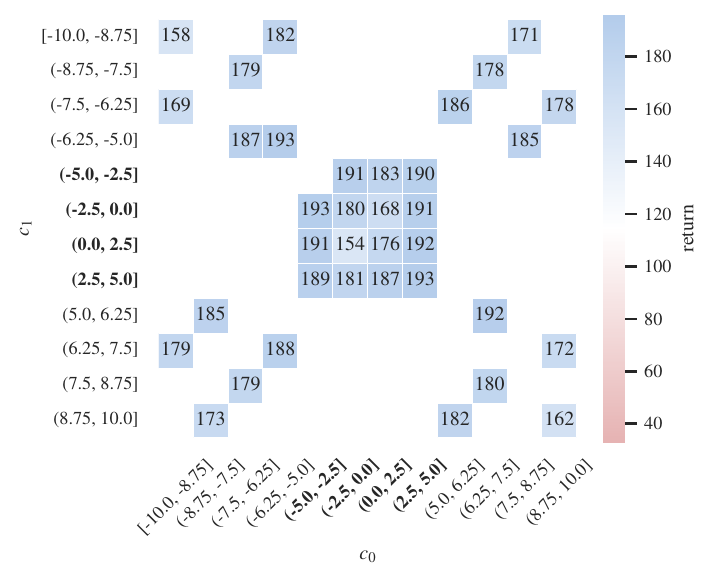}
        \caption{DMA*-SH.}
    \end{subfigure}\hfill
    \begin{subfigure}[b]{0.5\textwidth}
        \centering
        \includegraphics[width=0.90\textwidth]{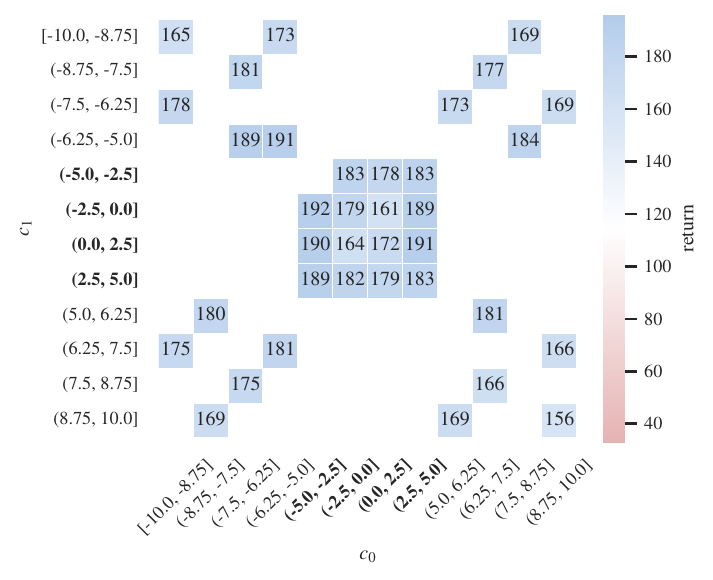}
        \caption{DMA*.}
    \end{subfigure}
    \par\bigskip
    \begin{subfigure}[b]{0.5\textwidth}
        \centering
        \includegraphics[width=0.90\textwidth]{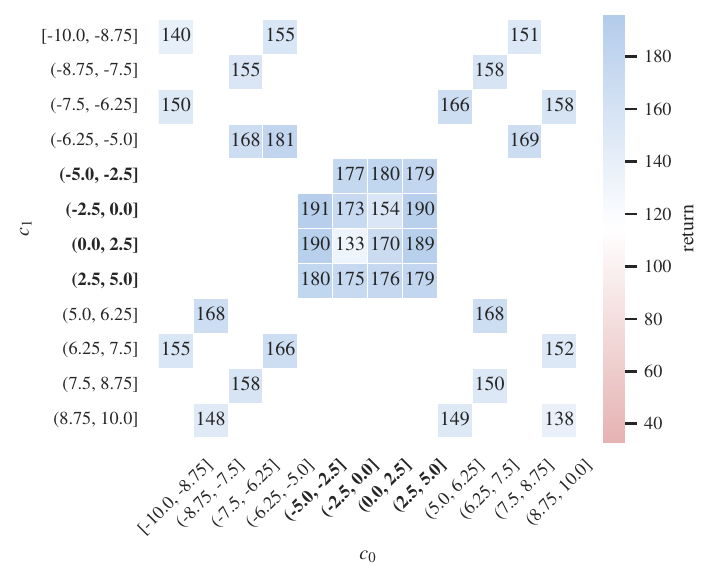}
        \caption{DMA.}
    \end{subfigure}\hfill
    \begin{subfigure}[b]{0.5\textwidth}
        \centering
        \includegraphics[width=0.90\textwidth]{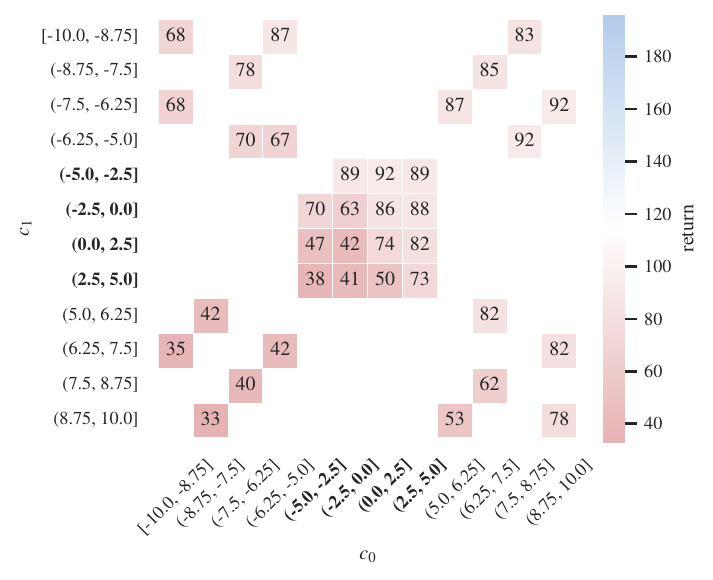}
        \caption{DR.}
    \end{subfigure}
    \par\bigskip
    \begin{subfigure}[b]{0.5\textwidth}
        \centering
        \includegraphics[width=0.90\textwidth]{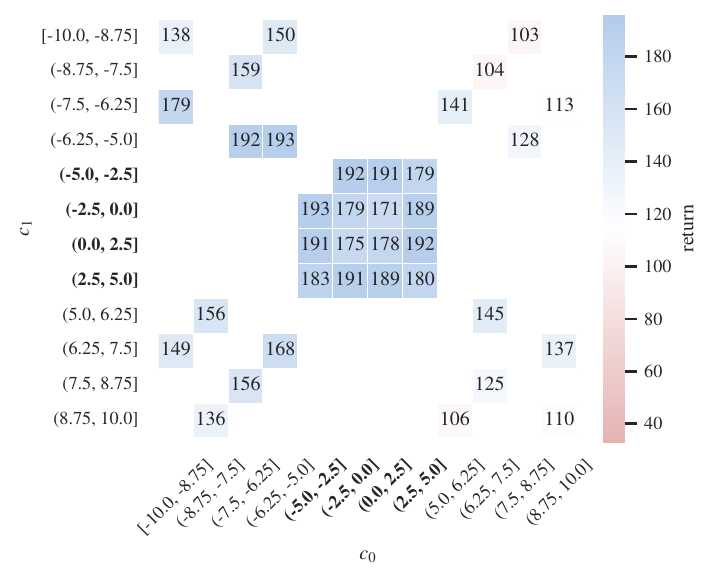}
        \caption{Concat.}
    \end{subfigure}\hfill
    \begin{subfigure}[b]{0.5\textwidth}
        \centering
        \includegraphics[width=0.90\textwidth]{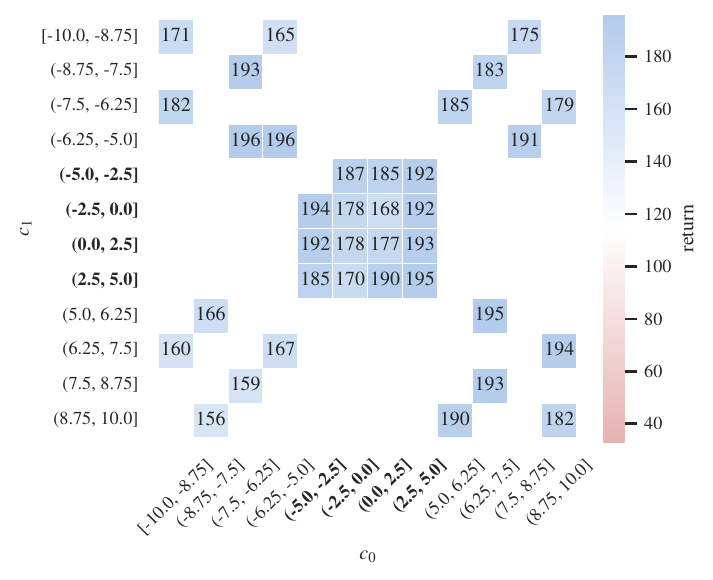}
        \caption{DA.}
    \end{subfigure}
    \caption{{Heatmaps for ODE to visualize AER for individual context instances. Bold labels refer to contexts used during training.}}
    \label{fig:heatmap2}
\end{figure}

\begin{figure}[ht]
    \centering
    \begin{subfigure}[b]{0.5\textwidth}
        \centering
        \includegraphics[width=0.90\textwidth]{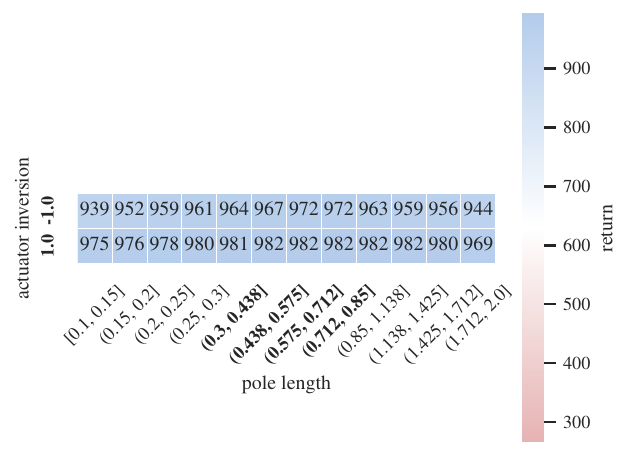}
        \caption{DMA*-SH.}
    \end{subfigure}\hfill
    \begin{subfigure}[b]{0.5\textwidth}
        \centering
        \includegraphics[width=0.90\textwidth]{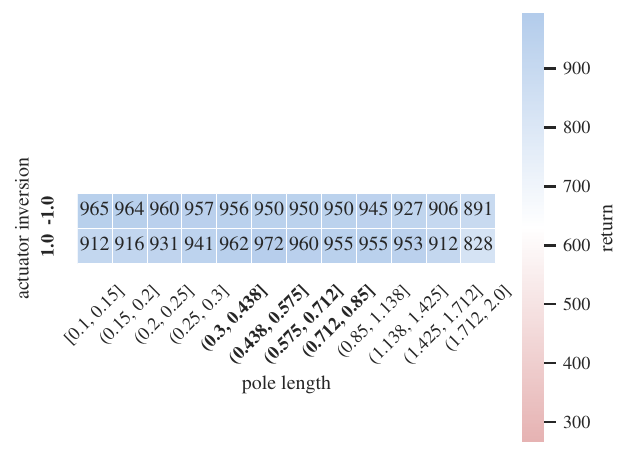}
        \caption{DMA*.}
    \end{subfigure}
    \par\bigskip
    \begin{subfigure}[b]{0.5\textwidth}
        \centering
        \includegraphics[width=0.90\textwidth]{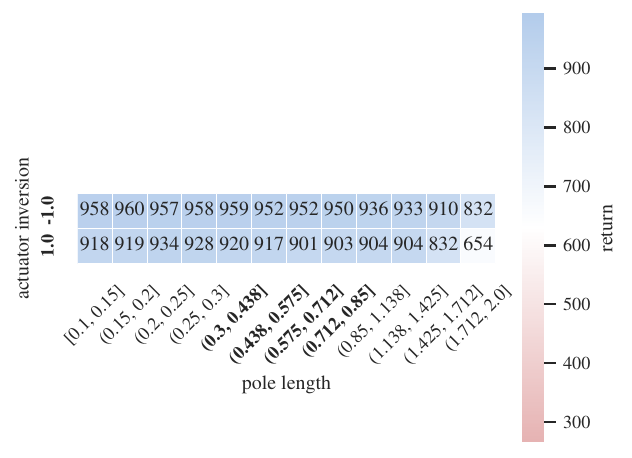}
        \caption{DMA.}
    \end{subfigure}\hfill
    \begin{subfigure}[b]{0.5\textwidth}
        \centering
        \includegraphics[width=0.90\textwidth]{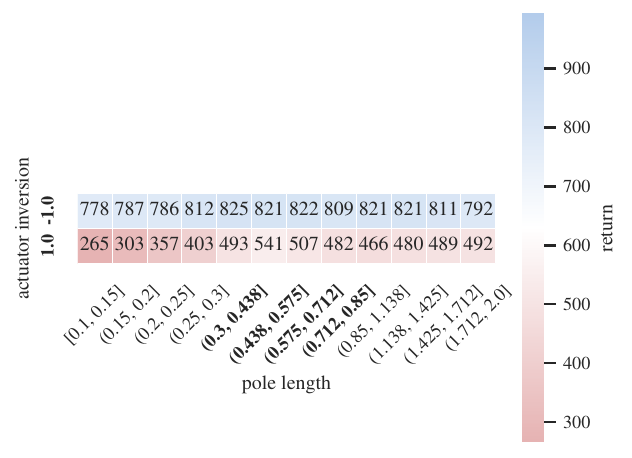}
        \caption{DR.}
    \end{subfigure}
    \par\bigskip
    \begin{subfigure}[b]{0.5\textwidth}
        \centering
        \includegraphics[width=0.90\textwidth]{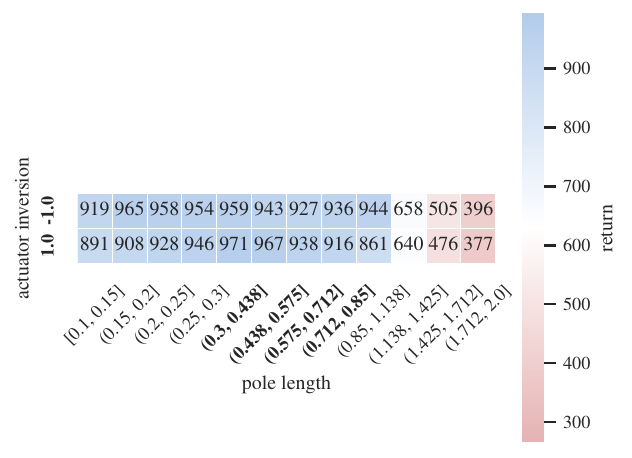}
        \caption{Concat.}
    \end{subfigure}\hfill
    \begin{subfigure}[b]{0.5\textwidth}
        \centering
        \includegraphics[width=0.90\textwidth]{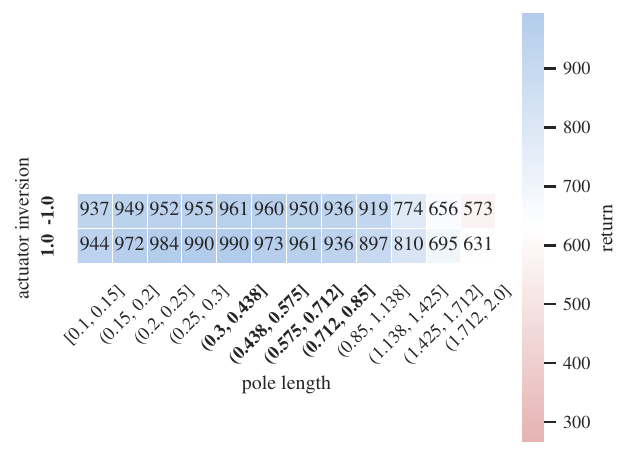}
        \caption{DA.}
    \end{subfigure}
    \caption{{Heatmaps for Cartpole to visualize AER for individual context instances. Bold labels refer to contexts used during training.}}
    \label{fig:heatmap3}
\end{figure}

\begin{figure}[ht]
    \centering
    \begin{subfigure}[b]{0.5\textwidth}
        \centering
        \includegraphics[width=0.90\textwidth]{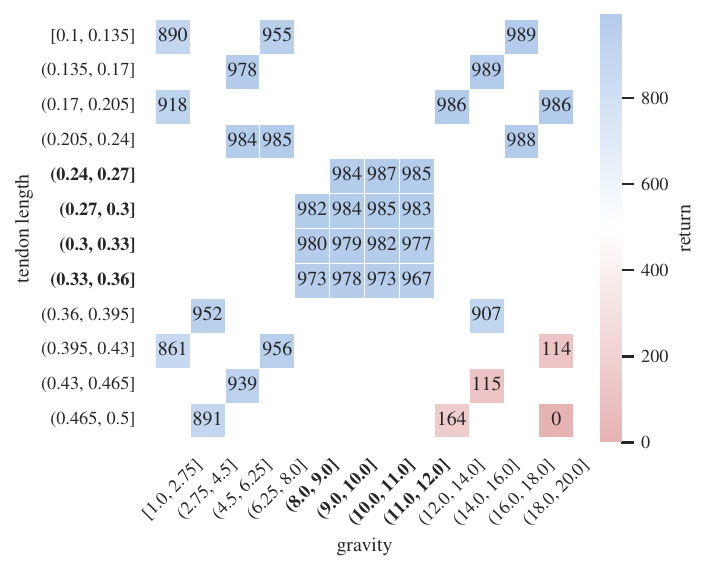}
        \caption{DMA*-SH.}
    \end{subfigure}\hfill
    \begin{subfigure}[b]{0.5\textwidth}
        \centering
        \includegraphics[width=0.90\textwidth]{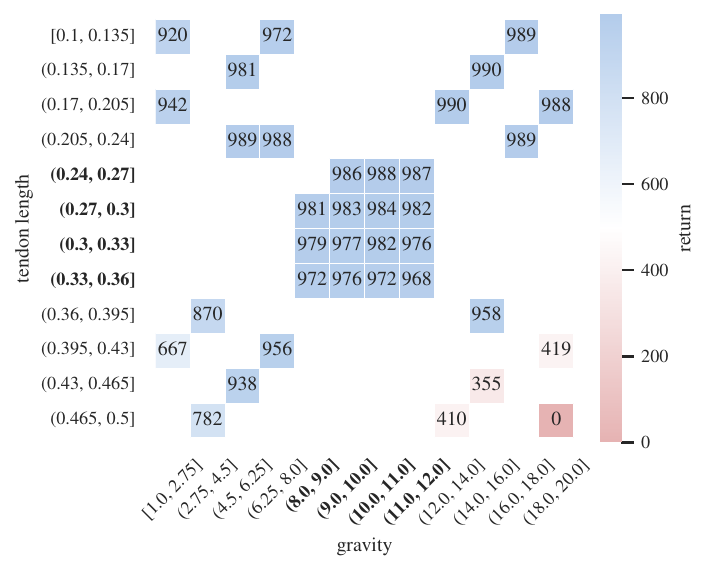}
        \caption{DMA*.}
    \end{subfigure}
    \par\bigskip
    \begin{subfigure}[b]{0.5\textwidth}
        \centering
        \includegraphics[width=0.90\textwidth]{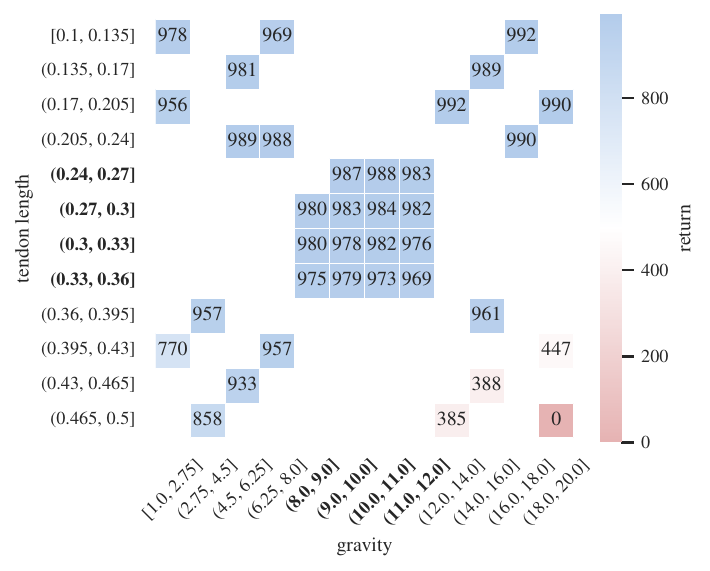}
        \caption{DMA.}
    \end{subfigure}\hfill
    \begin{subfigure}[b]{0.5\textwidth}
        \centering
        \includegraphics[width=0.90\textwidth]{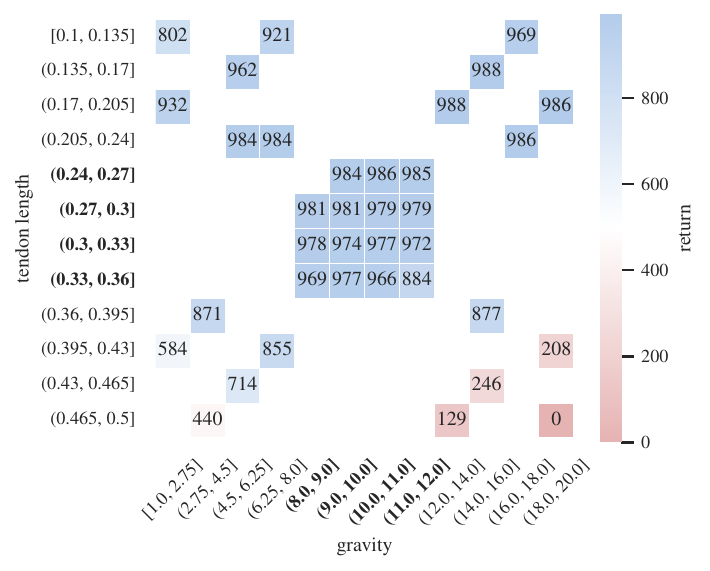}
        \caption{DR.}
    \end{subfigure}
    \par\bigskip
    \begin{subfigure}[b]{0.5\textwidth}
        \centering
        \includegraphics[width=0.90\textwidth]{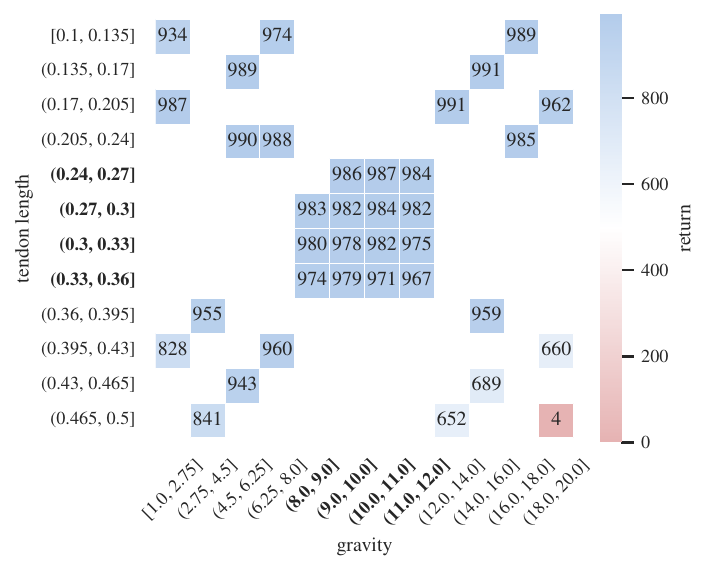}
        \caption{Concat.}
    \end{subfigure}\hfill
    \begin{subfigure}[b]{0.5\textwidth}
        \centering
        \includegraphics[width=0.90\textwidth]{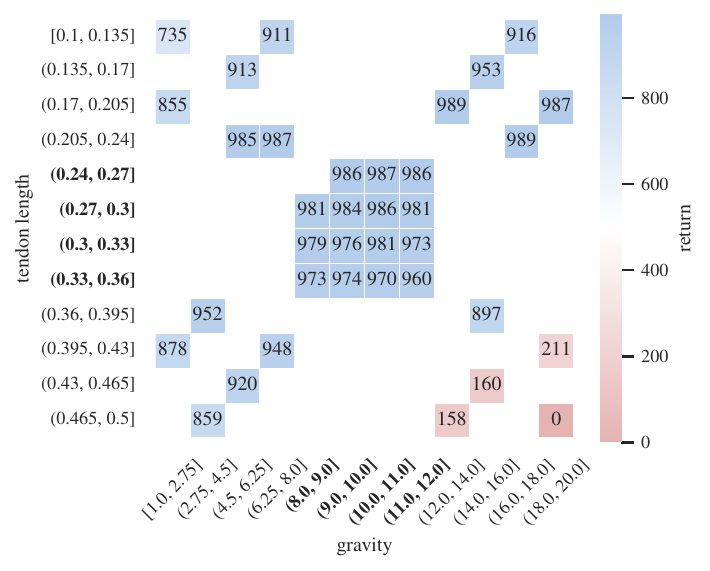}
        \caption{DA.}
    \end{subfigure}
    \caption{{Heatmaps for BallInCup to visualize AER for individual context instances. Bold labels refer to contexts used during training.}}
    \label{fig:heatmap4}
\end{figure}

\begin{figure}[ht]
    \centering
    \begin{subfigure}[b]{0.5\textwidth}
        \centering
        \includegraphics[width=0.90\textwidth]{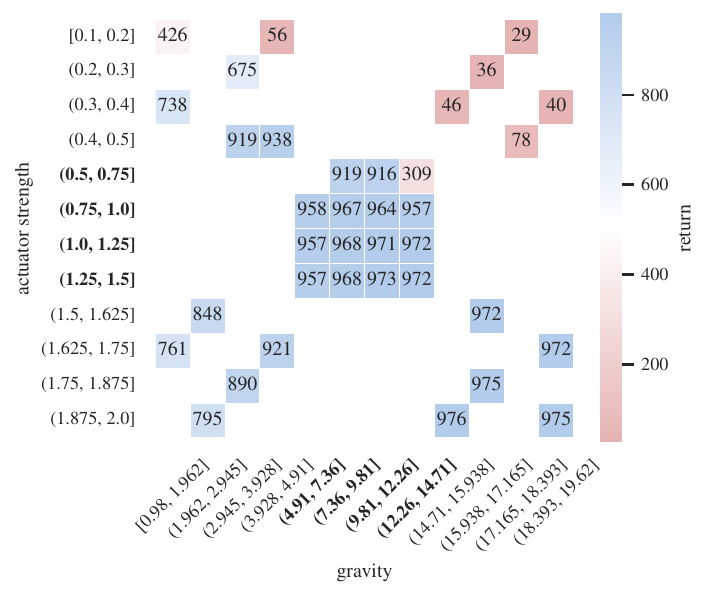}
        \caption{DMA*-SH.}
    \end{subfigure}\hfill
    \begin{subfigure}[b]{0.5\textwidth}
        \centering
        \includegraphics[width=0.90\textwidth]{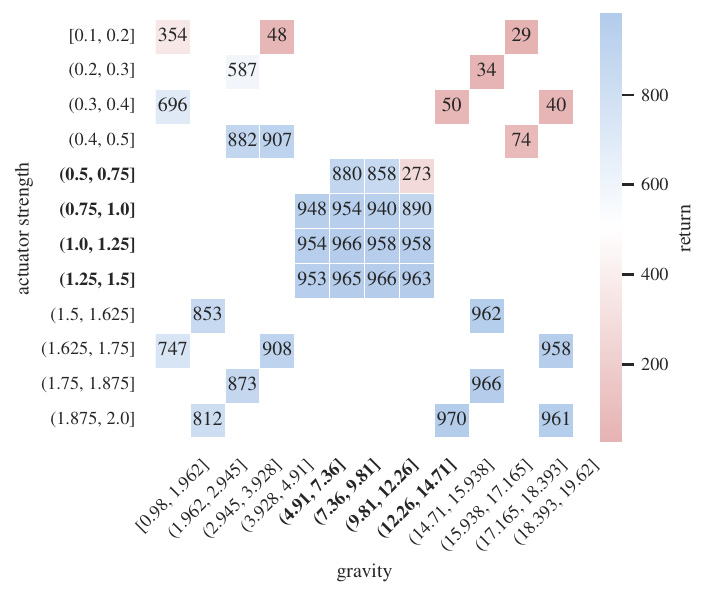}
        \caption{DMA*.}
    \end{subfigure}
    \par\bigskip
    \begin{subfigure}[b]{0.5\textwidth}
        \centering
        \includegraphics[width=0.90\textwidth]{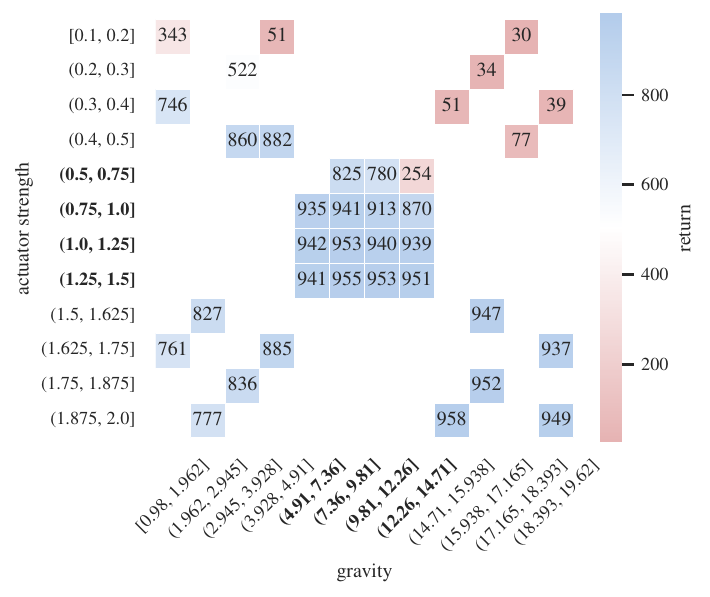}
        \caption{DMA.}
    \end{subfigure}\hfill
    \begin{subfigure}[b]{0.5\textwidth}
        \centering
        \includegraphics[width=0.90\textwidth]{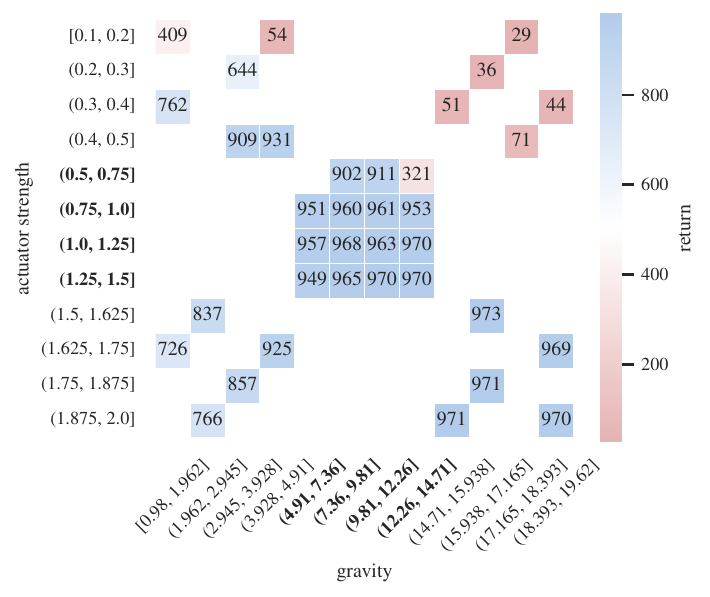}
        \caption{DR.}
    \end{subfigure}
    \par\bigskip
    \begin{subfigure}[b]{0.5\textwidth}
        \centering
        \includegraphics[width=0.90\textwidth]{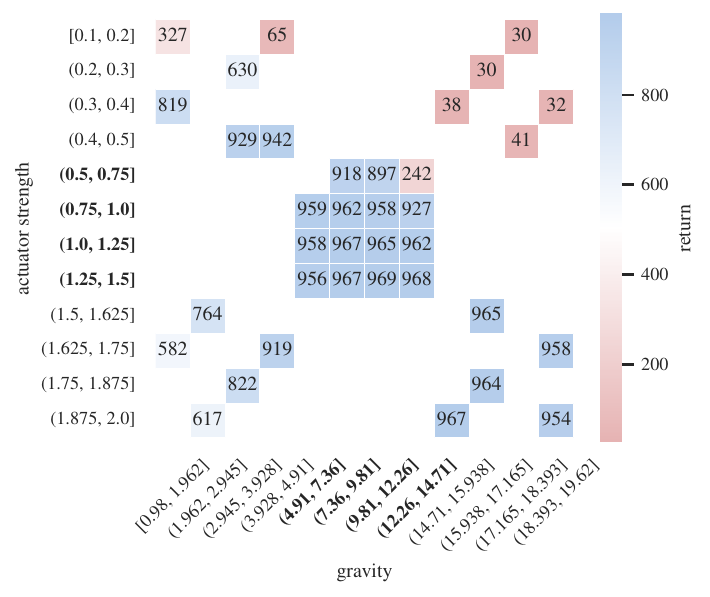}
        \caption{Concat.}
    \end{subfigure}\hfill
    \begin{subfigure}[b]{0.5\textwidth}
        \centering
        \includegraphics[width=0.90\textwidth]{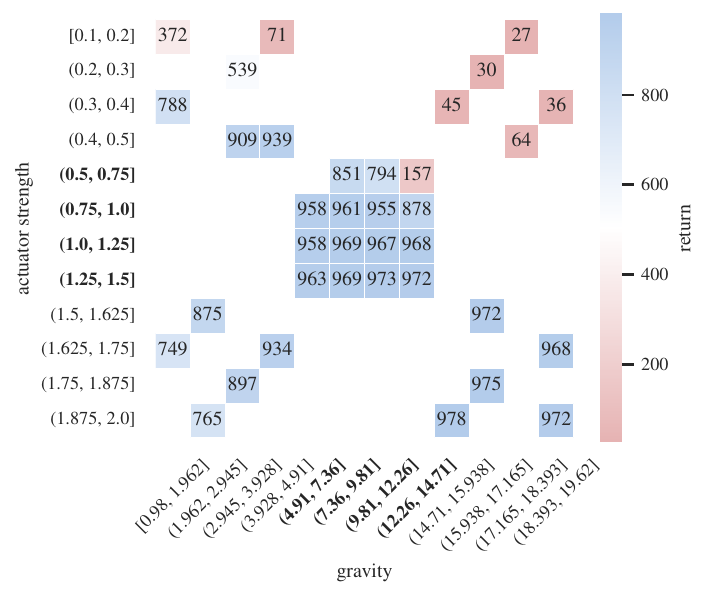}
        \caption{DA.}
    \end{subfigure}
    \caption{{Heatmaps for Walker to visualize AER for individual context instances. Bold labels refer to contexts used during training.}}
    \label{fig:heatmap5}
\end{figure}

\clearpage
\subsection{Scalability in the Explicit Context Dimension} 

The standard ODE task (ODE-2) is governed by a differential equation parameterized by two context variables $c_{1}$ and $c_{2}$ (Table~\ref{tab:envs}). To evaluate scalability with respect to context dimensionality, we extend this to ODE-$k$ using higher-order polynomials (Appendix~\ref{app:env_details}), increasing the number of context parameters from 1 to 6. Figure~\ref{fig:iqm-scale} shows performance aggregated across these six variants. The results demonstrate that DMA*-SH scales favorably to higher context dimensions, whereas the context-aware Concat baseline struggles in both Eval-in and Eval-out regimes.

\begin{figure}[ht]
    \centering
\centering
    \includegraphics[width=0.5\textwidth]{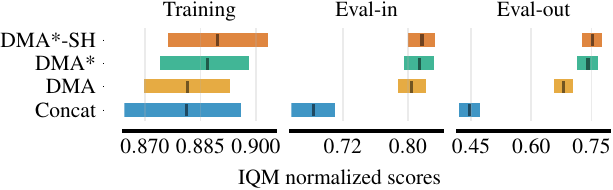}
    \caption{Interquartile mean (IQM) aggregated over six ODE variants, ODE-1, ODE-2, ... , ODE-6 to test for scalability with respect to context dimensionality.
    }
    \label{fig:iqm-scale}
\end{figure}

\end{document}